\documentclass[oneside,12pt,letterpaper]{Classes/CUEDthesisPSnPDF}
\usepackage[letterpaper, margin=1in]{geometry}
\usepackage{latexsym}
\usepackage{amsmath}
\usepackage{amssymb}
\usepackage{amsthm}
\usepackage{graphicx}
\usepackage[tight,footnotesize]{subfigure}
\usepackage{multirow}
\usepackage{tabularx}
\usepackage{array}
\usepackage{algpseudocode}
\usepackage{algorithm}
\usepackage{algpseudocode}
\usepackage{float}
\usepackage{color}
\usepackage[justification=centering]{caption}
\usepackage{changepage}
\usepackage{wrapfig}
\usepackage{comment}
\usepackage{diagbox}
\usepackage{rotating}
\usepackage[toc,page]{appendix}
\usepackage{cleveref}
\usepackage{titlesec}
\usepackage{soul}
\usepackage{setspace}
\usepackage{pdfpages}
\usepackage{ragged2e}

\ifpdf
    \pdfcatalog { /PageMode (/UseOutlines)
                  /OpenAction (fitbh)  }
\fi
\title{ EFFECTIVE LEARNING OF PROBABILISTIC MODELS FOR CLINICAL PREDICTIONS FROM LONGITUDINAL DATA }

\ifpdf
  \author{Shuo Yang}
  \collegeordept{}
  \university{}
  \crest
\else
  \author{Shuo Yang}
  \collegeordept{}
  \university{}
  \crest
\fi
\degree{}
\degreedate{}
\usepackage{StyleFiles/watermark}

\newtheorem{theorem}{Theorem}[section]

\newenvironment{definition}[1][Definition]{\begin{trivlist}
		\item[\hskip \labelsep {\bfseries #1}]}{\end{trivlist}}

\algnewcommand\algorithmicswitch{\textbf{switch}}
\algnewcommand\algorithmiccase{\textbf{case}}
\algnewcommand\algorithmicassert{\texttt{assert}}
\algdef{SE}[SWITCH]{Switch}{EndSwitch}[1]{\algorithmicswitch\ #1\ \algorithmicdo}{\algorithmicend\ \algorithmicswitch}%
\algdef{SE}[CASE]{Case}{EndCase}[1]{\algorithmiccase\ #1}{\algorithmicend\ \algorithmiccase}%
\algtext*{EndSwitch}%
\algtext*{EndCase}%

\titleformat{\chapter}[display]
{\normalfont\bfseries\itshape}
{\chaptertitlename\ \thechapter}{12pt}{}
\titlespacing*{\chapter}
{0pt}{0pt}{20pt}

\titleformat{\section}
{\normalfont\bfseries}
{\thesection.}{0.7em}{}

\titleformat{\subsection}
{\normalfont\bfseries}
{\thesubsection.}{0.5em}{}

\begin{document}

\maketitle
\doublespacing

\setcounter{secnumdepth}{3}
\setcounter{tocdepth}{3}

\frontmatter 
\pagenumbering{roman}
\setcounter{page}{2} 
\begin{AcceptancePage}
	{Doctoral Committee}\\
	
	\hfill\begin{tabular}{@{}p{.6\linewidth}@{}}		
		\line(1,0){280}
		\begin{flushright}
		\vspace{-3mm}			
		{Chairperson: Sriraam Natarajan, PhD }\\
		\vspace{10mm}
		\line(1,0){280}
		\vspace{-3mm}
		{Committee Member: David Leake, PhD}\\
		\vspace{10mm}
		\line(1,0){280}
		\vspace{-3mm}
		{Committee Member: Kay Connelly, PhD}\\
		\vspace{10mm}
		\line(1,0){280}
		\vspace{-3mm}
		{Committee Member: Kristian Kersting, PhD}\\
		\vspace{10mm}
		\line(1,0){280}
		\vspace{-3mm}
		{Committee Member: Predrag Radivojac, PhD}
		\end{flushright}
	\end{tabular}\\

{Date of Defense: 7/5/2017}	
\end{AcceptancePage}

\begin{dedication} 

{I would like to dedicate this thesis to my loving parents, my husband and my beloved dog Kraken for their unconditional love.}

\end{dedication}




\begin{acknowledgements}      

Firstly, I would like to express my sincere gratitude to my advisor Prof. Sriraam Natarajan who has been a tremendous mentor for me. I would like to thank him for recognizing my potential and enlightening me on machine learning, which opened a whole new world to me. His guidance and continuous support helped me through out my Ph.D study. I could not have imagined having any of my current accomplishments without his patience, motivation, and immense knowledge. His advice on both research as well as on my career have been priceless. He is the most supportive advisor and one of the smartest people I know. I hope that I could be as enthusiastic, and energetic as him.\\

I would also like to thank my committee members, Prof. David Leake, Prof. Kay Connelly, Prof. Kristian Kersting and Prof. Predrag Radivojac for serving as my committee members and their inputs on my work. Prof. Kersting, as a co-author of all my Ph.D. publications, has provided me innumerous valuable suggestions and insightful guidance.  I have learned a lot about other machine learning techniques and research directions through our discussions on my defense. I am thankful for their brilliant comments and suggestions. \\

I thank all my collaborators, Dr. Tushar Khot, Dr. Mohammed Korayem, Dr. Khalifeh AlJadda, Dr. Gautam Kunapuli, Dr. Shaun Grannis and Dr. Jefferey Carr for their constructive inputs. I would not have presented this thesis without their contributions. I also thank Phillip Odom, Mayukh Das and all my lab-mates for the stimulating discussions through out my Ph.D career. \\

A special thanks to my family. Words cannot express how grateful I am to my parents and my husband for their constant support and unconditional love. I would also like to thank all of my friends who supported me in my research as well as in my life.

\end{acknowledgements}




\begin{abstracts}        

With the expeditious advancement of information technologies, health-related data presented unprecedented potentials for medical and health discoveries but at the same time significant challenges for machine learning techniques both in terms of size and complexity.  
Those challenges include: the structured data with various storage formats and value types caused by heterogeneous data sources; the uncertainty widely existing in every aspect of medical diagnosis and treatments; the high dimensionality of the feature space; the longitudinal medical records data with irregular intervals between adjacent observations; the richness of relations existing among objects with similar genetic factors, location or socio-demographic background. This thesis aims to develop advanced Statistical Relational Learning approaches in order to effectively exploit such health-related data and facilitate the discoveries in medical research. It presents the work on cost-sensitive statistical relational learning for mining structured imbalanced data, the first continuous-time probabilistic logic model for predicting sequential events from longitudinal structured data as well as hybrid probabilistic relational models for learning from heterogeneous structured data. It also demonstrates the outstanding performance of these proposed models as well as other state of the art machine learning models when applied to medical research problems and other real-world large-scale systems, reveals the great potential of statistical relational learning for exploring the structured health-related data to facilitate medical research. 

\vspace{10mm}

\hfill\begin{tabular}{@{}p{.6\linewidth}@{}}		
	\line(1,0){280}
	\begin{flushright}
		\vspace{-3mm}			
		{Chairperson: Sriraam Natarajan, PhD }\\
		\vspace{10mm}
		\line(1,0){280}
		\vspace{-3mm}
		{Committee Member: David Leake, PhD}\\
	\end{flushright}
\end{tabular}

\hfill\begin{tabular}{@{}p{.6\linewidth}@{}}
	\vspace{10mm}
	  \line(1,0){280}
	   \begin{flushright}
		\vspace{-3mm}
		{Committee Member: Kay Connelly, PhD}\\
		\vspace{10mm}
		\line(1,0){280}
		\vspace{-3mm}
		{Committee Member: Kristian Kersting, PhD}\\
		\vspace{10mm}
		\line(1,0){280}
		\vspace{-3mm}
		{Committee Member: Predrag Radivojac, PhD}
	\end{flushright}
\end{tabular}\\

\end{abstracts}



\tableofcontents

\mainmatter 
\chapter{Introduction}
\ifpdf
    \graphicspath{{Introduction/IntroductionFigs/PNG/}{Introduction/IntroductionFigs/PDF/}{Introduction/IntroductionFigs/}}
\else
    \graphicspath{{Introduction/IntroductionFigs/EPS/}{Introduction/IntroductionFigs/}}
\fi

With its expeditious advancement, information technology provides a revolutionary way to collect, exchange and store enormous amount of health-related information. Such information includes: the database in modern hospital systems, usually known as Electronic Health Records (EHR), which store the patients' diagnosis, medication, laboratory test results, medical image data, etc.; information on various health behaviors tracked and stored by wearable devices, ubiquitous sensors and mobile applications, such as the smoking status, alcoholism history, exercise level, sleeping conditions, etc.; information collected by census or various surveys regarding sociodemographic factors of the target cohort; and information on people's mental health inferred from their social media activities or social networks such as Twitter, Facebook, etc. These health-related data come from heterogeneous sources, describe assorted aspects of the individual's health conditions.
Such data is rich in structure and information which has great research potentials for revealing unknown medical knowledge about genomic epidemiology, disease developments and correlations, drug discoveries, medical diagnosis, mental illness prevention, health behavior adaption, etc. 

In real-world problems, the number of features relating to a certain health condition could grow exponentially with the development of new information techniques for collecting and measuring data. To reveal the causal influence between various factors and a certain disease or to discover the correlations among diseases from data at such a tremendous scale, requires the assistance of advanced information technology such as data mining, machine learning, text mining, etc. Machine learning technology not only provides a way for learning qualitative relationships among features and patients, but also the quantitative parameters regarding the strength of such correlations. 
Such data-driven discoveries can provide better profiling on the target population with potential risks of certain diseases and hence facilitate the process of recruiting patients for case-control clinical studies or drug trials. 

The genomic research is one of those areas that first started adapting machine learning methods to facilitate domain discoveries such as exploring the genotype-phenotype relationships~\cite{Szymczak09,Menden12}. 
EHR data mining is another area where the machine learning technique has been successfully assisting the health discoveries. Oztekin et al.~\citeyear{OztekinDK09} applied machine learning models to predict the graft survival for heart-lung transplantation patients and got 86$\%$ accuracy rate with the assistance of various feature selection approaches. 
Besides the research on applying machine learning techniques to EHR data alone, there are also some research on integrating the EHR data with the genetic data in the area of pharmacotherapy, see for instance, \cite{Wilke11}.

Despite the fact that health related big data provides promising opportunities for the potential medical and health discoveries, the richness in structure and uncertainty also brings enormous challenges for the machine learning techniques to make full use of such health related data. There are still many open problems for machine learning technology to tackle when applied to the real-world problems.  

\section{Challenges in Mining Health Data }

The first challenge is that each data source that describes a specific aspect of health conditions is collected and stored based on its characteristics, hence has its own standards or protocols. Even within the EHR data, there are multiple controlled vocabularies guiding the processes of capturing and representing the data from different facets of care.  For example, as introduced in ~\cite{Jensen12}, imaging files follow Digital Imaging and Communication in Medicine (DICOM); laboratory data follows Logical Observation Identifiers Names and Codes (LOINC); prescription data follows RxNorm111; Clinical narratives follow International Classification of Disease-9 (ICD‑9) or ICD‑10 or Systematized Nomenclature of Medicine--Clinical Terms (SNOMED CT).
The key challenge is standardizing such multi-format data into an integrated and consistent form that can be exploited using machine learning models. Moreover, in order to better modeling certain health attributes, assumptions made on them by machine learning models, for example distribution assumptions, should be as close to the original data as possible, e.g. continuous variables as Gaussian distributions, categorical variables as multinomial distributions, count variables as Poisson distributions, etc. There is a necessity to develop an integrated machine learning model which can perform learning and inference in such hybrid domains.      

The second challenge is the noise and uncertainty which is a prominent phenomenon in medical domains, such as uncertainty in diagnosis (e.g. patients with the same symptoms could have different pathological causes), uncertainty in heredity (e.g. not all the individuals who have diabetic parent(s) have diabetes), uncertainty in treatment (e.g. patients with similar physiological conditions may have different responses to certain drugs), etc. A popular method of modeling uncertainty is using probabilistic models. Most machine learning methods are either probabilistic models by definition, such as naive Bayes, Bayesian networks, logistic regression and artificial neural networks, or can be extended into probabilistic ones, such as decision trees and support vector machines. By employing the probability theory, machine learning models can capture either the conditional probability of the target variable given the correlated factors (discriminative models) or the probability distribution in the joint space of target and related variables (generative models).

The third challenge is the high dimensionality of the feature space. There are more and more data being collected and believed to be closely related or potentially related to certain health conditions. The dimension of the feature space grows at the exponential rate. In order to train an accurate model (especially for a probabilistic model), machine learning algorithms need a significant amount of observed instances. However, two most common problems with medical research are sparseness (not all the possible configurations of the features have instances in the target cohort or the records accessible by the system) and class imbalance (the number of diagnosed positives is usually much smaller than that of the negatives). Some probabilistic graphical models such as Bayesian networks, Markov models or dependency networks, allow the domain knowledge to be easily incorporated into the model through certain forms of constraints on the structure or parameter space. In probabilistic models, such as Bayesian Networks, domain knowledge that was used to reduce the dimension is the Independence of Casual Influence (ICI) first defined in \cite{heckerman94}, which further reduces the dimension from exponential in sizes of parent sets to linear with the number of parent sets and hence helps with model learning. Another example of applicable domain knowledge is qualitative constraints. For example, the increase of cholesterol levels has a positive effect on the increase of cardiovascular disease risks (monotonic influence). 
Such qualitative constraints on the causal influences can also facilitate more accurate predictions for learning in  sparse data sets as discovered in \cite{Altendorf12}. Yang et. al~\citeyear{YangN13} proposed a method to combine those two types of domain knowledge and proved the improved performance of BN learning for high-dimensional sparse data. There are also numerous research on exploiting domain knowledge in other machine learning models~\cite{Towell94, Fung02, PazzaniK92}.

The forth challenge is the time-line analysis. As mentioned earlier, the features related to an individual's physical and mental status usually change over time, which means that instead of a single value, each variable actually has a sequence of values. These trajectory data generated by such dynamic processes require models which can represent the transition distributions over time. However, the occurrence frequencies of variables are different from each other. For example, the dietary, the exercise level and certain treatments could change every day while the Body Mass Index, blood pressure or cholesterol level would take months or even years to show any significant differences. However, most of the dynamic probabilistic models such as hidden Markov models or dynamic Bayesian networks, model such dynamic processes using a constant interval between adjacent time slices. So, in order to prevent missing any event along the trajectories, they need to model the system at the shortest possible interval. This leads to the increased computational cost when performing inference over time. The even bigger challenge is that many of the features do not have any natural transition rate. This makes it hard to choose the sampling rate in advance. Hence, continuous time probabilistic models have been developed~\cite{ElHayFKK06, NodelmanUAI02, Gunawardana11}, which model the time directly as a continuous variable instead of dissect the trajectory with a certain granularity, hence they are able to capture any time point whenever an interested event happens.

The fifth challenge is the richness of relations. Integrated health-related data of the whole population composes a vast collection of longitudinal multiple granularity profiles. Such data is rich in relations. From the perspective of an individual, the features regarding to one aspect could influence those of other aspects, e.g. the effect of hypertension on the risk of heart attack. From the dynamic system's point of view, the features from previous time points could affect some features in the following time points, e.g. the effect of current treatments on the disease future developments. From the cohort-wide standpoint, some features of one individual could change the features of other people who share similar sociodemographic or genetic factors with him/her, e.g. the heredity of diabetes in some families, or the impact on individual's mental health from other people who are connected with him/her through certain social media groups, in which case the i.i.d. assumption made in traditional machine learning algorithms is violated. In order to learn from the structured data, standard machine learning approaches need to integrate data into a flat table with fixed number of features for each example. But, this requires the employment of an extensive feature engineering process, if not an impossible one. 
For example, consider one's smoking habit which is often affected by his friends' smoking habits. Since the number of friends one can have varies from individual to individual, traditional machine learning methods need to either aggregate the smoking status of all his/her friends into one feature or learn numerous different models separately by training subsets of people who have the same amount of friends. It would be hard for the traditional models to learn more general and meaningful parameters since the data is more sparsely represented in the subsets of the training examples. This is just a simple example to illustrate the case. One may argue that it seems not very difficult for the feature engineering techniques to tackle. Then, consider the case where not only the smoking status needs to be examined but also the attributes of their sociodemographic background or other health behaviors. The difficulty faced by the traditional feature engineering approaches increases as the number of correlated attributes from related objects increases.    
Statistical Relational Learning~\cite{srlbook} extends the propositional models to first-order logic domain. It learns the general rules using inductive logic programming and performs inference either in the grounded graphical model or at the lifted level.

Besides those above, privacy issues also bring challenges. Because of the privacy legislation, many countries require the informed consent for accessing the patients' health records for research. This increases the time and cost to access the data and may also bring in biased data resources which are not representative enough for the larger population. Some technical solutions such as De-identification and Re-identification could circumvent the consent regimes. But, it may lead to the impossibility of certain population wide research~\cite{Jensen12}. 

\section{Thesis Statement}

My thesis investigates the claim:

\begin{itemize}
\item[] \textit{Developing advanced statistical relational learning approaches to allow the human input based learning, temporal modeling and hybrid model development can provide a more effective way to exploit the health-related data, hence facilitate discoveries in medical research.}
\end{itemize}

For class-imbalanced data, an efficient way for approaching the cost-sensitive learning with probabilistic logic models was proposed. Its prediction performance was evaluated in various standard relational domains as well as in a real medical study and a big-scale hybrid recommendation system. For sequential events prediction, the first relational continuous-time probabilistic model which combines first-order logic with standard continuous-time Bayesian networks was presented and an efficient relational learning approach was proposed for it. The performance of standard dynamic probabilistic models was investigated when applied to a real-world cardiovascular disease study. For learning from heterogeneous data, the theoretical functions for a hybrid statistical relational model was derived to handle variables with mixed value types in clinical data which includes continuous, multinomial and count variables. The performance of machine learning models with different probability distribution assumptions was investigated when applied to real electronic health records data for clinical surgery predictions.   

\section{Thesis Outline}

The rest of this thesis is organized as follows: 

Chapter~\ref{chap1} provides the necessary background for the proposed models described in this thesis. It includes the introduction on the theories of Bayesian networks, which gives the necessary knowledge for elucidating dynamic Bayesian networks and continuous-time Bayesian networks following it. Since all the proposed models were developed in the context of statistical relational learning, this section also sets up the basics of statistical relational learning and a state-of-the-art learning approach -- Relational Functional Gradient Boosting (RFGB).  

Chapter~\ref{chap2} presents the theoretical derivation and the algorithm of a cost-sensitive statistical relational learning approach which was developed on top of the RFGB framework. Then, the evaluation metrics which were designed specifically for class-imbalanced data sets are explained. Following it, the performance of the proposed approach is evaluated with standard structured data sets under these evaluation metrics. 

Chapter~\ref{chap3} presents the first relational continuous-time probabilistic model for modeling structured sequence data with irregular time intervals, which is called Relational Continuous-Time Bayesian Networks (RCTBNs). First, the Syntax and Semantics of RCTBNs are presented and its homogeneousness is proved. Then, an efficient approach based on RFGB for learning  RCTBNs is proposed and its convergence property is proved. Before concluding, the experimental results in standard CTBN data as well as relational sequence data are presented.

Chapter~\ref{chap4} presents the theoretical foundation of a proposed statistical relational learning framework for hybrid domains. The theoretical formulations as well as learning procedures of the proposed model under six different conditions specified by the value types of dependent variables and predictor variables are presented.  
  
Chapter~\ref{chap5} investigates the performances of various machine learning models when applied to real-world problems. They include applications of the cost-sensitive statistical relation learning approach presented in Chapter~\ref{chap2} to a large-scale job recommendation system and a rare disease use case; application of dynamic Bayesian networks to a cardiovascular disease study; and prediction of cardiovascular events using various machine learning models with different probability distribution assumptions. 
 
Chapter~\ref{chap6} proposed several possible ways to extend the presented models and explore their potentials on mining health-related data to facilitate an effective clinical decision support system. 



\chapter{Background}
\label{chap1}
\ifpdf
    \graphicspath{{Chapter1/Chapter1Figs/PNG/}{Chapter1/Chapter1Figs/PDF/}{Chapter1/Chapter1Figs/}}
\else
    \graphicspath{{Chapter1/Chapter1Figs/EPS/}{Chapter1/Chapter1Figs/}}
\fi

In this chapter, I will introduce the basic knowledge on dynamic probabilistic models which will be extended to first-order logic forms  and investigated for their application values to real medical studies and EHR data in Chapter~\ref{chap2} and Chapter~\ref{chap5}. I will also present the background on Statistical Relational Learning (SRL) and an efficient relational learning approach--Relational Functional Gradient Boosting (RFGB) which will be adapted to learn the proposed relational models for various practical purposes in this thesis.

\section{Bayesian Networks}
First, I will illustrate Bayesian networks with a simple medical diagnosis case which will be used as a running example in this thesis and introduce the basic notations that will be used through out this thesis.
 
A Bayesian network~\cite{pearlbook} could be represented by the following three indispensable components: i). Each node corresponds to a random variable, which may be discrete or continuous; ii). A set of directed links or arrows connect pairs of nodes, which indicate the direct dependences between the linked variables, the directions of the arcs signify the influence directions; iii). Each node $X_i$ has a conditional probability distribution $P(X_i|Parents(X_i))$ that quantifies the strengths of these influences.

Consider a simple medical diagnosis case where the target variable is the risk of having cardiovascular disease. The influential features include standard medical measurements such as blood pressure, cholesterol level, body mess index and blood glucose, which can be influenced by sociodemographic factors such as age and gender, health behaviors such as dietary and exercise level as well as the family medical history. This example can be represented by a Bayesian network as shown in Figure~\ref{BN}. 
\begin{figure}[htbp]
	\centering
	\includegraphics[scale=0.35]{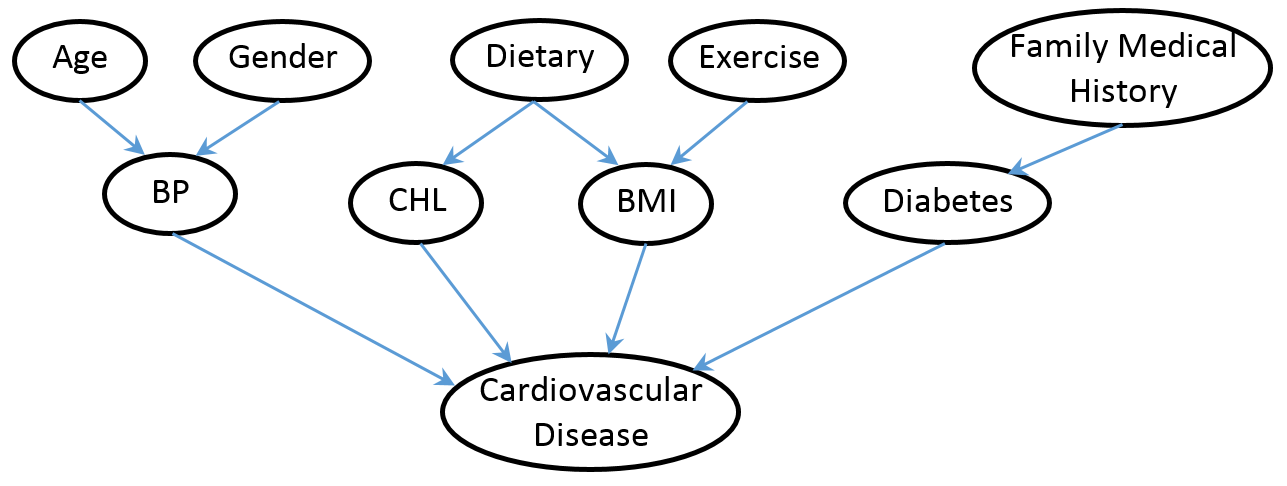}
	\caption{Example BN of Medical Diagnosis Domain}
	\label{BN}
\end{figure}

Now formally introduce some basic notations that are commonly used in BNs. In a Bayesian network with n discrete valued nodes $\textbf{X}=\{X_1,X_2,...,X_n\}$, the structure is denoted as $ \mathcal G$, the parameters by $\theta_{ijk}$ ($i\in\{1,2,...,n\},  j\in \{1,2,...,v_i\}, k\in \{1,2,...,r_i\}$) which means the conditional probability of $X_i$ to take its $k^{th}$ value given the $j^{th}$ configuration of its parents (i.e. $P(X_i^k|pa_i^j)$). $r_i$ denotes the number of states of the discrete variable $X_i$; $\textbf{Pa}(X_i)$ represents the parent set of node $X_i$; the number of configurations of $\textbf{Pa}(X_i)$ is $v_i=\prod_{X_t\in \textbf{Pa}(X_i)}r_t$. 

There is an essential conditional independence implied by Bayesian networks which is the foundation of almost all the mathematical properties, learning and inference methods for Bayesian networks.

\begin{theorem}
	\label{cibn}
	Each node is conditionally independent of its non-descendants given its parents.
\end{theorem}

Because of \textbf{Theorem}~\ref{cibn}, the general chain rule can be further simplified in BNs as below:
\begin{equation}\label{chain}
P(x_1, x_2, ..., x_n)=\prod_{i=1}^n P(x_i | Pa(x_i))
\end{equation}

That is why BNs can concisely represent essentially any full joint probability distribution based on the combination of its structure (topology) and parameters (conditional distributions). And such factored representation can be further simplified by employing the \textit{independence of causal influence} (ICI), which further decreases the dimension of the feature space and hence loosens the requirement on the number of training samples and improves the prediction accuracy of the learned model~\cite{heckerman94}.

Bayesian networks have a fundamental limitation: \textit{the graph cannot have directed cycles}. 
Intuitively, if variables $x_1$ and $x_2$ compose a cycle, any slight evidence that supports $x_1$ would be amplified through $x_2$ and then propagated back to $x_1$, which forges a stronger confirmation possibly without apparent factual justification~\cite{pearlbook}. From the perspective of probability theory, 
If the graph has directed cycles, the joint probability properties would be violated. The proof is straightforward. If assume there is a cycle in the graph where $A \rightarrow B, B \rightarrow A$, then according to the chain rule, $P(a,b)= P(b|a)P(a|b)$. According to Bayesian theory, $P(a,b)= P(b|a)P(a)$, so we have $P(a)=P(a|b)$, which can be satisfied only when the marginal probability $P(a)$ equals to the conditional probability $P(a|b)$, i.e. $a$ is independent with $b$, which conflicts with the original assumed graph where $A$ and $B$ are linked by arcs. This acyclic constraint significantly increases the computational complexity of the structure learning in BNs. 


Given the structure of a Bayesian network, the locally optimized conditional probability distributions that best fit the training data are learned by maximizing the likelihood function of the training data given the proposed model, which equals to solve the following optimization problem:

\begin{equation}\label{bnparal}
\hat{ \theta }= \textit{argmax}_{\theta} \prod_{i=1}^n \prod_{j=1} ^{v_i} \prod_{k=1}^{r_i} \theta_{ijk}^{N_{ijk}}
\end{equation}

Where $N_{ijk}$ represents the number of samples that satisfy the corresponding configuration of ($x_i^k$, $pa_i^j$).

The structure learning of BNs is much more complicated. Theoretically, there could be multiple Bayesian networks that represent the identical full joint probability distributions. Since we do not know the topology ordering of the nodes, the greedy hill-climbing search algorithms could return a BN that fits the training data as good as the optimal one but much more complicated in structure~\cite{pearlbook}. A term for penalizing the model complexity is often included in the score function which is used for evaluating different candidate structures. There are various score functions, such as Bayesian Dirichlet (BDe) scores, Bayesian Information Criterion (BIC) and Mutual Information Test (MIT) which will be discussed in more details in section~\ref{DBNmain}.




The structure learning employs certain search algorithms (usually greedy search) to find the structure with the highest score. It may appear that the calculation of the score function is just a counting problem which simply multiplies the appropriate entries in the conditional probability tables with the corresponding frequency counts derived from the database. However, Chickering et al.~\cite{Chickering94} showed that even assume all the variables can only have no more than 2 parents, the problem of learning the optimal Bayesian network structure is still NP-hard. The reason is that the optimal parent set for each node cannot be determined individually, the choice of parent set for one node affects the choices for others due to the necessity of avoiding creating cycles. Due to this reason, there are many researches focusing on simplifying the structure learning of BNs either by applying constraints from domain knowledge~\cite{deCampos09,CamposJ11} or modifying the basic semantics to allow cycles~\cite{HCMRK2000,Hulten03,TulupyevN05}.

To perform inference over certain queries (denoted as $\textbf{X}_q$) in BNs given some evidence (denoted as $\textbf{X}_e$), the most intuitive algorithm is to use Equation~\ref{chain} to calculate the joint probability with $\textbf{X}_e$ instantiated by the evidence values and then sum over all possible values of the non-evidence, non-query variables. It is called exact inference. \textit{Variable elimination} algorithm can reduce the computation burden to some extent by saving the intermediate results in the factor matrices which avoid re-computation. However, to enumerate all possible configurations of the non-evidence variables and perform the sum-product calculations over them is still unfavorable in large BNs. 

Hence, approximate inference algorithms are preferred when inference in large networks, and among them \textit{Sampling} and \textit{Belief Propagation} are two families of the most popular algorithms. The key idea of \textit{rejection sampling} is to generate samples based on the BN model (both the structure and conditional probability distributions) and reject the ones conflicting with the evidence, then calculate the proportion of the ones consistent with the query in the total samples. This procedure is further improved by \textit{Gibbs Sampling}. Instead of generating each sample by probabilistically assigning values to all non-evidence variables, it generates each sample by making a random change (on one non-evidence variable at a time) to the previously generated sample and keeping the values of other variables unchanged~\cite{Russell2003}. Belief propagation is a message passing algorithm~\cite{pearlbook}, which calculates the belief of the query variables by iteratively computing the beliefs of the variables based on the evidence they accessed through their parents and children nodes and sending messages to parents and children based on the updated beliefs. 

\section{Dynamic Bayesian Networks}

If we consider the medical diagnosis example with an additional dimension -- \textit{time}, we would find that almost all the features related to a person change over time, such as blood pressure, cholesterol level, body mass index, dietary, and etc. A model that can represent the states distributions over time will make best use of the data generated by such dynamic processes. One successful example would be the Dynamic Bayesian network (DBN)~\cite{Dean89}. DBNs model the dynamic process by dividing time into different time slots at a constant rate, and representing the probabilistic transitions of the states between different time slots with a Bayesian network fragment which allows both intra-time-slice arcs (which indicate the conditional dependences among the values of variables at the same instant) and inter-time-slice arcs (which can jump through one or more time slices indicating the conditional dependences among the values of variables over time), as Figure~\ref{DBN} shows.
\begin{figure}[t]
	\centering
	\includegraphics[scale=0.4]{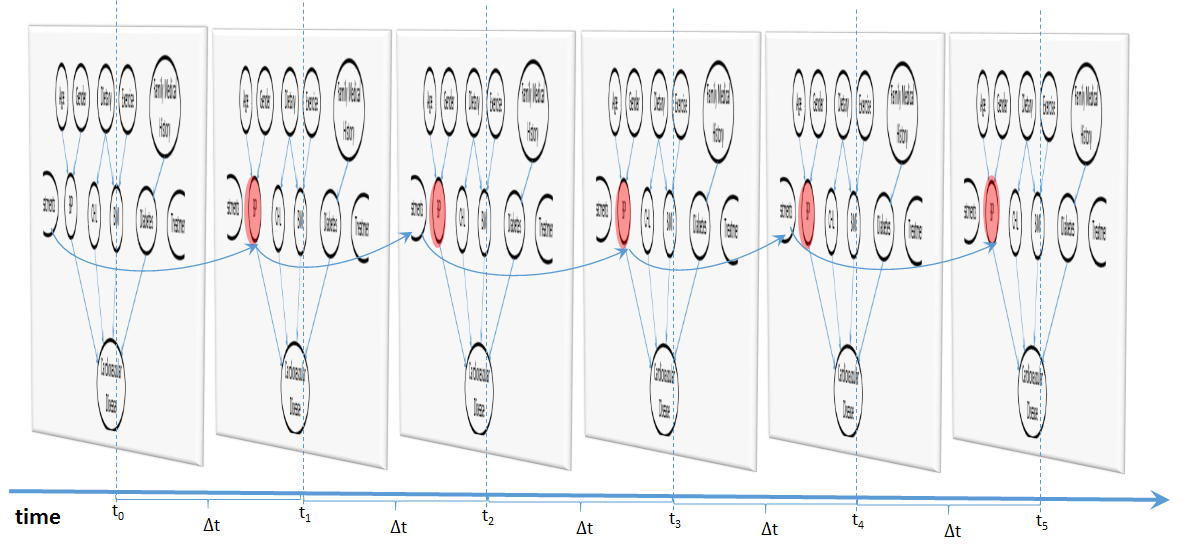}
	\caption{Sample DBNs}
	\label{DBN}
\end{figure}
DBNs can bypass the NP-hard structure learning problem by only allowing for edges between time slices. But the significance of DBNs lies more in its ability to deal with temporal processes.

\section{Continuous Time Bayesian Networks}

As mentioned earlier, DBNs can represent the state at different time points by discretizing time into time slots at specific time points and hence cannot answer queries outside these discretizations. In the medical diagnosis example, however, most of the variables do not have any natural identical changing rate. For example, dietary habits, weight, or smoking habits could stay the same for days, months, years or even decades. But, the effect of certain treatments could arise in minutes, hours, days or even months. They also vary from patient to patient. All events occur at their distinct paces as Figure~\ref{ctbn_samp} shows, where the red node indicates the variable which changes its state at the corresponding time point, for example, $BMI$ changes its state at time $t_1$, $Diabetes$ status transits to another state at time $t_2$, $Cardiovascular Disease$ changes its value at time $t_3$ and $t_1-t_0 \neq t_2-t_1 \neq t_3-t_2$.
\begin{figure}[t]
	\centering
	\includegraphics[scale=0.4]{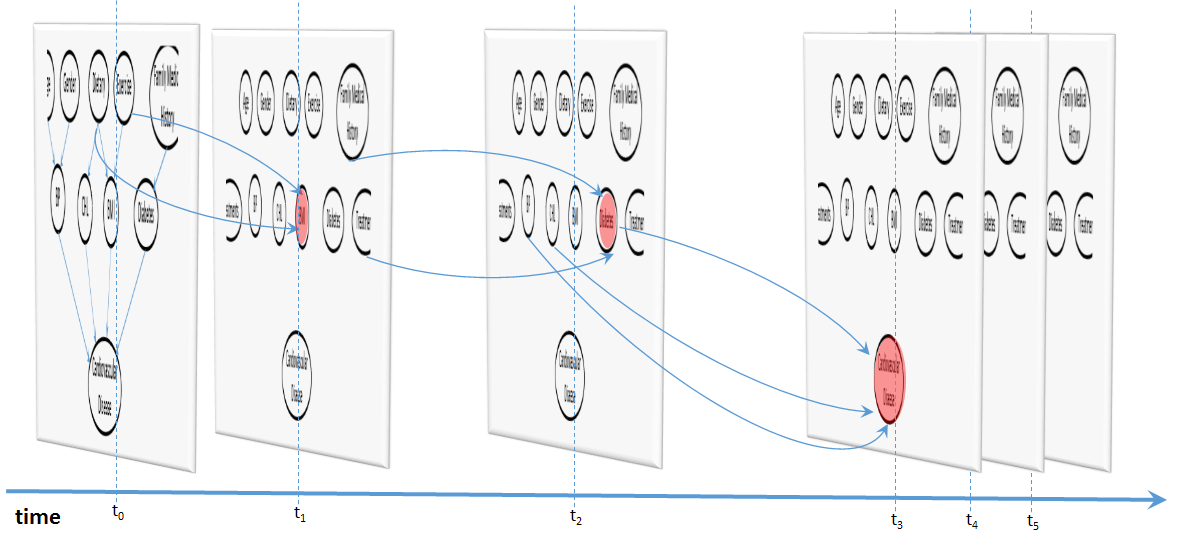}
	\caption{Sample CTBNs}
	\label{ctbn_samp}
\end{figure}

A straightforward way for DBNs to capture the transition probabilities over all the variables is to model the system at the shortest possible period (i.e. highest changing rate). However, this could lead to increased computational cost when performing inference over time especially when there is no evidence observed. Continuous Time Markov Process (CTMP)~\cite{ElHayFKK06} avoids the requirement for choosing the sampling rate by modeling the time directly. However, since it models the transition distributions over the joint states of all variables, the size of the CTMP intensity matrix increases exponentially in the number of variables. With the forte of compact representation carried over from BNs, Continuous Time Bayesian Networks (CTBNs)~\cite{NodelmanUAI02} model the continuous time process as the CTMP family does, but with factored representation which is a more efficient representation for joint probability distributions.

We still use $\textbf{X}=\{X_1,X_2,...,X_n\}$ to denote the collection of random variables, but now each $X_i$ is not a discrete-valued variable but a process variable indexed by time $t\in [0,\infty)$. So, if we want to query about the value of $X_i$ at time t, we use $x_i(t)$ to denote the value. Accordingly, the instantiation to parents of $X_i$ at time t is represented as $Pa_{X_i}(t)$. The instantiation of a particular set of values for $\textbf{X}(t)$ over all t is called a \textit{trajectory}. The key difference to BNs is the use of \textit{conditional intensity matrix} (CIM) instead of the \textit{conditional probability distribution}. Recall that BNs use parameters $P(X_i|pa_i^j)$ which indicates the conditional  distributions of $X_i$ given its parents taking the $j^{th}$ configuration. Since CTBNs model the distribution over the variable's trajectories, a CIM induces a distribution over the dynamics of $X_i(t)$ given the values of $Pa_{X_i}(t)$ (denoted by $pa_i^j$ for simplicity sake), which is defined as:

$Q_{X_i|pa_i^j}=
\begin{bmatrix}
-q_{x_i^1|pa_i^j} & q_{x_i^1x_i^2|pa_i^j} & \cdots & q_{x_i^1x_i^{r_i}|pa_i^j} \\
q_{x_i^2x_i^1|pa_i^j} & -q_{x_i^2|pa_i^j} & \cdots & q_{x_i^2x_i^{r_i}|pa_i^j} \\
\vdots & \vdots & \ddots & \vdots\\
q_{x_i^{r_i}x_i^1|pa_i^j}& q_{x_i^{r_i}x_i^2|pa_i^j} & \cdots & -q_{x_i^{r_i}|pa_i^j} 
\end{bmatrix}$\\

Where $q_{x_i^k|pa_i^j}=\sum_{k'\neq k}q_{x_i^kx_i^{k'}|pa_i^j}$ and $q_{x_i^k|pa_i^j} \geqslant 0, \quad q_{x_i^kx_i^{k'}|pa_i^j} \geqslant 0$.
It could be factored into two components: $\textbf{q}_{X_i} \textendash$ an exponential distribution over when the next transition will occur and $\boldsymbol{\theta}_{X_i}= \frac{\textbf{q}_{X_i^kX_i^{k'}}(k'\neq k)}{\textbf{q}_{X_i^k} } \textendash$ a multinomial distribution over where the state transits, i.e. its next state. For example, $q_{x_i^1|pa_i^j}$ defines that if $X_i$ is at its state 1, it would stay in this state for an amount of time exponentially distributed with parameter $q_{x_i^1|pa_i^j}$, given its parents value being $pa_i^j$, and upon transiting, the probability of $X_i$ transiting from its state 1 to state 2 is $\theta_{x_i^1x_i^2|pa_i^j}$. According to the properties of exponential distributions, the probability density function $f$ and the corresponding cumulative distribution function $F$ for $X_i$ transiting out of $x_i^k$ are given by
\begin{equation}
f(t)= q_{x_i^k|pa_i^j}e^{-q_{x_i^k|pa_i^j}\cdot t}, \quad F(t)= 1-e^{-q_{x_i^k|pa_i^j}\cdot t}, \quad \quad t \geqslant 0 
\end{equation} 

The expected time of transition is $\mathbb{E}_t = \frac{1}{q_{x_i^k|pa_i^j}}$. Upon transitioning, $X_i$ shifts from its $k^{th}$ state to ${k'}^{th}$ state with probability $\frac{q_{x_i^kx_i^{k'}|pa_i^j}}{q_{x_i^k|pa_i^j}}$.

Instead of a single probability distribution for $X_i$  given a specific instantiation of its parents as in BNs, CTBNs have one CIM for every configuration of its parent set. Since the number of configurations of $\textbf{Pa}(X_i)$ equals to $v_i=\prod_{X_t\in \textbf{Pa}(X_i)}r_t$,  there are $v_i$ conditional intensity matrices for $X_i$ in a CTBN model, and each CIM has $r_i\times r_i$ parameters. Because of the constraint $q_{x_i^k|pa_i^j}=\sum_{k'\neq k}q_{x_i^kx_i^{k'}|pa_i^j}$, the total number of independent parameters for a CTBN model over set $\textbf{X}$ is $\prod_{X_i\in \textbf{X}}v_i(r_i\times r_i-r_i)$. For an individual variable $X_i$, its trajectory can be viewed as a conditional Markov process, which is inhomogeneous, because the intensities vary with time. It is worth noting that the CIM is not a direct function of time but rather a function of the current values of its parent variables, which also evolve with time. However, the global process of the variable set $\textbf{X}$ is a homogeneous Markov process. In fact, any CTBN can be converted to a homogeneous Markov process by \textit{amalgamations}~\cite{Nodelmanthesis} over all its CIMs into a joint intensity matrix. Simply put, \textit{amalgamations}  are the summation over the expansions of all the CIMs over the entire variable set. Within the joint intensity matrix, all intensities corresponding to two simultaneous changes are zero because both variables cannot transit at the same instant. 

It is sufficient to construct a CTBN with two components: the initial distribution (denoted as $P_{\textbf{X}}^0$) which is specified with a Bayesian network $\mathcal B$ over \textbf{X} at the start time point and a continuous transition model, specified by: 
\begin{list}{\leftmargin=-0.5em \itemindent=0em}
\item{i.} A directed graph  $\mathcal G$ similar as a BN frame but possibly with cycles.
\item{ii.} A set of CIMs $\textbf{Q}_{\textbf{X}|\textbf{Pa}(\textbf{X})}$, for each variable $X_i\in \textbf{X}$.
\end{list}

\section{Statistical Relational Learning}

The traditional machine learning algorithms have a fundamental assumption about the data they model: the training samples are independent identically distributed (i.i.d.), which is not always the case in the application domains. Actually, in the real world, there are significant amount of data which are not independent among the samples. Still take the medical diagnosis case as an example, if the feature \textit{family medical history} is not given but need to be learned from family members' data  which are also represented as samples in the same form as the patient's itself as Figure~\ref{srl_samp} shows. It is clear that the training samples from the same family are related with each other, which violates the i.i.d. assumption. 
\begin{figure}[t]
	\centering
	\includegraphics[scale=0.4]{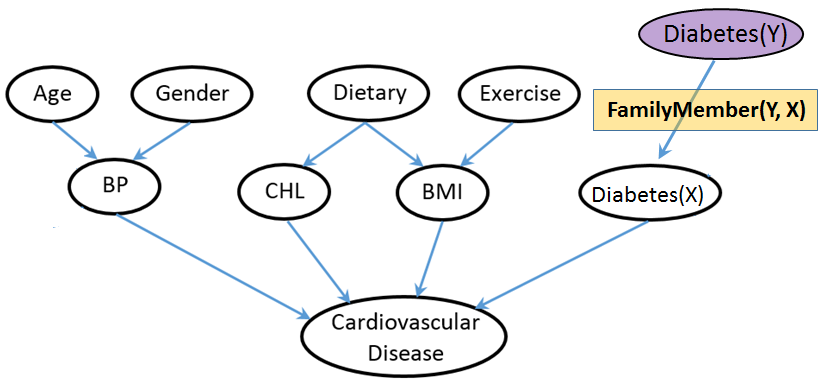}
	\caption{Sample Relational Domain}
	\label{srl_samp}
\end{figure}

Taking the related objects into consideration can potentially improve the prediction performance and reveal new insights. But, how to deal with such structured data? A straightforward strategy would be to introduce the relevant features of the related objects into the feature space of the target object. If there are $n$ attributes with heredity nature and $m$ family members, then there would be $m\times n$ more features added to the feature space of the target sample. In that case, not only the computation complexity but also the representation of the data would be problematic. This is because the propositional algorithms require a regular flat table to represent the data. However, the number of members in one's family varies from individual to individual. For example, one patient could have three ascendants who have diabetes and the other only has one, the number of features after combined with relatives' features would be different between them. There are aggregation-based strategies~\cite{Perlich03} to make the feature number identical to all the samples when combining the feature spaces. 
However, it would inevitably lose useful information or add noise. 

Statistical relational learning~\cite{srlbook} extends propositional models to the first-order logic domain. It learns the general rules using inductive logic programming and performs inference either in the grounded (base-level) graphical model or at the lifted level. SRL addresses the challenge of applying statistical learning and inference approaches to problems which involve rich collections of objects linked together in a complex, stochastic and relational world. 
Statistical relational models include directed models such as Bayesian Logic Programs (BLPs)~\cite{blp}, Relational Bayesian Networks (RBNs)~\cite{rbn} and Logical Bayesian Networks (LBNs)~\cite{lbn}, undirected models such as Markov Logic Networks (MLNs)~\cite{Richardson06}, bi-directed models like Relational Dependency Networks (RDNs)~\cite{Neville07}, and tree-based models such as logical decision trees~\cite{BlockeelR98}, etc. Almost all kinds of graphic models have their corresponding extensions in the relation domain. 
The advantage of SRL models is that they can succinctly represent probabilistic dependencies among the attributes of different related objects, leading to a compact representation of learned models.
Similar to standard graphical models, most of the SRL models have two components: structures which are represented as a set of first order logic rules to capture the dependences among the logic atoms and parameters which are weights or probability distributions to calculate the probability of a world given the satisfiability of the formulae. 

While these models are highly attractive due to their compactness and comprehensibility, the problem of learning them is computationally intensive.
Consequently, structure learning has received increased attention lately, particularly in the case of one type of SRL models called Markov Logic Networks~\cite{ecai08,hypergraph09,structuralMotifs10,bottomupmln07}. These approaches provide solutions that are theoretically interesting. However, applicability to real, large applications is nominal due to restricting assumptions in the models and complex search procedures inside the methods. A more recent algorithm called Relational Functional Gradient Boosting (RFGB)~\cite{Natarajan2012,icdm11}, based on Friedman's functional gradient boosting~\cite{friedman01}, addressed this problem by learning structures and parameters simultaneously. This was achieved by learning a set of relational trees for modeling the distribution of each (first-order) predicate given all the other predicates. The key insight is to view the problem of learning relational probabilistic functions as a sequence of relational regression problems, an approach that has been proved successful in the standard propositional data case. 
In turn, it is natural to expect that problems and solutions well known for propositional classification/regression problems can be studied for and transferred to SRL settings. Interestingly, this has not been fully considered so far. In this thesis, we focus on extending those propositional models/learning approaches to relational settings. Specifically, we consider the cost-sensitive learning approach for mining class-imbalanced data, the continuous-time model for learning from longitudinal data with irregular time intervals and the hybrid model for modeling heterogeneous data. 

\section{Relational Functional Gradient Boosting}

Standard gradient ascent approach starts with initial parameters $\theta_0$ and iteratively adds the gradient ($\Delta_i$) of an objective function w.r.t. $\theta_i$.  
Friedman \cite{friedman01} proposed an alternate approach where the objective function is represented using a regression function $\psi$ over the examples $\mathbf{x}$ and the gradients are performed with respect to $\psi(x)$. Similar to parametric gradient descent, after $n$ iterations of functional gradient-descent, $\psi_n(x)$ $=$ $\psi_0(x) + \Delta_1(x)$ $+$ $\cdots$ $+$ $\Delta_n(x)$.
 
Each gradient term ($\Delta_m$) is a set of  training examples and regression values given by the gradient w.r.t $\psi_m(x_i)$, i.e., $<x_i , \Delta_m(x_i) = \frac{\partial LL(\mathbf{x})}{\partial \psi_m(x_i)}>$. To generalize from these regression examples, a regression function $\hat{\psi}_m$ (generally regression tree)  is learned to fit to the gradients. The final model $\psi_m = \psi_0 + \hat{\psi}_1 + \cdots + \hat{\psi}_m$ is a sum over these regression trees. Functional-gradient ascent is also known as functional-gradient boosting (FGB) due to this sequential nature of learning models.

FGB has been applied to relational models~\cite{Natarajan2012,karwath08,PolicyGradient08,ijcaiImitation} due to its ability to learn the structure and parameters of these models simultaneously. Gradients are computed for relational examples which are groundings/instantiations (e.g., \textit{Diabetes(Jack)}) of target predicates/relations (e.g. \textit{Diabetes(X)}). Relational regression trees[~\cite{BlockeelR98}] are learned to fit the $\psi$ function over the relational regression examples. Since the regression function $\psi: X \rightarrow (-\infty, \infty)$ is unbounded, sigmoid function over  $\psi$  is commonly used to represent conditional distributions. Thus the RFGB log-likelihood function is:
\begin{equation}
\label{ll_rfgb}
\log J = \sum_i [ \psi(y_i=\hat{y}_i;\textbf{X}_i)-\log\sum_{y^\prime_i} e^{\psi(y^\prime_i;\textbf{X}_i)}]
\end{equation}
where $y_i$ corresponds to a target grounding (a ground instance of the target predicate) of example $i$  with parents set $\mathbf{X}_i$ which could be groundings of any first order variable that are stochastically correlated with the target grounding. $\hat{y}_i$ is the true label for the $i^{th}$ example which is $1$ for positive examples and $0$ for negative examples. $y^\prime_i$ iterates over all possible values of the target grounding $y_i$ which is $\{0, 1\}$ for binary valued predicates. 

The gradients with respect to the potentials (known as {\it functional gradients}) for each example can be shown to be:
\begin{equation} \label {hardG}
\dfrac{\partial \log J}{\partial \psi(y_i=1;\textbf{X}_i) } 
 = I(\hat{y}_i=1)- P(y_i=1;\textbf{X}_i)
\end{equation}
which is the difference between the true distribution ($I$ is the indicator function) and the current predicted distribution. Note the indicator function, $I$  returns $1$ for positive examples and $0$ for negative examples. Hence the positive gradient term for positive examples pushes the regression values closer to $\infty$ and thereby probability closer to 1, whereas for negative examples, the regression values are pushed closer to $-\infty$ and probabilities closer to $0$.




\chapter{Statistical Relational Learning for Imbalanced Data}
\label{chap2}
\ifpdf
\graphicspath{{Chapter2/Chapter2Figs/PNG/}{Chapter2/Chapter2Figs/PDF/}{Chapter2/Chapter2Figs/}}
\else
\graphicspath{{Chapter2/Chapter2Figs/EPS/}{Chapter2/Chapter2Figs/}}
\fi
 
Class-imbalance has remained a phenomenal and prevalent problem in statistical relational learning. This is due to the fact that in relational settings, only a limited number of relations actually exist among the objects, i.e. the number of ground facts is fairly a small portion of all the ground predicates. With regard to the training instances, negative examples far outnumber positive examples. For instance, consider the relation {\tt FamilyMember(Y, X)} in Figure~\ref{srl_samp}. If there are 1000 objects in the domain, which means {\tt X} and {\tt Y} both have 1000 different instances, the {\tt FamilyMember(Y, X)}  relation predicate can have 1M different ground instances with different configurations of X and Y. However, in real world scenario, most of {\tt FamilyMember} relationships are not true for those grounded instantiations. Such imbalance between positive and negative examples could even worsen as the number of instances for the first order logic atom increases. This problem is a critical one in all SRL domains involving relations between people, such as {\tt co-worker}, {\tt advised-by}, {\tt co-authors}, etc. The issue we address is directly faced by many SRL models.
Probabilistic Relational Models ~\cite{prm} bypasses this problem by creating a binary existence variable for every possible relation, but this introduces a similar class imbalance problem while learning the existence variables because this binary relation will be false for all but a very few number of grounded relations. 
 
The approach employed by more successful SRL algorithms such as RFGB is to sub-sample a set of negatives, and use this subset for training. While effective, this method can lead to {\it a large variance} in the resulting probabilistic model. An alternative method typically is to oversample the minority class but as observed by Chawla~\cite{chawla10}, this can lead to overfitting on the minority class. This is particularly true for incremental model-building algorithms such as RFGB, which iteratively attempt to fix the mistakes made in earlier learning iterations. In the propositional domains, this problem is typically addressed by operating and sampling in the feature space~\cite{chawla10}. However, in relational domains, the feature vector is not of a fixed length and the feature space can possibly be infinite~\cite{jensen02}. Previous work has shown that random sampling of features will not suffice in relational domains~\cite{khot14}.
 
One common solution in the propositional world to address the issue of class imbalance leading to overfitting is to perform some form of {\it margin maximization}. A common approach to margin maximization is via regularization, typically achieved via a regularization function \cite{saha13,hastie,schapirebook}. In propositional and relational functional-gradient boosting methods, common regularization approaches include restricting number of trees learned (number of iterations) or tree size. While reasonably successful, algorithm behavior is  sensitive to parameter selection, and requires additional steps such as cross validation to achieve optimal model tuning. Instead, we explore the use of {\em cost-based soft-margin maximization}, where a carefully designed cost function relaxes the hard margin imposed by maximum log-likelihood by penalizing certain misclassified examples differently.
 
Based on the observation that the current RFGB method simply optimizes the log-likelihood by considering the underlying model to be a (structured) log-linear model, we introduce a soft-margin objective function that is inspired by earlier research in log-linear models~\cite{Gimpel2010,sha+saul06}. The objective is very simple: high-cost examples should be penalized differently from low-cost examples. In relational domains, miss-classified negatives constitute low-cost examples, since the remarkably larger quantity of negative examples would easily dominate the boundary of the classifier after a few iterations. 
Our cost function addresses this class imbalance, essentially by placing different weights on the false negative and false positive cases. Thus, we propose a soft-margin function, or more specifically, a cost-augmented scoring function that treats positive and negative examples differently.
 
Our proposed work makes several key contributions: it addresses the {\it class imbalance problem} in relational data by introducing a soft-margin-based objective function. Then, it presents the derivation of gradients for this soft margin objective function, from which {\it it learns conditional distributions}. It also shows the relationship between the soft margin and the current RFGB algorithm by mapping one to another. It should be mentioned that since the original RFGB algorithm has been extended to learn several types of directed and undirected models~\cite{Natarajan2012,icdm11,ijcaiImitation}, our algorithm is broadly applicable and not restricted to a particular class of SRL model. Finally, we evaluate the algorithm on several standard data sets and demonstrate empirically that the proposed framework outperforms standard hard-margin RFGB algorithm in addressing the problem of class imbalance.
 
The rest of this chapter is organized as follows: First, the soft margin framework is presented and its functional and mathematical properties were analyzed, then the experimental setup is outlined for 5 standard relational data sets and the results are presented. Finally, the chapter is concluded with lists of some interesting avenues for future research. 
\section{Soft-Margin RFGB}
\label{softboost}
In order to approaching the cost-sensitive learning for probabilistic logic models, a cost function (soft-margin) term has been introduced into the objective function for RFGB. Standard RFGB approaches maximize the pseudo-loglikelihood given by Equation~\ref{ll_rfgb}.
The key assumption is that the conditional probability of a target grounding $y_i$, given all the other predicates is modeled as a sigmoid function: 
\begin{equation}
\label{sigmoid}
P(y_i;\textbf{X}_i)= \frac{1}{1 + \exp(-y_i\cdot\psi(y_i;\textbf{X}_i))}.
\end{equation}
The point-wise gradient with respect to the functional parameters can be derived as: 
$\Delta(y_i) \, = \, \textit{I}(\hat{y}_i=1) - \textit{P}(y_i=1;\textbf{x}_i)$.
When a positive (or negative) example is predicted to be true with probability $0$ (or $1$), the gradient will be $1$ (or $-1$). 
This gradient ensures that the probability (via the likelihood) of all positive examples is pushed towards $1$, and all the negative examples towards $0$.

While this appears to be reasonable, a closer inspection illustrates two key problems with RFGB.  First, as more trees being learned, weights are concentrated on outliers, which leads to over-fitting. 
In each iteration, outliers are misclassified and will consequently have larger gradients; this leads to the algorithm focusing on them in the subsequent iteration and resulting in a more complicated model than might be necessary. To avoid this, regularization strategies such as limiting the number of trees or tree size are usually implemented. While this is usually effective, this requires careful tuning of parameters from a large pool of potential candidates in order to learn the optimal classifier.

Second, RFGB treats both positive and negative examples equally. This is evident from the fact that the magnitude (absolute value) of the point-wise gradient does not depend on the true label of the example. For skewed data sets, such as ones common in relational domains, we may want to assign higher weights to examples of the minority class. For instance, as the number of positive examples is very small, we might wish to assign higher weights to positive examples in order to reduce false negatives. To alleviate this issue, RFGB currently sub-samples a smaller subset of negative examples while training. Again, while this is reasonable, the resulting model usually has high variance due to these varying sets of negative examples.



On the other hand, a cost-sensitive approach allows to address these issues and model the target task more faithfully. This is achieved by adding a cost function to the objective function that penalizes positive and negative examples differently. Thus, instead of just optimizing the log-likelihood of examples, the algorithm also takes into account these additional costs in order to account for the class imbalance.

Following the work of Gimpel and Smith~\citeyear{Gimpel2010}, a cost function is introduced into the objective function, which is defined as:
\begin{equation*}
c(\hat{y}, y)= \alpha \, I(\hat{y}=1\wedge y=0) \, + \, \beta I(\hat{y}=0 \wedge y=1),
\end{equation*}
where $\hat{y}_i$ is the true label of the $i^{th}$ instance and $y_i$ is the predicted label. $I(\hat{y}=1 \wedge y=0)$ is $1$ for false negative examples and $I(\hat{y}=0 \wedge y=1)$ is $1$ for false  positive examples. Intuitively, $c(\hat{y}, y) $ equals to $\alpha$ when a positive example is misclassified, while $c(\hat{y}, y) $ equals to $\beta$ when a negative example is misclassified, and $0$ for cases of $I(\hat{y}=0 \wedge y=0)$ and $I(\hat{y}=1 \wedge y=1)$ .

This cost function was hence being introduced into the normalization term of the objective function as:
\begin{equation*}
\log J=  \sum_i  [\psi(y_i;\textbf{X}_i)-\log\sum_{y^\prime_i} e^
{ \psi(y^\prime_i;\textbf{X}_i) + c(\hat{y}_i,y^\prime_i) }],
\end{equation*}
where $\hat{y}_i$ is the true label for the $i^{th}$ example and $y^\prime_i$ iterates over all possible values of the target grounding $y_i$. 
Thus, in addition to a simple log-likelihood of examples, the algorithm also takes into account these additional costs.

The gradient of the objective function w.r.t $\psi(y_i=1;\textbf{X}_i)$ can be derived as:
\begin{align*}
& \dfrac{\partial \log J}{\partial \psi(y_i=1;\textbf{X}_i) } \nonumber\\
& = I(\hat{y}_i=1)- \dfrac{\exp \left\{ \psi(y_i=1; \textbf{X}_i) +c(\hat{y}_i,y_i=1)\right\}}{\sum_{y^\prime_i}\exp\left\{\psi(y^\prime_i; \textbf{X}_i)+c(\hat{y}_i,y^\prime_i)\right\}}.
\end{align*}

Dividing the numerator and denominator by the normalization term $Z_i=\sum_{y^\prime_i}e^{\psi(y^\prime_i;\textbf{X}_i)}$:
\begin{align*}
& \dfrac{\partial \log J}{\partial \psi(y_i=1;\textbf{X}_i) } \nonumber\\
&  = I(\hat{y}_i=1;\textbf{X}_i)- \dfrac{e^ {\psi(y_i=1;\textbf{X}_i)+c(\hat{y}_i,y_i=1)}/Z_i}{\sum_{y^\prime_i}e^{\psi(y^\prime_i;\textbf{X}_i)+c(\hat{y}_i,y^\prime_i)}/Z_i} \nonumber\\
& = I(\hat{y}_i=1;\textbf{X}_i)- \dfrac{P(y_i=1;\textbf{X}_i)e^ {c(\hat{y}_i,y_i=1)}}{\sum_{y^\prime_i}[P(y^\prime_i;\textbf{X}_i)e^{c(\hat{y}_i,y^\prime_i)}]}
\end{align*}

To compactly present the gradients of the objective function, define the cost parameter as:
\begin{equation*}
\lambda \, = \, \dfrac{e^ {c(\hat{y}_i, y_i=1)}}{\sum_{y_i^\prime}[P(y_i^\prime; \textbf{X}_i) \,\, e^{c(\hat{y}_i,y_i^\prime)}]},
\end{equation*}

which equals to:
\begin{equation}
\label{softrfgb_lambda}
\begin{array}{l}
\lambda(y_i)\, = 
\left\{ \begin{array}{ll} \displaystyle{\frac{1}{P(y_i=1;\textbf{X}_i)+P(y_i=0;\textbf{X}_i)\cdot e^{\alpha}}},  & \textrm{if}\,\, \hat{y}_i=1 \\ [10pt]
\displaystyle{\frac{ e^{\beta}}{P(y_i=1;\textbf{X}_i)\cdot e^{\beta} +P(y_i=0;\textbf{X}_i)}},   & \textrm{if} \,\, \hat{y}_i=0 
\end{array} \right.
\end{array}
\end{equation}
Hence, the gradients of the objective function can be rewritten compactly as
\begin{equation}
\Delta(y_i) = I(\hat{y}_i=1) \, - \, \lambda P(y_i=1;\textbf{X}_i).
\label{simp_grad}
\end{equation}

To explicitly present the gradient functions, its different formulae for positive and negative examples are shown  separately as below:
\begin{equation*}
\begin{array}{l}
\Delta(y_i)\, = 
\left\{ \begin{array}{ll} 1 - \displaystyle{\frac{P(y_i=1; \textbf{x}_i)}{P(y_i=1;\textbf{X}_i)+P(y_i=0;\textbf{X}_i)\cdot e^{\alpha}}},  & \textrm{if}\,\, \hat{y}_i=1 \\ [10pt]
0 - \displaystyle{\frac{\textit{P}(y_i=1;\textbf{x}_i)\cdot e^{\beta}}{P(y_i=1;\textbf{X}_i)\cdot e^{\beta} +P(y_i=0;\textbf{X}_i)}},   & \textrm{if} \,\, \hat{y}_i=0 
\end{array} \right.
\end{array}
\end{equation*}

Algorithm \ref{softrfgb} shows the proposed approach. It iterates through M steps. In each iteration, it generates examples based on the soft-margin gradients; then uses  \textsc{FitRelRegressionTree} to learn a relational regression tree to fit the examples  which is added to the current model. The maximum leaves a tree can have is limited to $L$  and  the best node is greedily picked to expand the tree. 

Function \textsc{GenSoftMEgs} iterates through all the examples (N in the algorithm) to generate the regression examples. For each example, it calculates the probability of the example being true given the current model ($P(y_i=1|\mathbf{x}_i)$). It then calculates the parameter $\lambda$ for positive and negative examples respectively~\ref{softrfgb_lambda}, then the gradient based on the simplified formula~\ref{simp_grad}. The example and its gradient are added to the set of regression examples, $S$.

\begin{figure}
	\renewcommand\figurename{Algorithm}
	\caption{Soft-RFGB: RFGB with Soft-Margin}
	\label{softrfgb}
	\begin{algorithmic}[1]
		\Function{SoftRFGB}{$Data$}
		\For {$1 \leq m \leq M$} \Comment{Iterate through M gradient steps}
		\State $S := $\textsc{GenSoftMEgs}$(Data; F_{m-1})$ 
		\Comment{Generate examples}
		\State $\Delta_{m}:=$\textsc{FitRelRegressTree}$(S)$ 
		\Comment{Relational Tree learner}
		\State $F_m := F_{m-1} + \Delta_{m}$ 
		\Comment{Update model}
		\EndFor
		\EndFunction\\
		
		\Function{GenSoftMEgs}{$Data,F$}
		\Comment{ Example generator}
		\State $S := \emptyset$
		\For { $1 \leq i \leq N$} 
		\Comment{Iterate over all examples}
		\State $p_i = P(y_i=1|\mathbf{x}_i) $ 
		\Comment{Probability of the example being true}
		\If {$\hat{y}_i=1$}
		\State $\lambda = \frac{1}{p_i + (1 - p_i)\cdot e^{\alpha}}$
		\Comment{Compute parameter $\lambda$  for positive examples}
		\Else
		\State $\lambda = \frac{1}{p_i + (1 - p_i)\cdot e^{-\beta}}$
		\Comment{Compute parameter $\lambda$  for negative examples}
		\EndIf
		\State $\Delta(y_i;\mathbf{x}_i) := I(\hat{y}_i=1) - \lambda P(y_i = 1|\mathbf{x}_i)$ 
		\Comment{Gradient of current example }
		\State $ S := S \cup {[y_i, \Delta(y_i;;\mathbf{x}_i)]}$ 
		\Comment{Update relational regression examples}
		\EndFor
		\\
		\Return $S$ \Comment{Return regression examples}
		\EndFunction\\
		
		\Function{FitRelRegressionTree}{$S$}
		\Comment{ Relational tree Learner}
		\State Tree := createTree($P(X)$)
		\State Beam := \{root(Tree)\}
		\While {$numLeaves(Tree) \le L$}
		\State Node := popBack(Beam)\Comment{Node w/ worst score}
		\State C := createChildren(Node) \Comment{Create children}
		\State BN := popFront(Sort(C, S)) \Comment{Node w/ best score}
		\State addNode(Tree, Node, BN)
		\Comment{Replace Node with BN}
		\State insert(Beam, BN.left, BN.left.score) \Comment{Insert branch}
		\State insert(Beam, BN.right, BN.right.score)
		\EndWhile \\
		\Return Tree
		\EndFunction
		
	\end{algorithmic}

\end{figure}

Now, let us have a closer look at the functional properties of the proposed soft-margin RFGB. In the cost-sensitive gradients derived above, the cost parameter $\lambda$ depends on the parameters $\alpha$ and $\beta$.  
When $\alpha\, = \,\beta \, = \,0$, $\lambda=1 \, / \, \sum_{y^\prime} \, [P(y^\prime; \textbf{X}_i)] = 1$, and the original RFGB gradients are recovered; this setting corresponds to ignoring the cost-sensitive term.

For positive examples, we have
\begin{equation}
\lambda= \frac{1}{P(y_i=1;\textbf{X}_i)+P(y_i=0;\textbf{X}_i)\cdot e^{\alpha}}. \nonumber
\end{equation}
As  $\alpha \rightarrow \infty \Rightarrow e^{\alpha} \rightarrow \infty \Rightarrow \lambda \rightarrow 0$, 
hence the gradient (which equals to $\Delta(y_i) = 1 \, - \, \lambda P(y_i=1;\textbf{X}_i)$) ignores the predicted probability (as long as it is not consistent with the true label), and is being pushed closer to its maximum value $1$ (i.e. $\Delta \rightarrow 1$). This amounts to putting a large positive cost on the false negatives since the gradients have been pushed to its maximal possible value. 

On the other hand, as $\alpha \rightarrow -\infty \Rightarrow  e^{\alpha} \rightarrow 0 \Rightarrow \lambda \rightarrow \frac{1}{P(y_i=1;\textbf{X}_i)}$, $\Delta \rightarrow 1- 1 = 0$, 
which functions as putting no cost on the false negatives (no matter how far the predicted probability is from the true label) since the gradients are pushed closer to their minimum value of 0. 

Similarly, for negative examples, it can be shown that when $\beta \rightarrow \infty$, the gradient $\Delta \rightarrow -1$ and the boosting algorithm puts a large negative weight on the false positives (the absolute value of the gradient controls the moving distance and the sign of the gradient controls the moving direction), whereas as $\beta \rightarrow -\infty$, the gradient $\Delta \rightarrow 0$ and it puts no cost on the false positives. 

Generally, if $\alpha <0$ ($\beta < 0$), the algorithm is more tolerant of misclassified positive (negative) examples. Alternately, if $\alpha > 0$ ($\beta > 0$), the algorithm penalizes misclassified positive (negative) examples even more than standard RFGB. Thus, the influence of positive and negative examples on the final learned model boundary can be directly controlled by tuning the parameters $\alpha$ and $\beta$.

Simply put, if the correct classification of positive examples is very important in a certain domain, one can emphasize such importance by assigning positive values to $\alpha$. If there are many outliers with positive labels which could lead to over-fitting, one can soften the margin by setting $\alpha$ to negative, making the algorithm more tolerant. The choice of $\beta$ has a similar effect on the negative examples. Soft-RFGB {\it allows flexible adjustments to classification boundaries} in various domains by defining the cost function with two parameters ($\alpha$ and $\beta$) which control the gradients of misclassified positive and negative examples respectively. Returning to our recurring example of a class-imbalanced relational domain, we wish to correctly classify as many positive examples as possible, while at the same time, avoid over-fitting on the negative examples. In such cases, we set $\alpha > 0$ and $\beta < 0$ whose effectiveness have been explored in experiments with 5 standard statistical relational learning data sets.

Now, let us investigate the mathematical properties of the proposed soft-margin RFGB.
\begin{theorem}
	There exist global maxima for the loglikelihood function in soft-margin RFGB and the algorithm converges at the true labels of the examples.
\end{theorem}

\begin{proof}
   Let $Z_i$ denote the normalization term $Z_i=\sum_{y^\prime} e^{\psi(y_i = y^\prime; \textbf{X}_i)}$ and $p^+$ denote the predicted probability for positive examples, i.e., $p^+ = \dfrac{e^{ \psi(y_i=1; \textbf{X}_i)}}{Z_i}$.
   \begin{align*}
   \frac{\partial \ \ p^+}{\partial \ \ \psi(y_i=1; \textbf{X}_i) } & = \frac{e^{ \psi(y_i=1; \textbf{X}_i)} * Z_i - e^{ \psi(y_i=1; \textbf{X}_i)} * e^{ \psi(y_i=1; \textbf{X}_i)}}{Z_i^2} \\
   & = p^+ (1- p^+)
   \end{align*}
   For positive examples, the Hessian function w.r.t. $\psi(y_i=1; \textbf{X}_i)$ equals to:
   \begin{align*}
   H^+ & = \frac{\partial \ \ \Delta^+ }{\partial \ \ \psi(y_i=1; \textbf{X}_i)} = \frac{\partial \ \ \Delta^+ }{\partial \ \ p^+} * \frac{\partial \ \ p^+}{\partial \ \ \psi(y_i=1; \textbf{X}_i)} \\
   & = \frac{-e^{\alpha}}{[p^+ + (1-p^+)e^{\alpha}]^2} *  p^+ (1- p^+)
   \end{align*}
   	As $p^+ \in [0,1]$, $(1-p^+) \in [0,1]$, and that we also have $-e^{\alpha} < 0$ and $[p^+ + (1-p^+)e^{\alpha}]^2 \geq 0 $, $H^+$ is non-positive for the interval of $p^+$ as [0, 1]. 
   	
   For negative examples, the Hessian function w.r.t.  $\psi(y_i=1; \textbf{X}_i)$ equals to:
   \begin{align*}
   H^- & = \frac{\partial \ \ \Delta^- }{\partial \ \ \psi(y_i=1; \textbf{X}_i)} = \frac{\partial \ \ \Delta^- }{\partial \ \ p^+} * \frac{\partial \ \ p^+}{\partial \ \ \psi(y_i=1; \textbf{X}_i)} \\
   & = \frac{-e^{-\beta}}{[p^+ + (1-p^+)e^{-\beta}]^2} *  p^+ (1- p^+)
   \end{align*}
   Similarly, we also get $H^-$ is non-positive everywhere for the interval $p^+ \in [0,1]$. Hence, the log-likelihood function of $\psi(y_i=1; \textbf{X}_i)$ is a concave function and there exists the global optimum. \\ 
   
   When the algorithm converges, $\Delta = 0$. 
   For positive examples, $1-\frac{p^+}{p^+ +(1-p^+) e^{\alpha}}= 0  \ \ \Rightarrow$
   $ (1-p^+) e^{\alpha} = 0$. Since $e^{\alpha} > 0$, we have $1-p^+ = 0$, i.e., the predicted probability for positive examples being true equals to 1 when the algorithm converges. 
   Similarly, For negative examples, $-\frac{p^+}{p^+ +(1-p^+) e^{-\beta}}= 0  \ \ \Rightarrow \ \ p^+ =0$, i.e., the predicted probability for negative examples being true equals to 0 when the algorithm converges.
         
\end{proof}
	
\section{Experiments of Soft-Margin RFGB on Simulated Relational Data}
\label{experiments}
This chapter presents the evaluation of the proposed Soft-RFGB approach, including detailed evaluation criteria, experimental setup and discussion of the experimental results on five standard SRL data sets. 

\subsection{Domains}
Four standard relational learning domains (Cora, IMDB, UW, WebKB) and a propositional data set (Heart Disease) are employed  for empirical evaluation. Table~\ref{domains_softrfgb} shows important details of the data sets used. The target predicate for each domain is provided in the second column. 
The third column is the size of the data set which equals to the number of facts based on which the target predicate is predicted. The last column denotes the ratio of negative to positive examples, i.e. $\#neg:\#pos$. 

Details on the contents of these datasets are as follows:
\begin{itemize}
	\item \textit{Cora} is a citation matching data set where the goal is to identify a group of citations that refer to the same paper. This prediction is based on the title of the paper, the authors, the venue of the publication etc. 
	\item In \textit{IMDB}, the goal is to predict the gender of a person involved in the movie from information about the movies that the person has worked on, such as their directors, genre, etc.
	\item \textit{Heart} contains heart disease data from the UCI repository, and aims to predict whether a person is susceptible to heart attack given information such as their gender, age, cholesterol levels, the ECG information, etc.
	\item The goal in \textit{UW} is to predict who the advisor of a student is based on information such as their common publications, the courses they have taught and TAed, their levels in the department and program, etc. 
	\item \textit{WebKB} task aims to predict the TA of a course based on information extracted from the faculty, course and project web pages.
\end{itemize}
Note that the datasets are sorted in decreasing order of balance between positive and negative examples.

\begin{table}[h]	
	\caption{Details of the experimental domains.}
	\begin{center}
		\begin{tabular}{|c|c|c|c|c|c|}
			\hline
			\multirow{2}{*}{Domain} & Target  & Num of & Num of & Num of  & Imb. \\
			~ & Predicate & facts & pos & neg  &  ratio \\
			\hline 
			Cora & $samebib$ & 6731 & 30971 & 21952 & 0.71:1 \\
			IMDB & $female\_gender $ & 959 & 95 & 173 & 1.8:1\\
			Heart & $num$ & 7453 & 265 & 655 & 2.5:1\\
			UW & $advisedBy$  & 5039 & 113 & 54729 &  484:1  \\
			WebKB & $courseTA$ & 1912  & 121 & 71095 & 588:1 \\
			\hline
		\end{tabular}
	\end{center}	
	\label{domains_softrfgb}
\end{table}

\subsection{Evaluation Metrics}
Standard evaluation metrics on relational models include the use of Area Under ROC or PR curves (AUC-ROC or AUC-PR), $F_1$ score, etc., which measure accuracy with balanced weight between positive and negative examples.  Instead, we address domains where misclassification of positive instances (false negatives) costs significantly more than a false alarm (false positives), e.g., in medical diagnosis, security detection, etc.  In such domains, the model should identify as many positive cases as possible as long as the precision stays within a reasonable range.  

To better serve such a goal, an evaluation metric which assigns {\it higher weights to high recall regions, i.e., the top region under an ROC curve} is employed.
The weighted AUC-ROC measure shifts weight from the bottom regions to the top regions under the ROC curve, which was proposed by Weng et al.~\cite{Weng08}.  The vertical dimension of the ROC plot is divided into $N+1$ horizontal strips, and in the $x$'th section we assign the weight:
\begin{equation}
\label{weight}
W(x)\, = \,\left\{ \begin{array}{ll} 
1-\gamma, & x=0, \\ [6pt]
W(x-1)\times \gamma +(1-\gamma), & 0 < x < N, \\ [6pt]
\frac{W(x-1)\times \gamma +(1-\gamma)}{1- \gamma}, & x=N.
\end{array} \right.
\end{equation}
where $\gamma \in [0, 1]$ controls the amount of weight skewing. Here, $x=0$ corresponds to the bottom-most area of the ROC curve (see Figure~\ref{wauc}), and is assigned weight $1-\gamma$.  Then, weights are transferred recursively from each of the $N$ regions to the one above it, and $\gamma$ controls the amount of weight transferred. For our experiments, we use the weighted AUC-ROC metric with settings of N=4 and $\gamma=0.8$. The corresponding weights of the 5 regions under the ROC curve are calculated based on Equation~\ref{weight} and shown in Figure~\ref{wauc}.  
\begin{figure}[htbp!]
	\centering
	\includegraphics[scale=0.5]{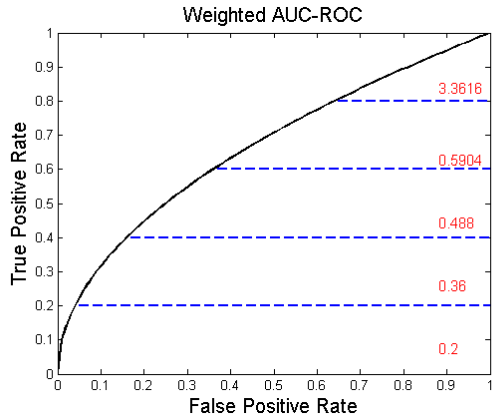}
	\caption{Weighted AUC-ROC}
	\label{wauc}
\end{figure}

As an aside, our proposed $W(x)$ contains a correction to the weighting proposed in Weng et al. In order for the parameter $\gamma$ to have the mathematical properties presented in their work~\cite{Weng08}, the bottom region must have a weight of $1 - \gamma$ instead of $\gamma$ as claimed by~\cite{Weng08}. We prove the correctness of this modification in the Appendix A.

Our second metric is the F-measure:
\begin{equation}
F_{\delta}=(1+ \delta^2)\frac{Precision \cdot Recall}{\delta^2 \cdot Precision+Recall},
\end{equation}
where $\delta$ controls the importance of Precision and Recall~\cite{Rijsbergen1979}. Note that $F_1$ (i.e., $F_\delta$ with $\delta = 1$) is the harmonic mean of the precision and recall. As $\delta \rightarrow \infty$,  $F_{\delta} \rightarrow Recall$, and $F_{\delta}$ is Recall dominated, while as $\delta \rightarrow 0$, $F_{\delta} \rightarrow Precision$, and $F_{\delta}$ is Precision dominated.  (Although F-measures typically are subscripted with the symbol $\beta$, $\delta$ is used here to avoid clashing with the parameters used in the cost function of soft-RFGB.) We use $\delta=5$ in our evaluation to increase the importance of recall over precision, and denote the  measure as $F_5$.

In summary, our experiments use three metrics: 1) false negative rate, 2) $F_5$ measure, and 3) weighted AUC-ROC.  By comparing the metrics, it becomes possible to better understand an algorithm's performance on different parts of the precision-recall curve. For instance, if Algorithm A has lower false negative rate and  higher $F_5$ measure, but similar weighted AUC-ROC to that of Algorithm B, then this suggests that Algorithm A improves on the false negative rate {\it without sacrificing overall predictive performance}.  (Note that false negative rate and F-measure require a classification threshold, this choice will be discussed below.)  

It is important to note that the evaluation metric parameters (i.e. $\gamma$, $\delta$) were \textbf{not} being changed for different data sets, but rather chosen arbitrarily so the metrics focus better on false negatives. Hence, the learning algorithms were not particularly tuned to the evaluation measures.

\subsection{Results}

\begin{figure*}[htbp]
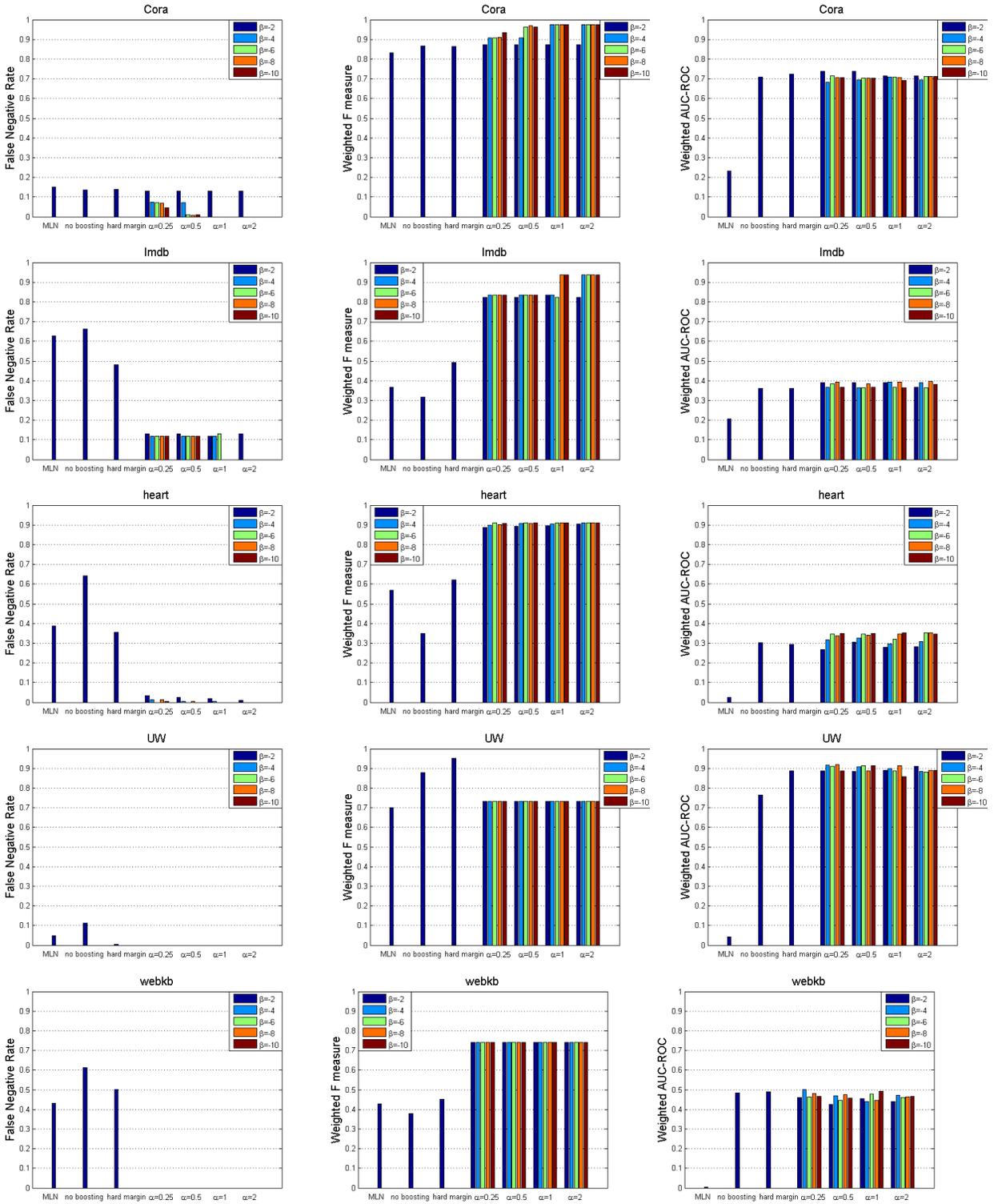

	\begin{minipage}[b]{0.32\textwidth}
		\centering
		\includegraphics[scale=0.36]{cora_fnr.png}
	\end{minipage}
	\begin{minipage}[b]{0.32\textwidth}
		\centering
		\includegraphics[scale=0.36]{cora_wf.png}
	\end{minipage}
	\begin{minipage}[b]{0.32\textwidth}
		\centering
		\includegraphics[scale=0.36]{cora_wauc.png}
	\end{minipage}
	\begin{minipage}[b]{0.32\textwidth}
		\centering
		\includegraphics[scale=0.36]{imdb_fnr.png}
	\end{minipage}
	\begin{minipage}[b]{0.32\textwidth}
		\centering
		\includegraphics[scale=0.36]{imdb_wf.png}
	\end{minipage}
	\begin{minipage}[b]{0.32\textwidth}
		\centering
		\includegraphics[scale=0.36]{imdb_wauc.png}
	\end{minipage}
	\hspace{-15mm}
	\begin{minipage}[b]{0.32\textwidth}
		\centering
		\includegraphics[scale=0.36]{heart_fnr.png}
	\end{minipage}
	\begin{minipage}[b]{0.32\textwidth}
		\centering
		\includegraphics[scale=0.36]{heart_wf.png}
	\end{minipage}
	\begin{minipage}[b]{0.32\textwidth}
		\centering
		\includegraphics[scale=0.36]{heart_wauc.png}
	\end{minipage}
	\hspace{-15mm}
	\begin{minipage}[b]{0.32\textwidth}
		\centering
		\includegraphics[scale=0.36]{uw_fnr.png}
	\end{minipage}
	\begin{minipage}[b]{0.32\textwidth}
		\centering
		\includegraphics[scale=0.36]{uw_wf.png}
	\end{minipage}
	\begin{minipage}[b]{0.32\textwidth}
		\centering
		\includegraphics[scale=0.36]{uw_wauc.png}
	\end{minipage}
	\hspace{-15mm}
	\begin{minipage}[b]{0.32\textwidth}
		\centering
		\includegraphics[scale=0.36]{webkb_fnr.png}
	\end{minipage}
	\begin{minipage}[b]{0.32\textwidth}
		\centering
		\includegraphics[scale=0.36]{webkb_wf.png}
	\end{minipage}
	\begin{minipage}[b]{0.32\textwidth}
		\centering
		\includegraphics[scale=0.36]{webkb_wauc.png}
	\end{minipage}
	\caption{Experimental results. Each row indicates one of five benchmark datasets.  Each column indicates false negative rate, $F_5$ measure, and weighted AUC-ROC, respectively.}
	\label{srfgbresults}
\end{figure*}

Four popular relational algorithms are considered: (1) The state-of-the-art MLN learning algorithm (denoted as MLN)~\cite{icdm11};  (2) A single relational probability tree (denoted as no boosting)~\cite{rpt}; (3) RFGB as presented in earlier research (denoted as  hard margin)~\cite{Natarajan2012}; and (4) our proposed approach soft-RFGB, with various parameter settings for $\alpha$ and $\beta$.  Figure~\ref{srfgbresults} presents the results of these four algorithms.

Experiments were performed with $\alpha = 0.25, 0.5, 1, 2$ and $\beta = -2, -4, -6, -8,$  $-10$.  The choice of $\alpha > 0$ reflects a harsher penalty for misclassified positive examples, while $\beta < 0$ reflects higher tolerance to misclassifying outlier negative examples as positive.   The last four clusters of each graph indicate soft-RFGB with different $\alpha$ settings, and each bar in each cluster indicates a different $\beta$ settings. 

To assign a classification threshold for the false negative rate and $F_5$ measure, the fraction $\#  positive/ \# (positive+negative)$ is used for the Cora, Heart, and IMDB datasets. However, for the WebKB and UW data sets, the data is so skewed that this fraction is extremely small which results in classifying every example as positive. To alleviate this, the negative examples are randomly subsampled during testing so that the pos/neg ratio is $1:10$ and the fraction $1/11$ is used to calculate the false negative rate and $F_5$ measure. 


Four- or five-fold cross validation were used on each dataset and averaged for the three evaluation measures discussed above. These results were used to investigate the following key questions:\\
\textbf{Q1:} How does soft-RFGB perform compared to boosting MLN?\\
\textbf{Q2:} How does soft-RFGB perform compared to a single relational tree without boosting?\\
\textbf{Q3:} How does soft-RFGB perform compared to hard-margin RFGB ?\\
\textbf{Q4:} How sensitive is soft-RFGB to the choice of parameter values?\\

First let us turn our attention to {\bf Q1} and compare soft-RFGB to boosting MLN (column 1 of each graph).  For all domains, there exists a parameter choice of $\alpha$ and $\beta$ that significantly decreases the false negative rate; this is especially so for WebKB and UW data sets, where the false negative rate of soft-RFGB reaches zero. Soft-RFGB significantly improves the false negative rate for nearly all values of $\beta$.  For all domains, soft-RFGB improves on both the weighted AUC-ROC and the $F_5$ measure for all experimented values of $\alpha$ and $\beta$. Thus, {\bf Q1} can be answered affirmatively.

To study {\bf Q2} and {\bf Q3}, consider the single relational tree without boosting (column 2 in each graph) and the standard RFGB with hard margin (column 3).  In each domain, soft-RFGB significantly decreases the false negative rate.  In all the domains except UW, soft-RFGB improves the $F_5$ measure. This is because our parameter settings for soft-RFGB sacrifice precision in order to achieve higher recall; for UW, standard-RFGB {\it already achieves the highest possible recall} (i.e. =1) and soft-RFGB does not have any room to improve recall, though it still sacrifices precision, causing $F_5$ to decrease.  Weighted AUC-ROC from soft-RFGB is similar or slightly better than that from hard-margin RFGB in all the domains. Since our goal is to decrease the false negative rate without hurting the overall performance of the algorithm, with the results on weighted AUC-ROC, these experiments strongly suggest that soft-RFGB can achieve better performance than a single relational tree or standard RFGB in high recall regions. 

Finally, {\bf Q4} is also answered affirmatively: {\it within a reasonably large value range of $\alpha$ and $\beta$, the algorithm is not sensitive to their settings in most domains}. In our tests, $\alpha > 1$ and $\beta < -2$ produces consistently good performance.  This is a major advantage of our cost-sensitive approach over commonly-used tree-regularization approaches, which tend to be highly sensitive to the choice of tree size, or the number of trees. In addition, our parameters $\alpha$ and $\beta$ have a nice intuitive interpretation: they reflect and incorporate the high costs of misclassifying certain types of examples, which is an important practical consideration for real-world domains.

It is worth noting that although Figure~\ref{srfgbresults} seems to suggest that the soft-margin performs uniformly better with larger $\alpha$ and $\beta$ values, this is not always the case.  Performance begins to degrade at some point; for example, with $\alpha=100$, $\beta=-100$, weighted-AUC is only 0.4187 on the WebKB benchmark, compared to 0.4892 for standard RFGB, and 0.5011 for soft-RFGB with the optimal settings of $\alpha$ and $\beta$. Similar results are observed in all the domains.  The point at which performance begins to decline appears to be problem-dependent, and is worthy of future study.
%
%

Figure~\ref{lc} plots the learning curves and the standard deviation of soft-RFGB, standard RFGB, and a single relational tree without boosting, in two sample domains: Heart Disease and WebKB.  We varied the number of training examples and averaged the results over four different runs. 
Again, soft-RFGB significantly outperforms the other two algorithms for all the conditions, especially when the number of training examples is small, which signifies that soft-RFGB can efficiently decrease the false negative rate simply via appropriate parameter settings.  This figure also shows that the standard RFGB and no boosting methods suffer from higher variance as compared to soft-RFGB.

\begin{figure*}[htbp]
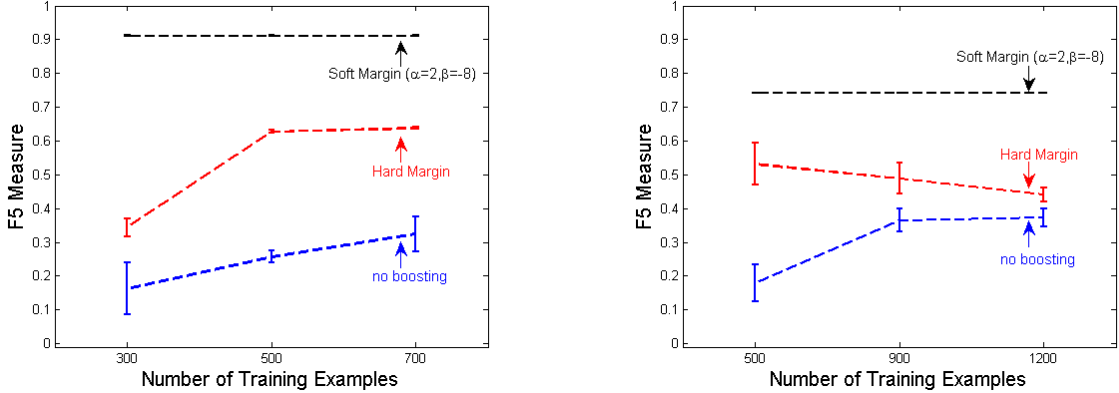

	\begin{minipage}[b]{0.5\textwidth}
		\centering
		\includegraphics[scale=0.5]{LC_heart.png}
	\end{minipage}
	\begin{minipage}[b]{0.5\textwidth}
		\centering
		\includegraphics[scale=0.5]{LC_webkb.png}
	\end{minipage}
	\caption{Sample learning curves, Heart (left) and WebKB (right) data sets.  Error bars indicate standard errors.}
	\label{lc}
\end{figure*}

We also performed experiments to evaluate our initial claim that sub-sampling negative examples to achieve better recall on class-imbalanced datasets runs the risk of increasing variance of the estimator. Here, a re-sampling strategy is used to  sub-sample the negative examples to match the positive examples (1:1 ratio) during learning. For example, on WebKB data,  RFGB with re-sampling achieved weighted-AUC of $0.41$ compared to $0.48$ for RFGB without re-sampling and $0.50$ for soft RFGB; re-sampling $0.43$, RFGB  without re-sampling $0.45$ and our method $0.74$ for $F_5$ measure. Similar results are observed in the other domains.  The key issue with re-sampling is that sub-sampling negatives may lose important examples, which leads to large variance particularly on highly skewed datasets.
To round off the analysis, we also performed over-sampling of positive examples on the heart dataset so that the number of positive examples equals to that of negatives. The weighted AUC was 
0.27 and $F_5$ was 0.82 when using RFGB with the over-sampled positive examples. To put this in perspective, our algorithm achieved 0.35 AUC and 0.91 $F_5$ scores. 

In summary, it can be concluded that our proposed approach faithfully addresses the class imbalance problem in relational data when compared to any under- or over-sampling strategies. 

\section{Discussion and Future Work}

We considered the problem of class-imbalance when learning relational models and adapted the recently successfully relational functional-gradient boosting algorithm for handling this problem. We introduced a soft margin approach that allows for the false positives and false negatives to be considered differently. We then derived the gradients (residues) for each example using this newly defined optimization function and adapted the original algorithm for learning using these residues. We provided the theoretical proof of the convexity of this new optimization function in the functional parameter space.  Finally, we showed empirically that soft-RFGB minimizes the false negative rate without sacrificing overall effectiveness.

This work can be extended in a few interesting directions: 
First, applying this algorithm to learning more complex models (such as RDNs and MLNs) can lead to interesting insights on resulting networks. Second, scaling this algorithm to real large-scale relational data such as EHRs where RFGB has been applied earlier is essential. Finally, applying and learning with this approach to label examples in active learning or allow for some unlabeled examples as in semi-supervised learning are interesting and high-impact future research directions.



\chapter{Statistical Relational Learning for Temporal Data}
\label{chap3}
\ifpdf
\graphicspath{{Chapter3/Chapter3Figs/PNG/}{Chapter3/Chapter3Figs/PDF/}{Chapter3/Chapter3Figs/}}
\else
\graphicspath{{Chapter3/Chapter3Figs/EPS/}{Chapter3/Chapter3Figs/}}
\fi

Modeling structured stochastic processes that evolve over time is an important and challenging task with applications in many fields ranging from surveillance to activity recognition to reliability monitoring of communication networks to treatment planning in biomedicine. Classical AI solutions such as Dynamic Bayesian networks~\cite{murphythesis} discretize the time into fixed intervals and perform the modeling on these intervals. While discretization is often indeed reasonable, there are domains such as medicine in which no natural discretization is available; if we discretize time too finely, the learning problem can quickly become intractable, but if we discretize time too coarsely, we lose information. 
For example, when predicting the occurrence of a disease such as heart attack or stroke, it may not suffice to predict at the level of week or month. However, if we try to predict the occurrence for every second, the modeling problem can quickly become intractable.

Therefore it is not surprising that models over finite spaces but across continuous time have been proposed. The most popular ones fall under the category of Continuous-Time Markov Processes. They are described by an initial distribution over the states of the model and a matrix that specifies the rate of transition between states. Its successor, Continuous-Time Bayesian Networks (CTBNs) make the problem more tractable by factorizing the rate matrix so that conditional independencies can be exploited ~\cite{NodelmanUAI02}. 
Moreover, CTBNs can be induced efficiently from data with an efficient learning method based on random forests which they called mfCTBNs~\cite{WeissNP12}.  The mfCTBN exploits the multiplicative assumption to combine \textit{Conditional Intensity Trees}(CITs) which induced partitions over the joint state space. Each path from the root to the leaf induces a condition that satisfies the splitting criteria specified by the nodes, and the corresponding conditional intensity is defined by the parameter at the leaf of this path. 
Without extensive feature engineering, however, it is difficult---if not impossible---to apply CTBNs to relational domains, in which e.g. there is a varying number of heterogeneous objects and relations among them.
Many of today's datasets, however, are inherently relational and have no natural time slices.
Consider an Electronic Health Record (EHR). It typically contains demographic information, prescription history, lab tests, diagnoses, along with imaging data and possibly in the near future, genetic (SNPs) data as well. Another example is the human-to-X communication network where users typically call many different people, use a multitude of apps, take pictures, listen to music, and so on.
Statistical Relational Learning (SRL) has been proven successful in such domains by combining the power of first-order logic and probability~\cite{srlbook}. 
As far as we are aware, there are no continuous-time SRL approaches. It is possible to model relational CTBNs in 
CTPPL~\cite{pfeffer09}, a general purpose probabilistic programming language for processes over continuous time. However, the use of a relational language allows us to develop a more powerful structure learning approach.

Consequently, we develop {\it Relational Continuous-Time Bayesian Networks} (RCTBNs). They extend SRL towards modeling in continuous time by ``lifting" CTBNs to relational data. The syntax and semantics are based on Bayesian Logic Programs (BLP)~\cite{blp}, and the use of the logical notation allows for an expressive representation that does not fix the number of features in advance yet results in a homogeneous process.
Although already interesting on its own, we go one step further.  
Based on Friedman's functional-gradient boosting~\cite{friedman01}, we develop a non-parametric learning approach for RCTBNs  that simultaneously learns the dependencies between trajectories in the data and  parameters that quantify these dependencies. 
Our extensive experimental evaluation demonstrates that RCTBNs are comparable to state-of-the-art methods on propositional data but can handle relations faithfully, where propositional methods either fail or need to be engineered specifically for each task.

To summarize, we make the following key contributions: 
(1) We present the first continuous-time relational model. (2) We develop a non-parametric learning algorithm for learning these models. (3) We prove the convergence properties of our algorithm. (4) Finally, we show that our algorithm can even model propositional data created by other methods effectively while faithfully modeling relational data. 

We proceed as follows. We first present how to lift CTBNs to the relational case and prove that the resulting RCTBNs are homogeneous. We then show how to learn RCTBNs from data and prove convergence.
Before concluding, we present our experimental results.

\section{Relational Continuous-Time Bayesian Networks}
\label{RCTBN}

In order to better explain the idea of extending the CTBN model to relational domain, the basic principles of BLP will be explained here. Simply put, BLP employs first-order logic rules (called Bayesian logic clauses) to represent the causal influence between variables with the child predicate as clause head and the conjunction of parent predicates as clause body. Each Bayesian logic clause has a conditional probability table as in BN. It adapts the ICI assumption and uses various combining rules (such as Noisy-Or, Noisy-And, Noisy-Max and etc.) to combine multiple Bayesian logic clauses with the same head. 
Besides the combining rules, it also allows aggregation operations when combining the multiple groundings for one Bayesian logic rule. 

Now, consider the following motivating example. It is well known that the probability of developing type 2 diabetes (a hereditary disease) increases with age. Its trajectory can be modeled by CTBNs based on the expected transition time learned from data. 
Due to the genetic disposition, the transition probability of a person's diabetes status depends not only on his/her behavior but also on trajectories of his/her family members' diabetes status. These are in turn dynamic processes. To model relations over time faithfully, we adapt first-order logic syntax. For example, we model the family dependency using two (temporal) predicates:
\begin{align*}
& domain (Diabetes/2, discrete, [true, false]),\\ \nonumber
& domain (FamilyMember/2, discrete, [true, false]). \nonumber
\end{align*}
The domain representation of $Diabetes/2$ can be interpreted as a predicate that has two arguments -- the first argument running over persons and the second argument denoting the continuous time. As a binary predicate, $Diabetes$ takes $discrete$ values of $true/false$. $FamilyMember$ is a binary relation between two persons and is not temporal. Hence, it does not have time as a parameter.

To state that Mary does not have diabetes initially, we use $ Diabetes(mary, t_0). \ominus $ \footnote{ $\ominus$ and $\oplus$ are values that we denote explicitly as against negatives in FOL. This is merely for notational simplicity.}. To denote that she gets diabetes at time $t_i$, we use $Diabetes(mary, t_i). \oplus$.  $FamilyMember(x, y)$ denotes that $x$ is a family member of $y$. A sample data set could just be:\\

\begin{tabular}{|l|l|}
	\hline
	Training examples & Background Knowledge\\ \hline
	$Diabetes(mary, t_0). \ominus $  & $FamilyMember(ann, mary).$  \\ 
	$Diabetes(ann, t_2). \oplus $  & $FamilyMember(eve, mary). $ \\ 
	$Diabetes(tom, t_3). \oplus $  & $FamilyMember(ian, tom). $  \\ 
	$Diabetes(john, t_4). \oplus $ & $FamilyMember(jack, bob). $ \\
	~ & $ FamilyMember(bob, mary).$ \\
	~ & $FamilyMember(tom, mary).$  \\
	\hline	
\end{tabular}\\

It is easy to observe that due to the efficiency of CTBN representations, only the facts where a certain event happens need to be represented explicitly. 

Now, consider the following (probabilistic) rule:
\begin{align*}
& \forall x,t_1, \exists y, t,  \quad Diabetes(x, t_1)\ominus \wedge \\
& FamilyMember(y, x) \wedge Diabetes(y, t)  \oplus \wedge t_1< t <t_2 \\
&  \Rightarrow \quad P(Diabetes(x, t_2) \oplus) = 1- e^{-q_1(t_2-t)} 
\end{align*}
It states that for all the persons ($x$) who do not have diabetes at time $t_1$, if one of his/her family members ($y$) develops diabetes at $t>t_1$, then the probability that $x$ will have diabetes at time $t_2$ follows the exponential distribution specified at the end of the rule (with parameter $q_1$). Another rule is: 
\begin{align*}
& \forall x,t_1, \nexists y, t,  \quad Diabetes(x, t_1)\ominus \wedge \\
& FamilyMember(y, x) \wedge Diabetes(y, t) \oplus \wedge t_1< t <t_2\\
& \Rightarrow \quad P(Diabetes(x, t_2) \oplus) = 1- e^{-q_2(t_2-t_1)}
\end{align*}

This rule contains different conditions from the first rule and says that if no family members develop diabetes in the time interval $t_1$ to $t_2$, the probability that the person will have diabetes is an exponential distribution with parameter $q_2$.

Note that the use of logical notations allows us to compactly define the parameters of the target predicate $Diabetes$ using only $q_1$ and $q_2$. They indicate the conditional intensities of a person getting diabetes when he/she has a family member who has diabetes and correspondingly, when they do not. 
It is clear that, if we had used a propositional CTBN, the dimension of this feature space (the number of CIMs) would be exponential in the number of family members the individual could have. In turn, when the data set is rich in $FamilyMember$ relationships, this can greatly deteriorate the learning efficiency as observed by Weiss et.al (\citeyear{WeissNP12}). Hence, following the standard observations inside SRL~\cite{srlbook}, we exploit the ability to tie parameters using logic. More formally: 
\begin{definition}
	A {\em RCTBN clause} consists of two components: a first-order horn clause that defines the effect of a set of influencing predicates on a target predicate and a conditional intensity matrix (CIM) that models the probability of transition of the target given the influences. 
\end{definition}

\begin{definition}
	A {\em RCTBN program}\footnote{This is akin to Bayesian Logic Programs~\cite{blp}. Whereas, in the case of BLPs, a probabilistic clause consists of a first-order logic clause and a probability distribution associated with it, RCTBNs have a conditional intensity matrix associated with every clause, and the predicates are indexed by time.}
	is a set of RCTBN clauses. 
\end{definition}

Note that there could be multiple clauses with the same head, indicating the target trajectory depends on multiple parent trajectories. Second, since we are in the relational setting, there can be many instances that satisfy a single rule for a single target. For instance, in the above rule, there can be multiple family members (i.e., several instances of $y$) for the current person $x$ and there can be multiple rules for predicting onset of diabetes. 
One is tempted to use combining rules, see e.g.~\cite{crjournal}, for combining the different distributions. However, for RCTBNs, the default amalgamation operation is sufficient to guarantee a homogeneous Markov process. 
We can mathematically prove the following:
\begin{theorem}
	\label{amalg_noisyor}
	Given a set of CIMs associated with multiple RCTBN clauses, the amalgamation operation (addition) on the CIMs will result in a consistent homogeneous Markov process. The amalgamation operation is equivalent to applying the Noisy-Or combining rule to these RCTBN clauses.   
\end{theorem}
 
The intuition is that the amalgamation process is additive in the parameters. Similarly, we can show that the Noisy-Or operation on the exponential distribution adds the conditional intensities as well.  
That is, simple addition of the intensities is our default combination function for the intensities (this is equivalent to a Noisy-Or combining rule for probabilities). 
Nevertheless, developing  equivalents of other combination functions such as weighted-mean and Noisy-And~\cite{crjournal} remains an interesting direction for future research. Moreover, aggregators as used in other SRL models such as PRMs~\cite{prm} can also be adopted easily in our learning formalism by allowing the search to consider aggregators of certain predicates.

Finally, it is interesting to note that unlike the directed model case of SRL models (for example BLPs), there is no necessity for explicitly checking for cycles. 
This is due to the fundamental assumption of the CTBN model: \textit{two variables cannot transition out of their current state at the same instant}. So, the arcs in the transition model represent the dependencies over the trajectories, which always point from the previous time towards the current time. 
More precisely, $X_i \rightleftarrows X_j$ in the transition model means $X_i(t)\rightarrow X_j(t+\bigtriangleup t)$ and $X_j(t)\rightarrow X_i(t+\bigtriangleup t)$. Note that $\bigtriangleup t$ could be different in these two relations as, unlike DBNs which have a predefined sampling rate before learning the model, CTBNs model transitions over continuous time. Thus, a target variable could transit anytime based on a specific CIM during the period after the parent set changes into its current joint state and before any of them transits out of the current state. As there are no cycles in a BN within any single time slice, chain rule of BNs holds in CTBNs. 

\section{Learning RCTBNs} 
We now  develop a non-parametric approach for learning RCTBNs from data.
Note that if the training data is complete, then during the entire trajectory of $\textbf{X}$, the values of $\textbf{Pa}(\textbf{X})$ are known. In turn, it is explicit at any time point which conditional intensity matrix is dominating the dynamics of \textbf{X}. Thus, the likelihood of the parameters can be decomposed into the product of the likelihood of the individual parameters~\cite{NodelmanUAI03}. This parameter modularity along with the structure modularity of CTBNs allow the parameters of CTBNs to be learned locally. Besides, we notice that the likelihood function is also decomposable over individual transitions along the trajectories because of the memoryless property of the exponential distribution. These lay out the theoretical foundation for applying the non-parametric RFGB approach. 

\begin{figure}[htbp!]
	\centering
	\includegraphics[scale=0.2]{trajectories4.png}	
			\begin{tabular}{|l|l|}
		\hline
		Examples & Facts\\ \hline
		$CVD(John, t_0). \ominus $  & $BP(John, t_0, low).$  \\ 
		$CVD(John, t_1). \ominus $  & $Diabetes(John, t_0, false). $ \\ 
		$CVD(John, t_2). \ominus $  & $BP(John, t_1, high). $  \\ 
		$CVD(John, t_3). \ominus $ & $BP(John, t_2, low). $ \\
		$CVD(John, t_4). \ominus $ & $ Diabetes(John, t_3, true).$ \\
		$CVD(John, t_5). \oplus $ & $BP(John, t_4, high).$  \\
		$CVD(John, t_6). \ominus $ & $BP(John, t_6, low).$ \\
		\hline
	\end{tabular}
	\caption{Example trajectories of John.\label{trajectories}}
\end{figure}

Note that each training example is a segment of a trajectory that includes {\em not more than one} transition of the target predicate. For example, Figure~\ref{trajectories} shows the trajectories of $Blood Pressure(BP)$, $Diabetes$ and $CVD$ for patient John. Let us assume that the goal is to predict the transition of $CVD$ from low to high. As can be observed, every feature/predicate transits at different time points (fundamental assumption of CTBNs). This can be represented in a logical notation as we show in the bottom half of Figure~\ref{trajectories}.

Now, we first define the transition distribution from the perspective of each segment:
\begin{equation}
\hspace{-1mm}\begin{array}{lr}
\text{\bf positive examples } (x_i^k \rightarrow x_i^{k'}): \\
\quad P(x_i^k \rightarrow x_i^{k'}| pa_i^j) = \\
\left( 1-\exp\left ( -q_{x_i^k|pa_i^j}T[x_i^kx_i^{k'}|pa_i^j]\right) \right )\times \frac{q_{x_i^kx_i^{k'}|pa_i^j}}{q_{x_i^k|pa_i^j}}\\
\text{\bf negative examples }(x_i^k \rightarrow x_i^{k''}): \\
\quad P(x_i^k \rightarrow x_i^{k''}| pa_i^j) =\\
\left( 1-\exp \left(-q_{x_i^k|pa_i^j}T[x_i^kx_i^{k''}|pa_i^j]\right) \right )\times \frac{q_{x_i^kx_i^{k''}|pa_i^j}}{q_{x_i^k|pa_i^j}}\\
\text{\bf negative examples } (x_i^k \rightarrow x_i^k ):  \\
\quad P(x_i^k \rightarrow x_i^k| pa_i^j) =  \exp\left(-q_{x_i^k|pa_i^j}T[x_i^kx_i^k|pa_i^j]\right)
\end{array}
\label{probs}
\end{equation}
where $T[x_i^kx_i^{k'}|pa_i^j]$ is the residence time of $X_i$ starting from being in its $k$-{th} state till transiting to its $k'$-{th} state given its parents being in their $j$-{th} joint state. 
We define the function value $\varphi_{x_i^kx_i^{k'}|pa_i^j} \in (-\infty, \infty)$ such that $q_{x_i^kx_i^{k'}|pa_i^j} = e^{\varphi_{x_i^kx_i^{k'}|pa_i^j}} \in \left[ 0, \infty \right)$. 
Also observe that since $q_{x_i^kx_i^{k'}|pa_i^j}$ monotonically increases with $\varphi_{x_i^kx_i^{k'}|pa_i^j}$, optimization w.r.t. $q_{x_i^kx_i^{k'}|pa_i^j}$ for CTBNs could be addressed by maximizing the loglikelihood function w.r.t. $\varphi_{x_i^kx_i^{k'}|pa_i^j}$. 

To summarize, instead of maximizing the likelihood function by calculating the \textit{sufficient statistics}---$\hat{T}[{x_i^k|pa_i^j}]$, the total amount of time that $X_i=x_i^k$ while $Pa(X_i)=pa_i^j$, and $M[x_i^k \rightarrow x_i^{k'}|pa_i^j]$, the total number of transitions---aggregated over the trajectories~\cite{NodelmanUAI03}, we  optimize the global loglikelihood function by maximizing the individual likelihood of all the segments each of which has a weight (i.e. $\varphi_{x_i^kx_i^{k'}|pa_i^j}$) attached to it. 

In the case of binary valued target predicates, the target variable has to transit to the other state upon transiting out of the current state, so $q_{x_i^kx_i^{k'}|pa_i^j} (k' \neq k) = q_{x_i^k|pa_i^j}$, hence $e^{\varphi_{x_i^kx_i^{k'}|pa_i^j}} (k' \neq k) = e^{\varphi_{x_i^k|pa_i^j}}$, and the negative examples now only contain one case where no transition happened. So, the gradient function could be derived as: 
\begin{align} 
& \frac{\partial LL}{\partial \varphi_{x_i^kx_i^{k'}|pa_i^j}} = \frac{\partial LL}{\partial q_{x_i^kx_i^{k'}|pa_i^j}} \cdot e^{\varphi_{x_i^kx_i^{k'}|pa_i^j}} \nonumber \\
& \hspace{-4mm}\begin{cases}
\text{\bf positive examples} \  (x_i^k \rightarrow x_i^{k'}) : \\
\frac{-\exp\left(-e^{\varphi_{x_i^k|pa_i^j}}T[x_i^kx_i^{k'}|pa_i^j]\right)\left(-e^{\varphi_{x_i^k|pa_i^j} }T[x_i^kx_i^{k'}|pa_i^j]\right)}{1-\exp \left(-e^{\varphi_{x_i^k|pa_i^j}}T[x_i^kx_i^{k'}|pa_i^j]\right)}  \\
\text{\bf negative examples} \  (x_i^k \rightarrow x_i^k): \\
-T[x_i^kx_i^k|pa_i^j]\cdot \exp(\varphi_{x_i^kx_i^{k'}|pa_i^j})\\\quad\quad\quad\quad\quad=-e^{\varphi_{x_i^k|pa_i^j}}T[x_i^kx_i^k|pa_i^j]\\
\end{cases}\hspace{-5mm}
\label{grad2}
\end{align}

Based on its definition in \eqref{probs}, the transition probability for positive examples: $prob^+ = 1-\exp(-e^{\varphi_{x_i^k|pa_i^j}}T[x_i^kx_i^{k'}|pa_i^j]) $, so the gradients function can be represented in the terms of $prob^+$ as:
\begin{equation} 
\frac{\partial LL}{\partial \varphi_{x_i^kx_i^{k'}|pa_i^j}} =
\begin{cases}
\text{\bf positive examples} \ (x_i^k \rightarrow x_i^{k'}) : \\
\frac{-(1-prob^+)\ln(1-prob^+)}{prob^+}\\
\text{\bf negative examples} \ (x_i^k \rightarrow x_i^k ) : \\
\ln(1-prob^+)\\ 
\end{cases}
\label{gradf}
\end{equation}
This presents the weight updates of the examples in each iteration based on the current model. 

Plugging \eqref{gradf} into the RFGB results in a convergent learner. More precisely, 
\begin{theorem}
	There exist global maxima for the loglikelihood function.
\end{theorem}

\begin{proof}
	For positive examples, the Hessian function w.r.t. $\varphi_{x_i^kx_i^{k'}|pa_i^j}$ equals to:
	\begin{align*}
	H^+ & = \frac{\partial \ \ \nabla^+}{\partial \ \ \varphi_{x_i^kx_i^{k'}|pa_i^j} } = \frac{\partial}{\partial \ \ prob^+} \frac{-(1-prob^+)\ln(1-prob^+)}{prob^+} * \frac{\partial \ \ prob^+}{\partial \ \ \varphi_{x_i^kx_i^{k'}|pa_i^j} } \\
	& = \frac{[\ln(1-prob^+) + 1] * prob^+ + (1-prob^+)\ln (1-prob^+)}{(prob^+)^2} * \\
	& \quad   [e^{-q_{x_i^k|pa_i^j}T[x_i^kx_i^{k'}|pa_i^j]} * q_{x_i^k|pa_i^j}T[x_i^kx_i^{k'}|pa_i^j]] \\
	& = \frac{prob^+ + \ln(1-prob^+)}{(prob^+)^2} * [-\ln(1-prob^+) (1-prob^+)]
	\end{align*}
	Since $prob^+ \in [0,1]$, we have $\ln(1-prob^+) \leq 0$, $(1-prob^+) \geq 0$ and $[prob^+ + \ln(1-prob^+)] \leq 0 $. Hence, $H^+$ is non-positive for the interval of $prob^+$ as [0, 1]. 
	
	For negative examples, the Hessian function w.r.t. $\varphi_{x_i^kx_i^{k'}|pa_i^j}$ equals to:
	\begin{align*}
	H^- & = \frac{\partial \ \ \nabla^-}{\partial \ \ \varphi_{x_i^kx_i^{k'}|pa_i^j} } = \frac{-1}{1-prob^+}* [-\ln(1-prob^+) (1-prob^+)] \\
	& = \ln(1-prob^+)
	\end{align*}
	$H^-$ is non-positive everywhere for the interval $prob^+ \in [0,1]$.\\
	
	Hence, the log-likelihood function of $\varphi_{x_i^kx_i^{k'}|pa_i^j}$ is a concave function and there exists the global optimum. 
\end{proof}

\begin{theorem}
	RFGB for CTBNs converges when the predictions reach the true values. In other words, RFGB for CTBNs will converge to global maxima. 
\end{theorem}

\begin{proof}
	For positive examples, let $a=1-prob^+$, based on Equation~\ref{gradf} and $L'H\hat{o}pital's$ rule
	\begin{equation*}
	\underset{prob^+ \to 1}{\lim} \nabla^+ = \underset{a \to 0^+}{\lim} \frac{-a\ln a}{1-a} = \underset{a \to 0^+}{\lim} \frac{\ln a}{1-\frac{1}{a}} = \underset{a \to 0^+}{\lim} \frac{1/a}{1/a^2} = 0^+
	\end{equation*}
	For negative examples, 
	\begin{equation}
	\underset{prob^+ \to 0}{\lim} \nabla^- = \underset{prob^+ \to 0}{\lim} \ln(1-prob^+) = 0^-
	\end{equation}	
	As shown above, $\nabla^+ \in [0, +\infty)$ and converges when the prediction equals the true value of positives (i.e. $prob^+=1$) while $\nabla^- \in (-\infty, 0]$ and converges when the prediction equals the true value of negatives (i.e. $prob^+=0$).
\end{proof} 	

The algorithm for learning RCTBNs is summarized in Algorithm~\ref{algo_rctbn}, where we use $OP$ to denote observed predicates, $TP$ to denote the set of target predicate transitions that we are interested in and $CT$, the current target transition (for example, low to high of CVD).
In each gradient step, we first generate examples (denoted as $Tr$) based on the current $\varphi$ function. We use standard off-the-shelf relational regression tree learning approaches~\cite{BlockeelR98} to get a regression tree that fits these gradients $Tr$ (function $FitRelRegressionTree$). Then, we add the learned tree $\Delta^{CT}_{m}$ to the model and repeat this process. Algorithm \ref{eg-algo_rctbn} describes the example generation process for RCTBN-RFGB. Since any non-transition can be used to generate a negative example, we can potentially have infinite negative examples (for every time point). To prevent skew and scalability issues, we generate negative examples only at time points when a certain predicate other than the target predicate has a transition. Algorithmically, we iterate over all the transitions in the trajectories. If the transition is over the target predicate, we generate a regression example using the gradients for positive examples from Equation \ref{gradf}. If the target predicate does not transit in this segment, we treat it as a negative example, and compute the gradient accordingly. 

\begin{algorithm}[t!]	
	\caption{RCTBN-RFGB: RFGB for RCTBNs}
	\label{algo_rctbn}
	\begin{algorithmic}[1]		
		\Function{RCTBNBoost}{$TP$, $OP$}
		\For {$CT$ in $TP$} \Comment{Iterate through target predicates}
		\State $\varphi^{CT}_0 := $ Initial function \Comment{Empty tree}
		\For {$1 \leq m \leq M$} \Comment{M gradient steps}
		\State $Tr := GenExamples(CT, \varphi^{CT}_{m-1},OP)$
		\State  \Comment{Generate examples}
		\State $\Delta^{CT}_{m} := FitRelRegressionTree(Tr,OP)$
		\State  \Comment{Fit trees to gradients}
		\State $\varphi^{CT}_{m} := \varphi^{CT}_{m-1} + \Delta^{CT}_{m}$  \Comment{Update model}
		\EndFor
		\EndFor
		\EndFunction
	\end{algorithmic}
	
\end{algorithm}
\begin{algorithm}[t!]
	\caption{Example generation for RCTBNs}
	\label{eg-algo_rctbn}
	\begin{algorithmic}[1]
		\Function{GenExamples}{$CT, \varphi, OP$}
		\State $Tr := \emptyset$
		\For { $tr_i \in OP$} \Comment{Iterate over all transitions in $OP$}
		\State Compute the residence time $T$ of the target predicate 
		\State Compute $prob^+ = f(T, \varphi )$ \Comment{Transition probability}
		\If { $tr_i \in CT $ }
		\State $\Delta(tr_i) = \frac{-(1-prob^+)\ln(1-prob^+)}{prob^+}$ 
		\State \Comment{Compute gradient of positive example}
		\Else
		\State $ \Delta(tr_i) = \ln(1-prob^+)$ 
		\State \Comment{Compute gradient of negative example}
		\EndIf
		\State $ Tr := Tr \, \cup <tr_i, \, \Delta(tr_i)>$ 
		\State \Comment{Update relational regression examples}
		\EndFor
		\\
		\Return $Tr$ \Comment{Return regression examples}
		\EndFunction
	\end{algorithmic}
\end{algorithm}

\section{Experiments for Relational CTBN}

Our intention is to investigate whether RCTBNs truly generalize CTBNs to relational domains by addressing the following two questions:
\begin{description}	
	\item{\bf(Q1)} How does RCTBN compare to CTBNs on propositional domains?
	\item{\bf(Q2)} Can RCTBN-RFGB learn relational models faithfully where CTBNs would fail or need extensive feature engineering?
\end{description}
To this aim we evaluated RCTBN-RFGB on propositional and relational data sets and compared it to state-of-the-art where possible.
More precisely, we employed three standard CTBN datasets from prior literature, i.e. the Drug model~\cite{NodelmanUAI03}, MultiHealth model~\cite{WeissNP12} and S100 model~\cite{WeissNP12}. In the drug domain, the data is generated for length of 10, the target predicate is $JointPain(X, T)$ and other variables include the trajectories of $Eating$, $Fullstomach$, $Hungry$, $Uptake$, $Concentration$, $Barometer$, and $Drowsy$. 
The MultiHealth data has the target predicate $Atherosclerosis(X, T)$ and trajectories of $Gender$, $Smoking$, $Age$, $Glucose$, $BMI$, $BP$, $ChestPain$, $AbnormalECG$, $MI$, $Troponin$, $ThrombolyticTherapy$, $Arrhythmia$, $Stroke$, and $HDL$, for a length of 100. The S100 domain has 100 binary valued trajectories of length 2 and the target is $S100(X, T)$.

We compared RCTBN-RFGB learning approach with the state-of-the-art propositional CTBN learning approach-- mfCTBN~\cite{WeissNP12} on these data. Note that these are {\em specialized domains created by the authors of the respective papers} (Nodelman et al. for Drug and Weiss et al. for MultiHealth and S100). Our goal is to demonstrate that even for data that are not created synthetically by our model, we can still learn a comparably accurate model from them. Moreover, our formalism can handle structured inputs/outputs while these prior work cannot handle them without significant feature engineering.
We ran 5-fold cross validation to generate the learning curve of the loglikelihood and AUC-ROC on the test set.  

\begin{figure}[htbp!]
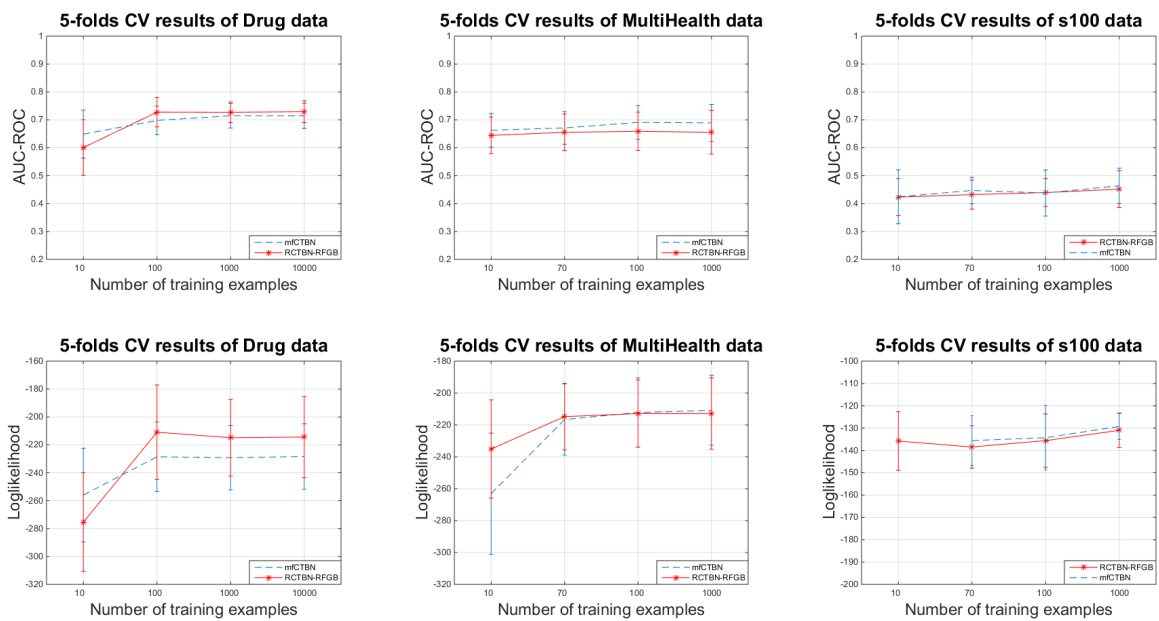

	
	\begin{minipage}[b]{0.32\textwidth}
		\centering
		\includegraphics[scale=0.34]{Drug_aucroc.png}
	\end{minipage}
	\begin{minipage}[b]{0.32\textwidth}
		\centering
		\includegraphics[scale=0.34]{mhm_aucroc.png}
	\end{minipage}
	\begin{minipage}[b]{0.32\textwidth}
		\centering
		\includegraphics[scale=0.34]{S100_aucroc.png}
	\end{minipage}\\
	
	\begin{minipage}[b]{0.32\textwidth}
		\centering
		\includegraphics[scale=0.34]{Drug_ll.png}
	\end{minipage}
	\begin{minipage}[b]{0.32\textwidth}
		\centering
		\includegraphics[scale=0.34]{mhm_ll.png}
	\end{minipage}
	\begin{minipage}[b]{0.32\textwidth}
		\centering
		\includegraphics[scale=0.34]{S100_ll.png}
	\end{minipage}
	\caption{ {\bf propositional domains (Q1):} The results show no statistically significant difference between RCTBNs and state-of-the-art methods.\label{results2}}
\end{figure}

\begin{figure}
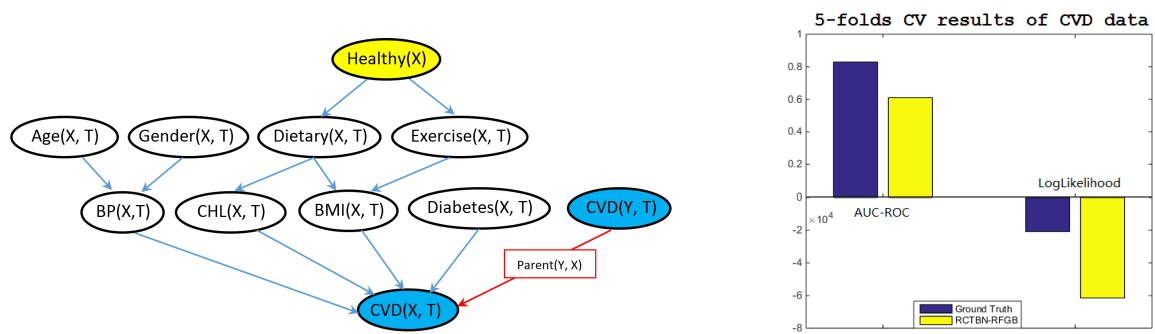

	\begin{minipage}[b]{0.6\textwidth}
		\centering
		\includegraphics[scale=0.2]{Model.png}
	\end{minipage} 
	\begin{minipage}[b]{0.4\textwidth}
		\centering
		\includegraphics[scale=0.4]{CVD.png}	
	\end{minipage}
	\caption{ {\bf relational experiment (Q2):} {\bf (left frame)} RCTBN model of CVD, \label{model} {\bf (right frame)} Results. The results show that RCTBNs can capture relations faithfully. As there is no propositional model, the results are compared against ``ground truth". Note that the ``ground truth"  does not have AUC = 1.0 because when we generated the data using forward sampling, the values are sampled based on the transition intensity calculated by $P=1-e^{-q*t}$. 
		The AUC of ground truth is calculated based on these generating transition probabilities instead of the determinant values (i.e. 0 for positive examples and 1 for negative examples).	\label{results1}}
\end{figure}
{\bf (Q1) RCTBN compares well to CTBN on propositional domains:} 
As Figure~\ref{results2} show, our method is comparable (no statistically significant difference at p=0.05) to the best propositional method i.e. mfCTBN on data created by mfCTBN model. This answers \textbf{Q1} affirmatively -- RCTBN is comparable to state-of-the-art algorithms for modeling propositional dynamic processes.

{\bf (Q2) RCTBN can capture relations faithfully where CTBNs would fail or need extensive engineering: }
The  relational data is generated by a cardiovascular disease(CVD) model shown in Figure~\ref{results1} (left) over a duration of 50 years. We used a latent variable, $Healthy(X)$ to generate the transition patterns of dietary and exercise habit for individuals with healthy/unhealthy lifestyle. $CVD(X, T)$ is our target predicate whose CIMs are conditioned on both the features related to the target individual and the CVD status of his/her parents. We used forward sampling~\cite{Nodelmanthesis} to generate 200 trajectories for individuals whose parents are absent in the data (i.e. sample their CVD transitions based on CIMs conditioned only on white parent nodes in Figure~\ref{results1}(left)) and 200 trajectories for individuals who have one or both parents in the data (CIMs conditioned on white and blue parent nodes in Figure~\ref{model}(left)). 
Our goal is to show that our RCTBN-RFGB learning approach could not only achieve comparable performance on propositional data but can also model the relationships faithfully.

For this dataset, we present the results averaged over 5 runs compared against the ground truth in Figure~\ref{results1} (right). As can be observed, our approach is comparable to the ground truth. A more compelling result is that we are able to retrieve the structure of the CVD predicate, i.e., the learned tree includes the CVD status of the parents. 
This answers  \textbf{Q2} affirmatively, i.e., our proposed approach does model relations faithfully\footnote{Note that our work is the first dynamic relational model to handle continuous time. Other formalisms exist to handle continuous variables but not continuous time which in our case is not a variable by itself but rather an argument of all the predicates. Hence, our relational baseline is the ground truth.}.

\section{Discussion and Future Work}
We proposed the first relational model that can faithfully model continuous time. Our model is inspired from the successes in the CTBNs and the SRL communities. Besides, we adapt and refine the leading learning algorithm for SRL models to the temporal case. The key innovation is that the conditional intensity is defined as a monotonic function of the RFBG function value, hence, maximizing w.r.t. function values can maximize the loglikelihood of the trajectories. Our resulting algorithm is non-parametric and can learn the dependencies and the parameters together. Our initial results are promising and show that we can indeed recover the original structure and faithfully model the relational trajectories.

Our current learning approach is discriminative and we are only learning for a single query. Extending this to generative learning and modeling jointly over multiple predicates is our next step. More applications of the algorithm to complex tasks such as Electronic Health Records, gene experimental databases, network flow modeling, and 
monitoring the reliability of human-to-X communication systems, among others, with asynchronous events are exciting directions. Adaptation of different types of combination functions is another important direction. Finally, improving inference is essential to scale up our learning algorithm to a large number of tasks and is also an interesting direction for future research.

\chapter{Statistical Relational Learning for Hybrid Data}
\label{chap4}
\ifpdf
    \graphicspath{{Chapter4/Chapter4Figs/PNG/}{Chapter4/Chapter4Figs/PDF/}{Chapter4/Chapter4Figs/}}
\else
    \graphicspath{{Chapter4/Chapter4Figs/EPS/}{Chapter4/Chapter4Figs/}}
\fi


\label{hplm}
EHR data is composed with heterogeneous records each of which describes a specific aspect of health condition and is collected according to its own standard and stored in its unique form. The value types of those variables vary among the data from different facets of health care.  For example, laboratory or clinical tests usually have continuous values, such as the blood pressure, cholesterol level, blood glucose, etc.;  the occurrence of certain procedure, treatment or symptom is recorded as boolean values;  there are also multinomial variables, such as gender, exercise level, clinical narratives described with International Classification of Disease-9 (ICD‑9) or ICD‑10; count data is generated when the binary variables are aggregated over the historical EHR, which could be the number of times a specific symptom the patients have experienced, the amount of a certain procedure has been performed, the count of a particular therapy or medication the patients have taken, the amount of cigarette or alcohol intake, the cumulative hours of exercise or REM sleep, etc. It is a challenge regarding how to standardize such multi-formats data into an integrated and consistent form so it could be described, mined and reasoned over by a hybrid machine learning model. In this chapter, a hybrid probabilistic logic model is introduced and an efficient statistical relational learning approach for it is proposed. The following sections will be organized as: first, the related work focusing on SRL inference and learning in hybrid domains are introduced; then, the proposed learning approach based on Relational Functional Gradient Boosting (RFGB) is presented from the perspective of exponential family; finally, the gradient update functions for this hybrid RFGB approach are derived for different value types of variables.  

\section{Related Work}

There are some work focusing on the inference and parameter learning problems in hybrid probabilistic logic models, such as the one for hybrid Markov logic networks~\cite{WangHML}, for hybrid probabilistic relational models~\cite{NarmanBKJ10}, for Hybrid Logical Bayesian Networks~\cite{Ravkic2012} and Continuous Bayesian Logic Program (CBLP)~\cite{AdaptiveBLPs},  the inference algorithms for piecewise-constant hybrid Markov networks (CHML) and piecewise-polynomial hybrid Markov networks (PHMN)~\cite{BellePB15}, for Relational Hybrid Models~\cite{ChoiA12}, as well as the inference procedure and parameter learning in the context of Probabilistic Logic Programming languages such as PRISM~\cite{Islam12IPL,Islam12Corr} and ProbLog~\cite{GutmannJR10}.
 
Wang et al.~\citeyear{WangHML} extended Markov Logic networks into the hybrid domain by allowing features to be continuous properties and functions over them. Markov networks represent the joint distribution of a set of variables as a log-linear function of arbitrary feature functions. Its first-order extension is Markov Logic networks (MLNs) which is a collection of first-order clauses and associated weights reflecting the strength of the constraints. A Markov network can be constructed by grounding the MLN template with one node per ground atom and one feature per ground clause and the probability of a world is an exponential function of the weighted sum of all true groundings of MLN clauses: $P(x) = \frac{1}{Z} \exp (\sum_i w_in_i(x))$. The Hybrid MLN (HMLN)~\cite{WangHML} generalizes feature counts to feature sums and defines a family of log-linear models as $P(x) = \frac{1}{Z} \exp (\sum_i w_is_i(x))$, where $s_i$ is the sum of the values of all groundings of the $i_th$ clause in x. HMLNs represent the first-order formulas that include numeric variables as a set of soft constraints (equality or inequality), with an implied Gaussian penalty for diverging from them.
Weights can be learned using the same algorithms as for MLNs~\cite{Richardson06,Lowd2007EWL}.

Kersting et al.~\citeyear{AdaptiveBLPs} extended the Bayesian Logic Programs to handle continuous random variables.
A Continuous Bayesian Logic Program (CBLP) consists of logical as well as quantitative components. The logical component is a (finite) set of Bayesian clauses which define the conditional independences among the Bayesian atoms whose grounded correspondence are random variables (discrete or continuous valued). The quantitative component is composed of the conditional probability density functions associated to Bayesian clauses and combing rules associated to Bayesian predicates, e.g. Noisy-Or for discrete random variables, linear regression model for Gaussian variables. In this work, they proposed to use a gradient ascent approach to learn the optimal parameters for mean values of continuous variables. However, it failed to present the explicit formulae of the gradient functions after substituting the conditional Gaussian distribution functions into the likelihood function, neither the process for learning the weight parameters when there are multiple continuous parent variables.   

Ravkic et al. presented their preliminary work on extending Logical Bayesian Networks (LBNs) to hybrid domains~\cite{Ravkic2012}. LBN is an extension of Bayesian networks to relational domains by combining the first-order logic with the three components of BNs (i.e. random variables, conditional dependences and conditional probability distributions) and adding an additional component -- \textit{logic program} which describes the background knowledge for specific problems or worlds. Ravkic's work allows continuous variables to appear as the head of the conditional dependency clauses. However, it has a restriction on the continuous parents, which is common in standard hybrid BNs.  


Choi et al.~\citeyear{ChoiHA10} proposed the Relational Continuous Model (RCM) which allow for par-factors with continuous valued variables, but restricting φ to Gaussian potentials. Choi et al. later extended the framework to allow inference in hybrid relational domains~\cite{ChoiA12}, which they called Relational Hybrid Models (RHMs).

Statistical Relational Learning approaches which use the (finite) First-Order Logic (FOL) as description language usually employ MAX-SAT solvers which convert the maximum a posterior (MAP) inference problem into maximizing the total weight of the satisfied formulae. From a different angle, Teso et al.~\citeyear{TesoSP13} proposed Learning Modulo Theories (LMT) which rely on Satisfiability Modulo Theories (SMT) -- an alternative class of formal languages. LMT converts the original inference problem to minimizing the total cost of the unsatisfied rules by employing different cost functions for formulae with variables of different value types (e.g. linear cost for formulae with only numerical variables). Then the formed optimization problem can be solved by Optimization Modulo Theories (OMT) solvers which can perform parameter learning and inference in mixed Boolean-numerical domains. 

Although there are numerous research on extending the statistical relational learning approaches to mixed discrete-continuous domains, there is little work having been done for the structure learning of hybrid probabilistic logic models.
However, for the real world problems, the necessary expertise and knowledge needed to construct the logic structure of the model is usually lacking and we need to gain such knowledge through learning from the data. 

The only work so far which tries to address the structure learning problems in hybrid relational domains is the work of Ravkic et al.~\citeyear{Ravkic2015}. Their work on hybrid RDNs tackled structure learning by employing  a greedy search approach with a decomposable score function. However, it only employed the multinomial distribution assumption for discrete variables. We, on the other hand, also considered the Poisson probability distribution assumption for discrete dependent variables to fulfill various needs when modeling heterogeneous medical data. Our learning approach built upon the state-of-the-art relational learning approach RFGB which is more efficient compared with learning one complicated global model for each dependent variable.


\section{Proposed Approach}

In this section, a statistical relational learning approach is presented, which can perform the structure learning and parameter learning simultaneously in hybrid relational domains for variables of any value type that belongs to the exponential family. 
  
\subsection{Graphical Models as Exponential Family}

Exponential Family~\cite{Erling1970} provides a general way to represent the probability distribution as a density $p$  continuous w.r.t some base measure $h(x)$~\cite{Wainwright08}. This base measure $h(x)$ could be the counting measure or the ordinary Lebesgue measure on $\mathbb{R}$. Many common distributions belong to the exponential family, such as the normal, exponential, gamma, Dirichlet, Bernoulli, Poisson, multinomial (with fixed number of trials) and etc. Generically, there could be numerous different distributions that are consistent with the observed data, so this problem can be converted to find the distribution that faithfully fits the data while has the maximal Shannon entropy~\cite{Wainwright08}. Hence, the exponential family is defined as a family of probability distributions whose probability mass function (discrete variables) or probability density function (continuous variables) can be represented with a parameterized collection of density functions:
\begin{equation*}
p_X(x|\theta) = h(x)\exp\{ \langle \theta, \phi(x) \rangle - A(\theta) \},
\end{equation*}
where $\phi(x)$ is a collection of \textit{sufficient statistics}, $\theta$ is an associated vector of \textit{exponential parameters} and $A(\theta)$ is the \textit{log partition function} which ensures $p_X(x|\theta)$ being properly normalized.

The target variable (discrete or continuous) is denoted as $y$ while the other variables are denoted as set $\textbf{X}$ whose instances are denoted as $\textbf{x}$. As in graphical models, each node corresponds to one random variable, nodes and variables will be used interchangeably in the following sections. Parent nodes denote the variables having dependence correlations with the target variable which are captured by the graphical models. The subscript $i$ indexing the examples indicates the $i^{th}$ example. The subscript $_{y;\textbf{X}}$ indicates the corresponding sufficient statistics (i.e. $\phi(y;\textbf{X})$) or exponential parameters (i.e. $\theta_{y;\textbf{X}}$) in a certain graphical model which could be a linear function such as in Logistic Regression and Naive Bayes or a non-linear models such as Decision Trees, Neural Networks and Poisson Dependency Networks.  So, the exponential family in terms of one example can be re-written as:
\begin{equation}
\label{ef}
 p_{y_i;\textbf{x}_i}(y_i;\textbf{x}_i|\theta) =  h(y_i;\textbf{x}_i)\exp\{ \langle \theta_{y_i;\textbf{x}_i},  \phi(y_i;\textbf{x}_i) \rangle - A(\theta_{y_i;\textbf{x}_i}) \}.
\end{equation} 

As a general form of most probabilistic graphical models,  exponential family provides a better perspective from which the fundamental connection between their theories as well as learning algorithms can be revealed. Hence, the theories are illuminate from the perspective of exponential families and the gradient update formulae of RFGB are derived for three probability distributions of exponential family -- multinomial distribution, Poisson distribution and Gaussian distribution so a hybrid probabilistic logic model can be learned in mixed discrete-continuous domains.

\subsection{RFGB for Multinomial Distributed Variables}

\textit{Multinomial distribution} is the most common assumption machine learning models employed for discrete valued variables. It assumes that each sample is independently extracted from an identical categorical distribution, and the numbers of samples falling into each category follow a {\em multinomial distribution}. Figure~\ref{distributions} shows an example data set of $N$ samples where variable $X$ can take one of five possible values $\{x^1, x^2, x^3, x^4, x^5\}$. Multinomial distribution assumption (Figure~\ref{distributions} left) states that the occurrences of the samples belonging to categories  $\{x^1, x^2, x^3, x^4, x^5\}$ follow a multinomial distribution with parameters $\{p_1, p_2, p_3, p_4, p_5\}$ where $\sum_{i=1}^5 p_i =1$, which means that when $N$ is large enough, $N_i$-- the occurrence of $X=x^i$ can be modeled as a normal distribution with mean value $ N \cdot p_i$. 

\begin{figure}[htbp!]
	\centering
	\begin{minipage}[b]{0.4\textwidth}
		\includegraphics[scale = 0.25]{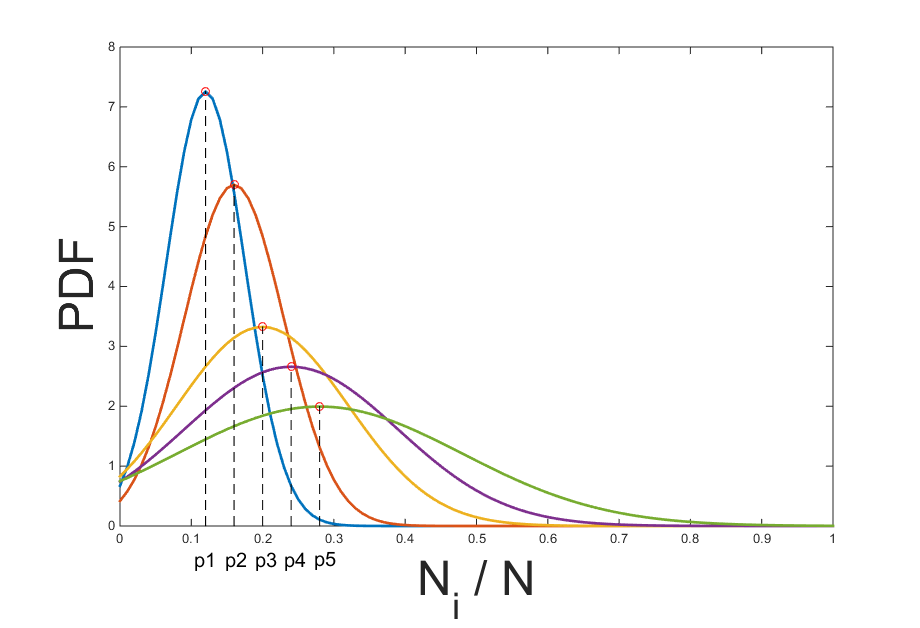}
		\label{multi}
	\end{minipage}
	\begin{minipage}[b]{0.4\textwidth}
		\includegraphics[scale = 0.4]{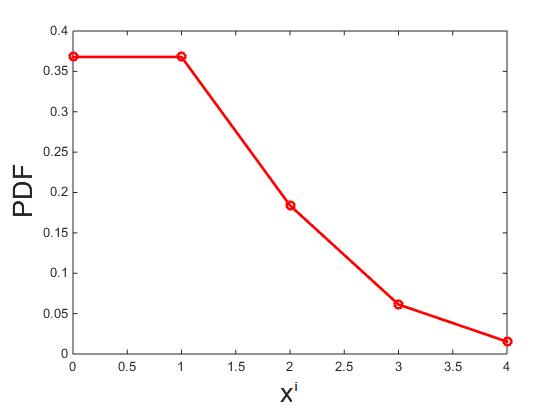}
		\label{poisson}
	\end{minipage}
	\caption{Two Examples of Exponential Family Distributions - (left) Multinomial, (right) Poisson.}
	\label{distributions}
\end{figure}

Now, let us consider the objective function for optimizing the log-likelihood under the multinomial distribution assumption and derive the corresponding gradient. Assuming that given the observations in the learned discriminative model, the target variable follows a categorical distribution with parameters $\{p_{y^1}, ..., p_{y^r}\}$ where $\sum_k p_{y^k} = 1$ and $\{ y^1, ..., y^r\}$ indicate possible values $y$ can take, the probability distribution over the number of all possible values is a multinomial distribution:
\begin{equation*}
p(N_{y^1}, ..., N_{y^r}; p_{y^1}, ..., p_{y^r}) = \frac{\Gamma (\sum_k N_{y^k} +1)}{\prod_k \Gamma (N_{y^k}+1)} \prod_{k=1}^r p_{y^k}^{N_{y^k}},
\end{equation*}
where $y \in \{0, \ \ 1, \ \ ..., \ \ r-1\}$ and $r-1$ is the number of possible values $y$ can take.
According to the definition in Equation~\ref{ef}, the parameters in the general form of exponential family can be calculated through the natural parameters in this multinomial distribution by:
\begin{align*}
& \phi(y;X) = [N_{y^1}, \ \ ..., \ \ N_{y^r} ]^T  \\
& \theta_{y;X} = [\ln p_{y^1}, \ \ ..., \ \ \ln p_{y^r}]^T \\
& h(y;X) = \frac{(\sum_{k=1}^{r} N_{y^k})!}{\prod_{k=1}^{r} N_{y^k} !}  \\
& A(\theta_{y;X}) = 0
\end{align*}
The gradients of the logarithm of the exponential family probability function w.r.t. the natural parameter $p_{y^k}$ equals to:
\begin{align}
\label{multie}
\frac{\partial LL}{\partial p_{y^k} } & = \frac{\partial}{\partial p_{y^k}} \left [\ln h(y;X) + \langle \theta_{y;X}, \ \ \phi(y;X) \rangle - A(\theta_{y;X}) \right ]  \nonumber \\
& = \langle \frac{\partial \theta_{y;X}}{\partial p_{y^k}},\ \ \phi(y;X) \rangle - \frac{\partial A(\theta_{y;X}) }{\partial p_{y^k}} \nonumber \\
& = \frac{1}{ p_{y^k}}  N_{y^k}. 
\end{align}

Because of the simplex constraint $\sum_k p_{y^k} = 1$ and the value range of a probability $0\leq p_{y^k} \leq 1$, in order to optimize the parameters separately while keeping these constraints satisfied, the multinomial regression models are usually converted into the form of a log-linear model with a new set of parameters $\{\beta_{y^1},  \ \ ..., \ \  \beta_{y^r} \}$ which satisfies $p_{y^k} = \frac{e^{\beta_{y^k}}}{\sum_{k'} e^{\beta_{y^{k'}} }}$.

To learn such log-linear models with multinomial variables with functional gradient boosting, one needs to view this probability distribution assumption from the perspective of a single example $(y_i; \textbf{x}_i)$ each of which attached with a function parameter $\psi_{y_i; \textbf{x}_i}$. Instead of optimizing the log-likelihood function w.r.t. the natural parameters $\{\beta_{y^1},  \ \ ..., \ \  \beta_{y^r} \}$, functional gradient boosting optimizes the objective function w.r.t. the function parameter $\psi_{y_i; \textbf{x}_i}$ for every example. 

Under the multinomial distribution assumption, the probability of a single example is defined as a function of parameter $\psi_{y_i; \textbf{x}_i}$ as: $p_{y_i; \textbf{x}_i} = \frac{e^{\psi_{y_i; \textbf{x}_i}}}{\sum_{k} e^{\psi_{y_i=y^k; \textbf{x}_i} }}$.
The probability for this example to have value $x_k$ can be treated as a categorical distribution, in which case the sufficient statistics in the general exponential family form~\ref{ef} can be derived as  $\phi(y_i;\textbf{x}_i) = [I_{y_i=y^1}, \ \ ..., \ \ I_{y_i=y^r} ]^T$ where $I$ is the indicator function. So the optimization problem w.r.t. $p_{y^k}$ can be converted to an optimization problem w.r.t. $\psi_{y^k}$. The gradients of the logarithm of function~\ref{ef} w.r.t. functional parameter $\psi_{y_i;\textbf{x}_i}$ can be derived by substituting Equation~\ref{multie} as:
\begin{align}
\frac{\partial LL_i}{\partial \psi_{y_i = y^k; \textbf{x}_i} } & =  \frac{1}{ p_{y_i; \textbf{x}_i}}  \cdot \frac{\partial p_{y_i; \textbf{x}_i}}{\partial \psi_{y_i = y^k; \textbf{x}_i}} \nonumber \\
& = I(y_i = y^k; \textbf{x}_i) - p(y_i = y^k; \textbf{x}_i).  
\end{align}

This is the gradient function employed by regression models with multinomial distribution assumption when learning with RFGB. At each step, a regression function is learned to fit these gradients.

\subsection{RFGB for Poisson Distributed Variables}

Among the discrete variables, there is an important class -- count variable which is ubiquitous in almost every domain and contains vast amount of information and great potential available for machine learning and data mining techniques to explore. 
Since such count data typically only take the non-negative integer values, there are numerous machine learning algorithms which model such count variable as categorical values~\cite{KollerFriedman09} by assuming the underlying distribution to be a multinomial one. Their applications covers various domains such as the prediction on prevalence of certain disease~\cite{luo15}, tourism demand forecasting~\cite{Hong2011},  online shopper demands prediction~\cite{FerreiraLS16}, etc. The key issue with this assumption is that it places an upper bound on the counts (i.e. limited number of categories) of the target. While in many domains, it is possible to obtain a reasonable upper bound, in other domains, this requires guessing a good one. 

Another common assumption that most prior work assumed is that of the Gaussian assumption over these count variables. However, low counts can lead to the left tail of the Gaussian distribution predict negative values for these counts. Subsequently, there are research directions~\cite{YangRAL12, YangRAL13a, pdn15} that model such data from a different perspective by assuming that these counts are distributed according to a Poisson. Such models have been extended to learn dependencies among multiple Poisson variables as Poisson graphical models. 
 
{\em Poisson distribution} (as shown in Figure~\ref{distributions} right)                                                                                                                                                                                                                                                                                                                                                                                                                                                                                                                                                                                                                   states that each sample is a instance of a count variable following Poisson distribution with one parameter $\lambda$.  
The probability distribution function of Poisson distribution is as following:
\begin{equation*}
p(y) = \frac{\lambda^y e^{-\lambda}}{y!},
\end{equation*}
where $y \in \{0, \ \ 1, \ \ 2\ \ ... \}$. The components in the general exponential family form~\ref{ef} can be converted from the point-wise perspective to functions of the natural parameters of Poisson distribution as:
\begin{align*}
& \phi(y_i;\textbf{x}_i) = y_i  \\
& \theta_{y_i;\textbf{x}_i} = \ln \lambda_{y_i;\textbf{x}_i} \\
& h(y_i;\textbf{x}_i) = \frac{1}{y_i !} \\
& A(\theta_{y_i;\textbf{x}_i}) = e^{\theta_{y_i;\textbf{x}_i}} = \lambda_{y_i;\textbf{x}_i}
\end{align*}

Then, the point-wise gradients of the log-likelihood function w.r.t. the natural parameter $\lambda_{y_i;\textbf{x}_i}$ can be derived as:
\begin{align}
\label{poisson1}
\frac{\partial LL_i}{\partial \lambda_{y_i;\textbf{x}_i} } & =  \frac{\partial}{\partial \lambda_{y_i;\textbf{x}_i}} \left [\ln h(y_i;\textbf{x}_i) + \langle \theta_{y_i;\textbf{x}_i}, \ \  \phi(y_i;\textbf{x}_i) \rangle - A(\theta_{y_i;\textbf{x}_i}) \right ] \nonumber \\
& = \langle \frac{\partial \theta_{y_i;\textbf{x}_i}}{\partial \lambda_{y_i;\textbf{x}_i}},\ \  \phi(y_i;\textbf{x}_i) \rangle - \frac{\partial A(\theta_{y_i;\textbf{x}_i}) }{\partial \lambda_{y_i;\textbf{x}_i}} \nonumber \\
& = \frac{1}{ \lambda_{y_i;\textbf{x}_i}}  y_i - 1. 
\end{align}

Here, $\lambda_{y_i;\textbf{x}_i}$ is the Poisson parameter for the probability distribution of the target variable of the $i^{th}$ example given the learned discriminative Poisson model and the values of $y_i$'s parent nodes $X_i$, which are embodied by the subscript $_{y_i;\textbf{x}_i}$.

For relational functional gradient boosting, we define the function parameters as $\psi_{y_i;\textbf{x}_i}$ so that $\lambda_{y_i;\textbf{x}_i} = e^{\psi_{y_i;\textbf{x}_i}}$, in order to satisfy the constraints on $\lambda$ which is $\lambda >0$. Since exponential is a monotonic increasing function, the optimal Poisson parameters $\lambda$ can be found by boosting the log-likelihood function with respect to the function parameters $\psi$. Based on Equation~\ref{poisson1}, the gradients of the log-likelihood function of \ref{ef} w.r.t. functional parameter $\psi_{y_i;\textbf{x}_i}$ is derived as:
\begin{align}
\label{poisson2}
\frac{\partial LL_i}{\partial \psi_{y_i;\textbf{x}_i} } &  = \frac{\partial LL_i}{\partial \lambda_{y_i;\textbf{x}_i} } \cdot \frac{\partial \lambda_{y_i;\textbf{x}_i}}{\partial \psi_{y_i;\textbf{x}_i} } \nonumber  \\
& = (\frac{1}{ \lambda_{y_i;\textbf{x}_i}}  y_i - 1) \cdot  e^{\psi_{y_i;\textbf{x}_i}} = y_i - \lambda_{y_i;\textbf{x}_i}. 
\end{align}
Equation~\ref{poisson2} is the gradient function used by the additive RFGB to optimize a discriminative Poisson model.



\subsection{RFGB for Gaussian Distributed Variables}
Now, Let us turn to the continuous variables which are commonly being assumed as normally distributed. \textit{Gaussian distribution} assumes that a continuous variable follows the probability density function:
\begin{equation}
p(y) = \frac{1}{\sqrt {2\pi}\sigma} e ^{-\frac{(y-\mu)^2}{2 \sigma^2}}
\end{equation}
where $\mu$ and $\sigma$ are parameters of Gaussian distribution and indicate the mean and standard deviation of the Gaussian distribution respectively. 

From the point-wise perspective, the parameters in the general exponential family form~\ref{ef} can be calculated by functions of natural parameters $\mu_{y_i;\textbf{x}_i}$ and $\sigma_{y_i;\textbf{x}_i}$ :
\begin{align*}
& \phi(y_i;\textbf{x}_i) = [y_i, \ \ y_i^2]^T  \\
& \theta_{y_i;\textbf{x}_i} = [\frac{\mu_{y_i;\textbf{x}_i}}{\sigma_{y_i;\textbf{x}_i}^2}, \ \ -\frac{1}{2\sigma_{y_i;\textbf{x}_i}^2} ]^T \\
& h(y_i;\textbf{x}_i) = \frac{1}{\sqrt {2\pi}} \\
& A(\theta_{y_i;\textbf{x}_i}) = \frac{\mu_{y_i;\textbf{x}_i}^2}{2\sigma_{y_i;\textbf{x}_i}^2} + \ln \sigma_{y_i;\textbf{x}_i}
\end{align*}

Then, the point-wise gradients of the log-likelihood function w.r.t. the natural parameter $\mu_{y_i;\textbf{x}_i}$ can be derived as:
\begin{align}
\label{gaussian1}
\frac{\partial LL_i}{\partial \mu_{y_i;\textbf{x}_i} } & =  \frac{\partial}{\partial \mu_{y_i;\textbf{x}_i}} \left [\ln h(y_i;\textbf{x}_i) + \langle \theta_{y_i;\textbf{x}_i},\ \   \phi(y_i;\textbf{x}_i) \rangle - A(\theta_{y_i;\textbf{x}_i})  \right ] \nonumber \\
& = \langle \frac{\partial \theta_{y_i;\textbf{x}_i}}{\partial \mu_{y_i;\textbf{x}_i}},\ \  \phi(y_i;\textbf{x}_i) \rangle - \frac{\partial A(\theta_{y_i;\textbf{x}_i}) }{\partial \mu_{y_i;\textbf{x}_i}} \nonumber \\
& = \frac{1}{ \sigma_{y_i;\textbf{x}_i}^2}  (y_i - \mu_{y_i;\textbf{x}_i})
\end{align}

The point-wise gradients of the log-likelihood function w.r.t. the natural parameter $\sigma_{y_i;\textbf{x}_i}$ can be derived as:
\begin{align}
\label{gaussian2}
\frac{\partial LL_i}{\partial \sigma_{y_i;\textbf{x}_i} } & =  \frac{\partial}{\partial \sigma_{y_i;\textbf{x}_i}} \left [ \ln h(y_i;\textbf{x}_i) + \langle \theta_{y_i;\textbf{x}_i}, \ \  \phi(y_i;\textbf{x}_i) \rangle - A(\theta_{y_i;\textbf{x}_i}) \right ]  \nonumber \\
& = \langle \frac{\partial \theta_{y_i;\textbf{x}_i}}{\partial \sigma_{y_i;\textbf{x}_i}},\ \  \phi(y_i;\textbf{x}_i) \rangle - \frac{\partial A(\theta_{y_i;\textbf{x}_i}) }{\partial \sigma_{y_i;\textbf{x}_i}} \nonumber \\
& = \frac{-2\mu_{y_i;\textbf{x}_i}}{\sigma_{y_i;\textbf{x}_i}^3} \cdot y_i + \frac{1}{\sigma_{y_i;\textbf{x}_i}^3} \cdot y_i^2 + \frac{\mu_{y_i;\textbf{x}_i}^2}{\sigma_{y_i;\textbf{x}_i}^3} - \frac{1}{\sigma_{y_i;\textbf{x}_i}} \nonumber \\
& = \frac{1}{ \sigma_{y_i;\textbf{x}_i}^3}  (y_i - \mu_{y_i;\textbf{x}_i})^2 - \frac{1}{\sigma_{y_i;\textbf{x}_i}}
\end{align}

Here, $\sigma_{y_i;\textbf{x}_i}$ and $\mu_{y_i;\textbf{x}_i}$ are the point-wise parameters of the Gaussian distribution. Since there is no limit on the value range of Gaussian parameters (i.e. $\mu \in (-\infty, \infty)$, $\sigma \in (-\infty, \infty)$), it is reasonable to define the functional parameters as $\psi^{\mu}_{y_i;\textbf{x}_i} = \mu_{y_i;\textbf{x}_i}$ and $\psi^{\sigma}_{y_i;\textbf{x}_i} = \sigma_{y_i;\textbf{x}_i}$. The gradients of the log-likelihood function of \ref{ef} w.r.t. the functional parameters are equal to:
\begin{align}
& \frac{\partial LL_i}{\partial \psi^{\mu}_{y_i;\textbf{x}_i} }  = \frac{1}{ {\psi^{\sigma}_{y_i;\textbf{x}_i} }^2}  (y_i - \psi^{\mu}_{y_i;\textbf{x}_i}) \\
& \frac{\partial LL_i}{\partial \psi^{\sigma}_{y_i;\textbf{x}_i} }  = \frac{1}{ {\psi^{\sigma}_{y_i;\mathbf{x}_i}}^3}  (y_i - \psi^{\mu}_{y_i;\textbf{x}_i})^2 - \frac{1}{\psi^{\sigma}_{y_i;\textbf{x}_i}}
\end{align} 

Hence, finding the optimal parameters in the feature space can be converted to optimizing the log-likelihood function in the function space by learning a relational regression function to fit the summation of those gradient values $\psi^{\mu}_{y_i;\textbf{x}_i}$ and $\psi^{\sigma}_{y_i;\textbf{x}_i}$ in each iteration.
Similar as in the relational Multinomial models and relational Poisson models, the subscript $_{y_i;\textbf{x}_i}$ indexes the $i^{th}$ example's corresponding functional parameters, it actually indicates the condition of the parameters, i.e.  the current estimation of mean $\psi^{\mu}$ or variance $\psi^{\sigma}$ of the target variable $y_i$ given the learned discriminative relational Gaussian model and the values of $y_i$'s parent nodes $X_i$. Note that, $\psi^{\mu}_{y_i;\textbf{x}_i}$ and $\psi^{\sigma}_{y_i;\textbf{x}_i}$ convey not only the quantitative parameters of the learned model but also the qualitative logic structures of the relational model in each iteration. 

The process of the proposed RFGB learning approach for hybrid probabilistic logic models is shown in Algorithm~\ref{hrfgbalgo}. The key difference from the standard RFGB lies in the functions for generating examples (i.e. \textit{GenGaussianEgs}, \textit{GenPoissonEgs} and \textit{GenMultiNEgs}) which follow different gradient functions under different probability distribution assumptions for different value types.

\begin{figure}[htbp!]
	\renewcommand\figurename{Algorithm}
	\caption{Additive RFGB for Learning Hybrid Probabilistic Logic Models}
	\label{hrfgbalgo}
	\begin{algorithmic}[1]
		\Function{HPLM-ARFGB}{$Data$}
		\For {$1 \leq t \leq T$}  
		\Comment{Iterate through T predicates}
		\For {$1 \leq m \leq M$} 
		\Comment{Iterate through M gradient steps}
		\Switch {$ValueType(t)$}\Comment{Generate examples}
		\Case{$Gaussian$} 
		\State $S_t := $\textsc{GenGaussainEgs}$(t, Data, F^t_{m-1})$
		\EndCase
		\Case{$Poisson$}
		\State $S_t := $\textsc{GenPoissonEgs}$(t, Data, F^t_{m-1})$
		\EndCase
		\Case{$MultiN$}
		\State $S_t := $\textsc{GenMultiNEgs}$(t, Data, F^t_{m-1})$
		\EndCase
		\EndSwitch
		
		\State $\Delta_{m}(t):=$\textsc{FitRegressTree}$(S_t, t)$ 
		\Comment{Regression Tree learner}
		\State $F_m^t := F_{m-1}^t + \Delta_{m}(t)$ 
		\Comment{Update model}
		\EndFor
		\State $P( y^t | \mathbf{x}^t ) \propto F_M^t(y^t)$ 
		\EndFor
		\EndFunction
		
		\vspace{8mm}
		\Function{FitRegressionTree}{$S$}
		\State Tree := createTree($P(X)$)
		\State Beam := \{root(Tree)\}
		\While {$numLeaves(Tree) \le L$}
		\State Node := popBack(Beam)\\
		\Comment{Node w/ worst score}
		\State C := createChildren(Node) \\
		\Comment{Create children}
		\State BN := popFront(Sort(C, S)) \\
		\Comment{Node w/ best score}
		\State addNode(Tree, Node, BN)
		\State \Comment{Replace Node with BN}
		\State insert(Beam, BN.left, BN.left.score) 
		\State insert(Beam, BN.right, BN.right.score)
		\EndWhile \\
		\Return Tree
		\EndFunction
			
	\end{algorithmic}
\end{figure}

\begin{figure}[htbp!]
	\begin{algorithmic}[1]
		
		\Function{GenGaussianEgs}{$t,Data,F$}
		\State $y := grounded \ \ target \ \ predicate \ \ t$
		\State $\mathbf x := grounded \ \ other \ \ predicates $		
		\State $\mathbf S := \emptyset$
		\For { $1 \leq i \leq N$} 
		\Comment{Iterate over all examples}
		\State $\Delta_{\mu}(y _i;\mathbf{x}_i) := \frac{(y_i - F^{\mu} ({y_i;\textbf{x}_i}) )}{ {F^{\sigma}({y_i;\textbf{x}_i}) }^2}   $
		\State $\Delta_{\sigma}(y_i;\mathbf{x}_i) := \frac{1}{ {F^{\sigma} ({y_i;\mathbf{x}_i}}) ^3}  (y_i - F^{\mu} ({y_i;\textbf{x}_i}) )^2 - \frac{1}{F^{\sigma} ({y_i;\textbf{x}_i}) } $ 
		\State $ \mathbf S := \mathbf S \cup {[(y_i), \Delta_{\mu}(y_i;;\mathbf{x}_i); (y_i), \Delta_{\sigma}(y_i;;\mathbf{x}_i) ]}$ \\
		\Comment{Update relational regression examples}
		\EndFor
		\\
		\Return $\mathbf S$ \Comment{Return regression examples}
		\EndFunction
		
		\vspace{8mm}
		\Function{GenPoissonEgs}{$t, Data,F$}
		\State $y := grounded \ \ target \ \ predicate \ \ t$
		\State $\mathbf x := grounded \ \ other \ \ predicates $
		\State $ S := \emptyset$
		\For { $1 \leq i \leq N$} 
		\Comment{Iterate over all examples}
		\State $\Delta(y_i;\mathbf{x}_i) := y_i - e^{F ({y_i;\textbf{x}_i}) } $ 
		\State $ S :=  S \cup {[(y_i), \Delta(y_i;;\mathbf{x}_i)]}$ 
		\Comment{Update relational regression examples}
		\EndFor
		\\
		\Return $S$ \Comment{Return regression examples}
		\EndFunction
		
		\vspace{8mm}				
		\Function{GenMultiNEgs}{$t, Data,F$}
		\State $y := grounded \ \ target \ \ predicate \ \ t$
		\State $\mathbf x := grounded \ \ other \ \ predicates $
		\State $\mathbf S := \emptyset$
		\For { $1 \leq i \leq N$} 
		\Comment{Iterate over all examples}
		\For {$1 \leq k \leq K$}
		\Comment{Iterate over all possible values of $y$}
		\State Compute $P(y_i = y^k | \textbf{x}_i)$\\
		\Comment {Probability of an example being the $k^{th}$ value of $y$}
		\State $\Delta_{k}(y_i;\mathbf{x}_i) := I(y_i = y^k; \textbf{x}_i) - P(y_i = y^k; \textbf{x}_i) $
		\EndFor
		\State $ \mathbf S := \mathbf S \cup {[(y_i), \Delta_1(y_i;;\mathbf{x}_i); (y_i), \Delta_2(y_i;;\mathbf{x}_i); ... ; (y_i), \Delta_K(y_i;;\mathbf{x}_i) ]}$ \\
		\Comment{Update relational regression examples}
		\EndFor
		\\
		\Return $\mathbf S$ \Comment{Return regression examples}
		\EndFunction
		
	\end{algorithmic}
\end{figure}			

\newpage
\subsection{Extensions}

The previous section presents the gradient function for hybrid RFGB when the parent predicates are discrete values. This section presents the way to extend the proposed learning approach to handle the continuous valued parents. Symbol $\mathbf X_G$ is used to indicate continuous parents and $\mathbf X_D$ for discrete valued parents. 

\textbf{Multinomial variables with mix boolean-numeric parents}: when the target predicate follows multinomial distribution and the parent variables are mix boolean-numeric. The SoftMax function can be used to calculate the conditional probability for the target variable given other grounded predicates as following:
\begin{equation}
\label{multi2}
p_{y_i; \textbf{x}_i} = \frac{\exp \left [ \psi^0_{y_i; \textbf{x}_{Di}} + \sum_{x^j \in \mathbf X_G}{ x_i^j \psi^j_{y_i; \textbf{x}_{Di}}}    \right ]} {\sum_{k} \exp \left [ \psi^0_{y_i=y^k; \textbf{x}_{Di}} + \sum_{x^j \in \mathbf X_G} {x_i^j \psi^j_{y_i=y^k; \textbf{x}_{Di}} } \right ]}
\end{equation}

\textbf{Poisson variables with mix boolean-numeric parents}: when the target predicate follows Poisson distribution and the parent variables are mix boolean-numeric. The conditional Poisson parameter for the target variable given other grounded predicates  can be calculated as a log-linear function of functional parameters from all Gaussian parents:
\begin{equation}
\label{poisson3}
\lambda_{y_i; \textbf{x}_i} =  \exp \left [ \psi^0_{y_i; \textbf{x}_{Di}} + \sum_{x^j \in \mathbf X_G}{x_i^j \psi^j_{y_i; \textbf{x}_{Di}}}    \right ]
\end{equation}

\textbf{Gaussian variables with mix boolean-numeric parents}: when the target predicate follows Gaussian distribution and the parent variables are mix boolean-numeric. The conditional Gaussian mean parameter for the target variable given other grounded predicates can be calculated as a linear combination of functional parameter of every Gaussian parent:
\begin{equation}
\label{gaussain2}
\mu_{y_i; \textbf{x}_i} =  \psi^0_{y_i; \textbf{x}_{Di}} + \sum_{x^j \in \mathbf X_G}{x_i^j \psi^j_{y_i; \textbf{x}_{Di}}} 
\end{equation}

In each iteration, instead of learning one relational regression tree, multiple regression trees are learned each of which corresponds to one Gaussian parent. The functional parameter update function can be derived following the chain rule of differentiating compound functions. If the gradient function involves the functional parameter of other Gaussian parents, then use the current model of the corresponding Gaussian parent to estimate the value for it. When the learning approach converges, the conditional probabilities for the target predicate can be calculated based on the corresponding probability distribution assumptions by combining the models learned for all the Gaussian parents (following Equation~\ref{multi2}, \ref{poisson3}, \ref{gaussain2})




\chapter{Real-world Applications}
\label{chap5}
\ifpdf
    \graphicspath{{Chapter5/Chapter5Figs/PNG/}{Chapter5/Chapter5Figs/PDF/}{Chapter5/Chapter5Figs/}}
\else
    \graphicspath{{Chapter5/Chapter5Figs/EPS/}{Chapter5/Chapter5Figs/}}
\fi

\section{Application of Soft-Margin RFGB to Large-Scale Hybrid Recommendation Systems}

With their rise in prominence, recommendation systems have greatly alleviated information overload for their users by providing personalized suggestions for countless products such as music, movies, books, housing, jobs, etc. 
Since the mid-1990s, not only new theories of recommender systems have been proposed but also their application softwares have been developed which involve various domains including e-government, e-business, e-commerce/e-shopping, e-learning, etc~\cite{LU201512}.
We consider a specific recommender system domain, that of job recommendations, and propose to build a hybrid job recommendation system with a novel statistical relational learning approach. This domain easily scales to billions of items including user resumes and job postings, as well as even more data in the form of user interactions with these items. CareerBuilder, the source of the data for our experiments, operates one of the largest job boards in the world. It has millions of job postings, more than 60 million actively-searchable resumes, over one billion searchable documents, and receives several million searches per hour~\cite{aljadda2014crowdsourced}. 
The scale of the data is not the only interesting aspect of this domain, however. The job recommendations use case is inherently relational in nature, readily allowing for graph mining and relational learning algorithms to be employed. 
As Figure~\ref{rgraph} shows, the jobs that are applied to by the same users as well as the users who share similar preferences are not i.i.d. but rather connected through application relations. If we treat every single job post or user as an object which has various attributes, the probability of a match between the target user and a job does not only depend on the attributes of these two target objects (i.e. target user and target job) but also the attributes of the related objects such as patterns of the user's previous applied jobs, behaviors of users living in the same city or having the same education level. As we show in this work, richer modeling techniques can be used to capture these dependencies faithfully.
However, since most of the statistical relational learning approaches involve a searching space exponential to the number of related objects, how to efficiently build a hybrid recommendation system with statistical relational learning in such a large scale real-world problem remains a challenge in this field.
\begin{figure*}[htbp]
	\begin{minipage}[b]{\textwidth}
		\centering
		\includegraphics[width=0.9\textwidth]{rgraph.png}
	\end{minipage}
	\caption{ \small Job recommendation domain}
	\label{rgraph}
\end{figure*}

One of the most popular recommender approaches is {\em content-based filtering}~\cite{Basu98recommendationas}, which exploits the relations between (historically) applied-to jobs and similar features among new job opportunities for consideration (with features usually derived from textual information). An alternative recommendation approach is based on {\em collaborative filtering}~\cite{Breese1998}, which makes use of the fact that users who are interested in the same item generally also have similar preferences for additional items. Clearly, using both types of information together can potentially yield a more powerful recommendation system, which is why model-based hybrid recommender systems were developed~\cite{Basilico2004}. While successful, these systems typically need extensive feature engineering to make the combination practical. 

The hypothesis which we sought to verify empirically was that recent advancements in the fields of machine learning and artificial intelligence could lead to powerful and deployable recommender systems. In particular, we assessed leveraging Statistical Relational Learning (SRL)~\cite{srlbook}, which combines the representation abilities of rich formalisms such as first-order logic or relational logic with the ability of probability theory to model uncertainty. We employed a state-of-the-art SRL formalism for combining content-based filtering and collaborative filtering. SRL can directly represent the probabilistic dependencies among the attributes from different objects that are related with each other through certain connections (in our domain, for example, the jobs applied to by the same user or the users who share the same skill or employer). SRL models remove the necessity for an extensive feature engineering process, and they do not require learning separate recommendation models for each individual item or user cluster, a requirement for many standard model-based recommendation systems~\cite{Breese1998,Pazzani1997}.  

We propose a hybrid model combining content-based filtering and collaborative filtering that is learned by a cost-sensitive statistical relational learning approach - soft-margin Relational Functional Gradient Boosting (RFGB) presented in Chapter~\ref{chap2}. Specifically, we define the target relation as $Match(User, Job)$ which indicates that the user--job pair is a match when the ground fact is true, hence that job should be recommended to the target user. The task is to predict the probability of this target relation $Match(User, Job)$ for users based on the information about job postings, the user profile, the application history, as well as application histories of users who have similar preferences or profiles as the target user. RFGB is a boosted model which contains multiple relational regression trees with additive regression values at the sink node of each path. These trees are set of weaker classifiers that capture the conditional dependences between the target relation $Match(User, Job)$  and all other predicates in the data. Our hypothesis is that combining this set of weaker classifiers would yield a more accurate and efficient model that presents the $Match$ between the target user and the target job with the features from both content-based filtering and collaborative filtering.

In addition, this domain has practical requirements which must be considered. For example, we would rather overlook some of the candidate jobs that could match the users (false negatives) than send out numerous spam emails to users with inappropriate job recommendations (false positives). The cost matrix thus does not contain uniform cost values, but instead needs to represent a higher cost for the user--job pairs that are false positives compared to those that are false negatives, i.e. precision is preferred over recall. To incorporate such domain knowledge on the cost matrix, we adapted the previous work from ~\cite{Yang14}, which extended RFGB by introducing a penalty term into the objective function of RFGB so that the trade-off between the precision and recall can be tuned during the learning process. 

In summary, we considered the problem of matching a user with a job and developed a hybrid content-based filtering and collaborative filtering approach. We adapted a successful SRL algorithm for learning features and weights and are the first to implement such a system in a real-world big data context. Our algorithm is capable of handling different costs for false positives and false negatives making it extremely attractive for deploying within many kinds of recommendation systems, including those within the domain upon which we tested. Our proposed approach has three main innovations: 1. it is the first work which employs probabilistic logic models to build a real-world large-scale job recommendation system; 2. it is the first work which allows the recommender to incorporate special domain requirements of an imbalanced cost matrix into the model learning process; 3. it is the first to prove the effectiveness of statistical relational learning in combining the collaborative filtering and content-based filtering with real-world job recommendation system data.

\subsection{Recommendation System}

Recommendation systems usually handle the task of estimating the relevancy or ratings of items for certain users based on information about the target user--item pair as well as other related items and users. The recommendation problem is usually formulated as $f: U\times I \rightarrow R$ where $U$ is the space of all users, $I$ is the space of all possible items and $f$ is the utility function that projects all combinations of user-item pairs to a set of predicted ratings $R$ which is composed by nonnegative integers. For a certain user $u$, the recommended item would be the item with the optimal utility value, i.e. $u_i^* = argMax_{i \in I}  f(u,i)$. The user space $U$ contains the information about all the users, such as their demographic characteristics, while the item space $I$ contains the feature information of all the items, such as the genre of the music, the director of a movie, or the author of a book.

Generally speaking, the goal of \textit{content-based filtering} is to define recommendations based upon feature similarities between the items being considered and items which a user has previously rated as interesting~\cite{Adomavicius2005}, i.e. for the target user-item rating $f(\hat u, \hat i)$, \textit{content-based filtering} would predict the optimal recommendation based on the utility functions of $f(\hat u, I_h)$ which is the historical rating information of user $\hat u$ on items ($I_h$) similar with $\hat i$. 
Given their origins out of the fields of information retrieval and information filtering, most content-based filtering systems are applied to items that are rich in textual information. From this textual information, item features $I$ are extracted and represented as \textit{keywords} with respective weighting measures calculated by certain mechanisms such as the \textit{term frequency/inverse document frequency (TF/IDF)} measure~\cite{Salton89}. The feature space of the user $U$ is then constructed from the feature spaces of items that were previously rated by that user through various keyword analysis techniques such as averaging approach~\cite{Rocchio71}, Bayesian classifier~\cite{Pazzani1997}, etc. Finally, the utility function of the target user-item pair $f(\hat u, \hat i)$ is calculated by some scoring heuristic such as the cosine similarity~\cite{Salton89} between the user profile vector and the item feature vector or some traditional machine learning models~\cite{Pazzani1997}. \textit{Overspecialization} is one of the problems with content-based filtering, which includes the cases where users either get recommendations too similar to their previously rated items or otherwise never get recommendations diverse enough from the items they have already seen. Besides, since the model-based content-based filtering builds its recommendation model based on the previous rated items of the target user, it requires a significant amount of items to be rated in advance in order to give accurate recommendations especially for the probabilistic machine learning models which require the number of training examples at the exponential scale of the dimension of the feature space. 

On the other hand, the goal of the \textit{collaborative filtering} is to recommend items by learning from users with similar preferences~\cite{Adomavicius2005,Su2009,NIPS2015_5938,WANG201680}, i.e. for the target user-item rating $f(\hat u, \hat i)$, \textit{collaborative filtering} builds its belief in the best recommendation by learning from the utility functions of $f(U_s, \hat i)$ which is the rating information of the user set $U_s$ that has similar preferences as the target user $\hat u$. The commonly employed approaches fall into two categories: \textit{memory-based (or heuristic-based)} and \textit{model-based} systems. The heuristic-based approaches usually predict the ratings of the target user-item pair by aggregating the ratings of the most similar users for the same item with various aggregation functions such as mean, similarity weighted mean, adjusted similarity weighted mean (which uses relative rating scales instead of the absolute values to address the rating scale differences among users), etc. The set of most similar users and their corresponding weights can be decided by calculating the correlation (such as Pearson Correlation Coefficient~\cite{Resnick94}) or distance (such as cosine-based~\cite{Breese1998} or mean squared difference) between the rating vectors of the target user and the candidate user on common items. Whereas model-based algorithms are used to build a recommendation system by training certain machine learning models~\cite{Breese1998, Salakhutdinov2007,  Si03,sahoo2010hidden} based on the ratings of users that belong to the same cluster or class as the target user. 
Hence, prior research has focused on applying statistical relational models to collaborative filtering systems~\cite{Getoor99,newton2004hierarchical,gao2007recommendation, huang2005unified}.
Although collaborative filtering systems can solve the overspecialization problem present in the content-based filtering approach, it has its own problems as well, such as the new user/item problem (commonly known as the "cold start" problem) and the \textit{sparsity} problem, which occurs when the number of users ratings on certain items is not sufficient. 
Hence, there are works focusing on enhancing collaborative filtering systems for solving such problems. Shambour et al.~\cite{Shambour2011} proposed a G2B recommendation e-services which alleviated the sparsity and cold start problems by employing additional domain knowledges of trust and trust propagation.

There are \textit{Hybrid approaches} which combine collaborative filtering and content-based filtering into a unified system~\cite{Basilico2004,de2010combining,balabanovic1997fab}. For instance Basilico et al.~\cite{Basilico2004} unified content-based and collaborative filtering by engineering the features based on various kernel functions, then trained a simple linear classifier (Perceptron) in this engineered feature space.

There are some research focusing on job recommender systems. However, most of them only exploit the techniques of content-based filtering~\cite{Guo14, Almalis15, Salazar2015, Malherbe14, Diaby13}. Hong et al.~\cite{Hong13} proposed a hybrid job recommender system by profiling the users based on the historical applied jobs and behaviors of job applicants. Lu et al.~\cite{Yao13} proposed a directed weighted graph which represents the content-based and interaction-based relations among users, jobs and employers with directed or bidirectional edges. It computes the content-based similarity between any two profiles of objects (user, employer or job). The key difference of our model from theirs is that the graph they used is not a machine learning model trained from the historical data, but rather built based on known facts of the target objects whereas our model is a first-order logic probabilistic model trained with historical data and only partially grounded with the related objects when it is necessary for inference on the target objects. Pacuk et al.~\cite{PacukSWWW16} also exploited gradient boosting. But they only built a content-based filtering recommender by using the standard gradient boosting. We built a hybrid recommender with relational functional gradient boosting, which can capture the dependencies among the features not only from the target user-item pair but also from similar users. Besides, our model is a cost-sensitive learning approach which allows the tunning of precision and recall in a principled way. 

The most related work to ours is \cite{HoxhaR13}, where they proposed to use Markov Logic Networks to build hybrid models combining content-based filtering and collaborative filtering. Their work only employed one type of probabilistic logic model, which is demonstrated later in this paper to be a worse one. Besides, it did not consider the special requirement of many recommendation systems, which is that precision should be preferred over recall (or at least that the relative weights of the two should be configurable).

\subsection{Building Hybrid Job Recommendation System with Soft-Margin RFGB}

Traditional machine learning algorithms make a fundamental assumption about the data they try to model: the training samples are independent and identically distributed (i.i.d.), which is not the typical case in recommendation systems. In order to represent the data in a flat table, the standard model-based recommendation systems need an exhaustive feature engineering process to construct the user profile by aggregating the attributes over all the similar users who share the same background or similar preferences as the target user. The aggregation-based strategies are necessary because the standard algorithms require a regular flat table to represent the data. However, the number of similar users related to the target user may vary a lot among different individuals. For example, users with common preferences could have more similar users than the users with unique tastes. There are aggregation-based strategies~\cite{Perlich03} to make the feature amount identical for all the samples when extending the feature space. However, such strategies would inevitably lose meaningful information otherwise introduce some amount of noise.

We propose to employ SRL for the challenging task of implementing a hybrid recommendation system. Specifically, we consider the formulation of Relational Dependency Networks (RDNs)~\cite{Neville07}, which are approximate graphical models that are inferred using the machinery of Gibbs sampling.
Figure~\ref{rdn} shows a template model of RDNs. As can be seen, other than the attributes of the target user A and target job B, it also captures the dependencies between the target predicate $Match(A, B)$ and attributes from the similar user D and previous applied job C. The interpretation of the learned model will be explained in more detail in section~\ref{exp}.
\begin{figure}[htbp]
	\centering
	\includegraphics[width=0.6\textwidth]{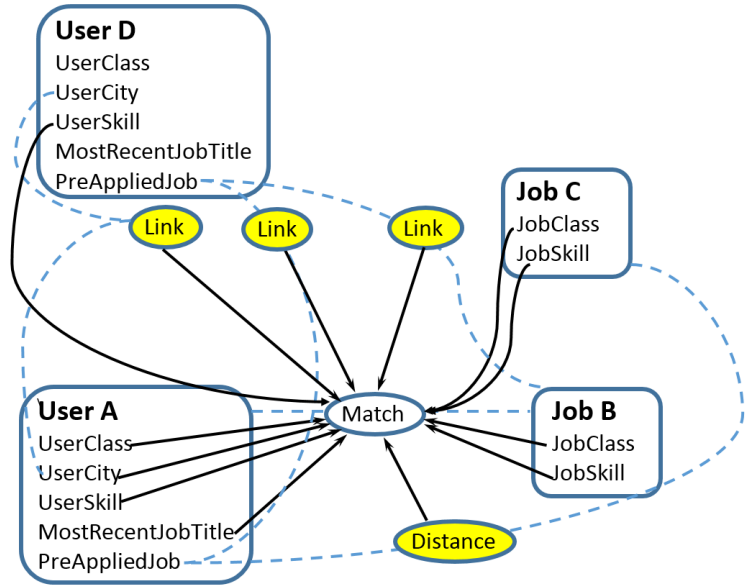}
	\caption{ \small Template Model of a Sample RDN. The target is $Match(UserA, JobB)$ while the related objects are User D (introduced by the link nodes) and previous applied Job C. Note that D and C are first-order variables which could have multiple groundings for different target user--job pairs. }
	\label{rdn}
\end{figure}

As an approximation of Bayesian Networks, Dependency Networks (DNs) make the assumption that the joint distributions can be approximated as the product of individual conditional probability distributions and that these conditional probability distributions are independent from each other. Since it prevents the need for acyclicity checking, the structure and conditional probability of each node can be optimized separately by certain local search strategy.
RDNs extend DNs to relational data and are considered as one of the most successful SRL models that have been applied to real-world problems. Hence, we propose to construct a hybrid recommendation system by learning an RDN using a state-of-the-art learning approach--Relational Functional Gradient Boosting (RFGB) which has been proven to be one of the most efficient relational learning approaches~\cite{Natarajan2012}. In our case, the target predicate is \textit{Match(User, Job)}, and the parents $\mathbf{X}_i$ would be the attributes of the target user and target job, and the jobs previously applied to by the target user and similar users sharing the same preferences. $\hat{y}_i$ is the true label for a user--job pair which is $1$ for a positive matching pair and $0$ for a negative matching pair. The key assumption is that the conditional probability of a target grounding $y_i$, given all the other predicates, is modeled as a sigmoid function. 

Following the work of Yang et al.~\cite{Yang14}, we also take into account the domain preference on Precision over Recall when learning the hybrid job recommendation system.  In this specific application, 
the gradient of the objective function w.r.t $\psi(y_i=1;\textbf{X}_i)$ equals to:
\begin{equation}
\Delta = I(\hat{y}_i=Match) \, - \, \lambda P(y_i=Match;\textbf{X}_i).
\label{simp_grad_hrs}
\end{equation}
where $\lambda\, =$
\begin{equation}
\begin{array}{l}
\left\{ \begin{array}{ll}
\displaystyle{For \ \ matching \ \ User-Job \ \ pairs:} \\
\displaystyle{\frac{1}{P(y_i=Match;\textbf{X}_i)+P(y_i=MisMatch;\textbf{X}_i)\cdot e^{\alpha}}} \\ 
\\
\displaystyle{For\ \ mismatching\ \ User-Job\ \ pairs:} \\
\displaystyle{\frac{ e^{\beta}}{P(y_i=Match;\textbf{X}_i)\cdot e^{\beta} +P(y_i=MisMatch;\textbf{X}_i)}} 
\end{array} \right.
\end{array}
\label{lambda_hrs}
\end{equation}

As shown above, the cost function $c(\hat{y}_i, y_i)$ is controlled by $\alpha$ when a potentially matching job being ruled out by the recommender, while being controlled by $\beta$ when a mis-matching job being recommended. The cost matrix of our approach can be formally defined as Table~\ref{costmatrix} shows.

\begin{table}[htbp]
	\caption{Cost Matrix}
	\label{costmatrix}
	\small
	\begin{tabular}{ccc|l}
	
		& & \multicolumn{2}{ c }{Actual Class} \\ \cline{3-4}
		& & \multicolumn{1}{ |c| } {True} &  \multicolumn{1}{ c| } {False}  \\ \cline{2-4}
		\multicolumn{1}{ c }{\multirow{2}{*}{Predicted } } &
		\multicolumn{1}{ |c| }{True} & 0 & \multicolumn{1}{ c| } {$ I(\hat{y}_i=1) \, - \, \frac{ {P(y_i=1;\textbf{X}_i)}}{P(y^\prime=1;\textbf{X}_i) +P(y^\prime=0;\textbf{X}_i) \cdot e^{ - \beta} } $} \\ \cline{2-4}
		\multicolumn{1}{ c }{Class}                        &
		\multicolumn{1}{ |c| }{False} & \multicolumn{1}{ c| } {$ I(\hat{y}_i=1) \, - \, \frac{P(y_i=1;\textbf{X}_i)}{P(y^\prime=1;\textbf{X}_i)+P(y^\prime=0;\textbf{X}_i)\cdot e^{\alpha}} $} &  \multicolumn{1}{ c| } 0 \\ \cline{2-4}

	\end{tabular}
\end{table}

As the cost matrix shows, the influence of false negative and false positive examples on the final learned model can be directly controlled by tuning the parameters $\alpha$ and $\beta$ respectively.

Now, consider the special requirement on the cost matrix in most job recommendation systems, which is that we would rather miss certain candidate jobs which to some extent match the target user than send out recommendations that are not appropriate to the target user. In other words, we prefer high precision as long as the recall maintains above such a reasonable value that the system would not return zero recommendation for the target user. 

Since $\alpha$ is the parameter controlling the weights of false negative examples, we simply assign it as 0 which makes $\lambda=1 \, / \, \sum_{y^\prime} \, [P(y^\prime; \textbf{X}_i)] = 1$ for misclassified positive examples. As a result, the gradient of the positive examples is the same as it was in the original RFGB settings. 

For the false positive examples, we assign a harsher penalty on them, so that the algorithm would put more effort into classifying them correctly in the next iteration. 
According to Equation~\ref{lambda_hrs}, when it is a negative example ($\hat{y}_i = 0$), we have
\begin{equation}
\lambda= \frac{1}{P(y_i=Match;\textbf{X}_i)+P(y_i=MisMatch;\textbf{X}_i)\cdot e^{-\beta}}. \nonumber
\end{equation}
As $\beta \rightarrow \infty$, $e^{-\beta} \rightarrow 0$, hence $\lambda \rightarrow 1\,/\,P(y_i=Match; \textbf{X}_i)$, so
\begin{equation*}
\Delta = 0 \, - \, \lambda P(y_i=Match;\textbf{X}_i) \rightarrow -1 
\end{equation*}

This means that the gradient is pushed closer to its maximum magnitude $|-1|$, no matter how close the predicted probability is to the true label. On the other hand, when $\beta \rightarrow -\infty$,  $\lambda \rightarrow  0$, hence $\Delta \rightarrow 0$, which means that the gradients are pushed closer to their minimum value of 0. 
So, in our experiment, we set $\beta > 0$, which amounts to putting a large cost on the false positive examples.

The proposed approach follows the same procedure as described in Algorithm~\ref{softrfgb}. Figure~\ref{algo_hrs} shows the function for generating the regression examples in this application use case where the parameter $\lambda$ for positive and negative examples are calculated respectively based on the set values of parameters $\alpha=0$ and $\beta >0$ to satisfy the domain requirement of recommendation systems. 
 
\begin{figure}[htbp]
	\begin{algorithmic}[1]
		
		\Function{GenSoftMEgs}{$Data,F$}
		\Comment{ Example generator}
		\State $S := \emptyset$
		\For { $1 \leq i \leq N$} 
		\Comment{Iterate over all examples}
		\State Compute $ P(y_i=1|\mathbf{x}_i) $ 
		\Comment{Probability of the example being true}
		\If {$\hat{y}_i=1$}
		\State $\lambda = 1$
		\Comment{Compute parameter $\lambda$ for positive examples}
		\Else
		\State $\lambda = 1/[ P(y_i = 1|\mathbf{x}_i) + P(y_i = 0|\mathbf{x}_i)\cdot e^{-\beta}]$
		\State \Comment{Compute parameter $\lambda$ for negative examples}
		\EndIf
		\State $\Delta(y_i;\mathbf{x}_i) := I(\hat {y}_i=1) - \lambda P(y_i = 1|\mathbf{x}_i)$ 
		\Comment{Cost of current example }
		\State $ S := S \cup {[(y_i), \Delta(y_i;\mathbf{x}_i))]}$ 
		\Comment{Update relational regression examples}
		\EndFor
		\\
		\Return $S$ \Comment{Return regression examples}
		\EndFunction
		
	\end{algorithmic}
	\caption{Cost-sensitive RFGB for Job Recommendation Systems}
	\label{algo_hrs}
\end{figure}

In job recommendation systems, the major goal is typically not to have misclassified negative examples (false positive). As a result, we need to eliminate the noise/outliers in negative examples as much as possible. Most algorithms generate negative examples by randomly drawing objects from two related variables, and the pair that is not known as positively-related based on the given facts is assumed to be a negative pair. However, in our case, if we randomly draw instances from $User$ and $Job$, and assume it is a negative example if that grounded user never applied to that grounded job, it could introduce a lot of noise into the data since not-applying does not necessarily indicate the job not matching the user. For example, it could simply be due to the fact that the job has never been seen by the user. Hence, instead of generating negative instances following a ``closed-world assumption'', as most of the relational approaches did, we instead generated the negative examples by extracting the jobs that were sent to the user as recommendations but were not applied to by the user. In this way, we can guarantee that this User--Job pair is indeed not matching. 

\subsection{Experimental Results}
\label{exp}

We extracted 4 months of user job application history and active job posting records and evaluated our proposed model on that data. Our goal was to investigate whether our proposed model can efficiently construct a hybrid recommendation system with cost-sensitive requirements by explicitly addressing the following questions:
\begin{description}	
	\item{\bf(Q1)} How does combining collaborative filtering improve the performance compared with content-based filtering alone?
	\item{\bf(Q2)} Can the proposed cost-sensitive SRL learning approach reduce false positive prediction without sacrificing too much on other evaluation measurements?
\end{description}

To answer these questions, we extracted 9 attributes from user resumes as well as job postings, which are defined as first-order predicates: \textit{JobSkill(JobID, SkillID)}, \textit{UserSkill(UserID, SkillID)}, \textit{JobClass (JobID, ClassID)}, \textit{UserClass(UserID, ClassID)}, \textit{PreAppliedJob(UserID, JobID)}, \textit{UserJobDis(UserID, JobID, Distance)},\\  \textit{UserCity(UserID, CityName)}, \textit{MostRecentCompany(UserID, CompanyID)}, \textit{MostRecentJobTitle(UserID, JobTitle)}.

There are 707820 total job postings in our sample set, and the number of possible instances the first order variables can take is shown in the Table~\ref{variables}. 
\begin{table}[htbp]
	\caption{Feature Space}
	\centering
	\small
	\begin{tabular}{|c|c|c|c|}
		\hline
		Variable Name & SkillID & ClassID & Distance  \\ \hline
		Num of Instances & 8,534 & 1,867 & 4  \\ \hline
		Variable Name & CityName & CompanyID & JobTitle \\ \hline
		Num of Instances & 22,137 & 1,154,623 & 823,733 \\ \hline
	\end{tabular} 
	\label{variables}	
\end{table}

Information on the \textit{JobClass} and \textit{UserClass} are extracted based upon the work of Javed et al.~\cite{Javed2015}.
The other features related to users are \textit{UserSkill}, \textit{UserCity}, \textit{MostRecentCompany} and \textit{MostRecentJobTitle} which are either extracted from the user's resume or the user's profile document, whereas the job feature \textit{JobSkill} represents a desired skill extracted from the job posting. Predicate \textit{UserJobDis} indicates the distance between the user (first argument) and the job (second argument), which is calculated based on the user and job locations extracted from respective documents. The \textit{UserJobDis} feature is discretized into 4 classes (1: $<15$ mile; 2: [15 miles, 30 miles); 3: [30 miles, 60 miles]; 4: $>60$ miles). The predicate \textit{PreAppliedJob} defines the previous applied jobs and serves as an independent predicate which indicates whether the target user is in a cold start scenario, as well as acts as a bridge which introduces into the searching space the attributes of other jobs related to the target user during the learning process. 

To incorporate collaborative filtering, we use three additional first-order predicates: \textit{CommSkill(UserID1, UserID2)}, \textit{CommClass (UserID1, UserID2)} and \textit{CommCity(UserID1, UserID2)} which are induced from the given groundings of the predicates \textit{UserSkill}, \textit{UserClass} and \textit{UserCity} and also serve as bridges which introduce features of other users who share the similar background with the target user.

The performance of our model is evaluated in 3 different user classes, each of which has its data scale description shown in Table~\ref{domains}.  

\begin{table*}[!htbp]
	\caption{Domains} 
	\centering
	\begin{tabular}{|c|c|c|c|c|c|c|}
		\hline
		\multirow{2}{*}{JobTitle} & \multicolumn{3}{c}{Training} & \multicolumn{3}{|c|}{Test} \\ \cline{2-7}
		~& pos & neg & facts & pos & neg & facts \\ \hline
	    Retail Sales Consultant & 224 & 6,973 & 13,340,875 & 53 & 1,055 & 8,786,938 \\ \hline
		Case Manager & 387 & 35,348 & 13,537,324 & 87 & 5,804 & 8,815,216\\ \hline
		District Manager & 358 & 16,014 & 13,396,635 & 87 & 3,521 & 8,798,522\\ \hline
		
	\end{tabular}	
	\label{domains}
\end{table*}

As Natarajan et al. discovered in their work~\cite{Natarajan2012}, limiting the number of leaves in each tree and learning a set of small trees can improve the learning efficiency as well as prevent overfitting compared with learning a single complex tree.
To choose the appropriate number of trees, we sampled a smaller tuning set of just 100 examples. Our tuning set results showed that choosing the number of trees beyond 20 did not improve the performance. However, between 10 and 20 trees the performance had significant improvement. This is similar to the original observation of Natarajan et al.~\cite{Natarajan2012} and hence, we followed their discoveries and set the number of iterations $M$ as 20 and maximum number of leaves $L$ as 8.

The results are shown in Table~\ref{results_hrs}. For each of these user classes, we experimented with our proposed model using first-order predicates for the content-based filtering alone (denoted as CF in Table~\ref{results_hrs}), as well as the first-order predicates for both content-based filtering and collaborative filtering (denoted as HR in Table~\ref{results_hrs}). We also tried different settings of the parameters $\alpha$ and $\beta$ (denoted in the parentheses following the algorithm names in Table~\ref{results_hrs} ) for both scenarios in order to evaluate their effectiveness on reducing the false positive prediction.
\begin{sidewaystable*}	
	\caption{Results }		
	\small	
	\begin{tabular}{|c|c|c|c|c|c|c|c|}		
		\hline
		Job Title & Approach & \textbf{FPR} & FNR & Precision & Recall & Accuracy & AUC-ROC \\ \hline
		
		~ & Content-based Filtering (CF) & 0.537 & 0.321 & 0.060 & 0.679 & 0.473 & 0.628  \\ \cline{2-8}
		Retail & Cost-sensitive CF ($\alpha$0$\beta$2) & \textbf{0.040} & 0.868  & 0.143 & 0.132 & \textbf{0.921} & 0.649  \\ \cline{2-8}
		 Sales  & Hybrid Recommender (HR) & 0.516 & \textbf{0}  & 0.089 & \textbf{1.0} & 0.509 & \textbf{0.776}   \\ \cline{2-8}
		Consultant & Cost-sensitive HR ($\alpha$0$\beta$2) & 0.045  & 0.906 & 0.096 & 0.094 & 0.914 & 0.755  \\ \cline{2-8}
		~ & Cost-sensitive HR ($\alpha$0$\beta$1) & 0.113  & 0.623 & \textbf{0.144} & 0.377 & 0.863 & 0.772  \\ \hline \hline
		
		~ & Content-based Filtering (CF) & 0.220 & 0.184 & 0.053 & 0.816 & 0.781 & 0.861  \\ \cline{2-8}
		Case & Cost-sensitive CF ($\alpha$0$\beta$2) & 0.084 & 0.609 & 0.066 & 0.391 & 0.909 & 0.847  \\ \cline{2-8}
		 Manager & Hybrid Recommender (HR) & 0.239  & \textbf{0} & 0.059 & \textbf{1.0} & 0.765 & \textbf{0.911}  \\ \cline{2-8}
		~ & Cost-sensitive HR ($\alpha$0$\beta$2) & \textbf{0.037}  & 0.736 & \textbf{0.096} & 0.264 & \textbf{0.952} & \textbf{0.911}   \\ \hline \hline
		
		~ & Content-based Filtering (CF) & 0.427  & 0.195 & 0.045 & 0.805 & 0.579 & 0.746  \\ \cline{2-8}
		District & Cost-sensitive CF ($\alpha$0$\beta$2) & 0.017 & 0.920 & \textbf{0.104} & 0.080 & 0.961 & 0.745   \\ \cline{2-8}
		 Manager & Hybrid Recommender (HR) & 0.439 & \textbf{0} & 0.053 & \textbf{1.0} & 0.572 & 0.817  \\ \cline{2-8}
		~ & Cost-sensitive HR ($\alpha$0$\beta$2) & \textbf{0.013} & 0.977 & 0.042 & 0.023 & \textbf{0.964} & 0.812  \\ \cline{2-8}
		~ & Cost-sensitive HR ($\alpha$0$\beta$1) & 0.068 & 0.678 & \textbf{0.104} & 0.322 & 0.917 & \textbf{0.825}  \\ \hline
	\end{tabular}
	\label{results_hrs}
\end{sidewaystable*}
	
As the table shows, although the two approaches show similar performance on False Positive Rate, Precision, and Accuracy, the hybrid recommendation system improves a lot on the False Negative Rate, Recall and AUC-ROC compared with content-based filtering alone, especially on the Recall (reached 1.0 for all three of the user classes). So, question {\bf(Q1)} can be answered affirmatively. The hybrid recommendation system improves upon the performance of content-based filtering alone, by taking into consideration the information of similar users who have the same expertise or location as the target user. 

The first column of Table~\ref{results_hrs} shows the False Positive Rate which we want to reduce. As the numbers shown, the cost-sensitive approach greatly decreases the FPR compared with prior research which does not consider the domain preferences on the cost matrix. It also significantly improves the accuracy at the same time. Hence, question {\bf(Q2)} can also be answered affirmatively. Note that, although it seems that recall has been considerably sacrificed, our goal here is not to capture all the matching jobs for the target user, but instead to increase the confidence on the recommendations we are giving to our users. Since we may have hundreds of millions of candidate jobs in the data pool, we can usually guarantee that we will have a sufficient number of recommendations even with relatively low recall. Moreover, our proposed system can satisfy various requirements on the trade-off of precision and recall for different practical consideration by tuning the parameters $\alpha$ and $\beta$. If one does not want the recall too low, in order to guarantee the quantity of recommendations, one can simply decrease the value of $\beta$; if one does not want the precision too low, in order to improve the customer satisfaction, one can just increase the value of $\beta$. 

\begin{figure*}[htbp]
	\begin{minipage}[b]{\textwidth}
		\includegraphics[width=1\textwidth]{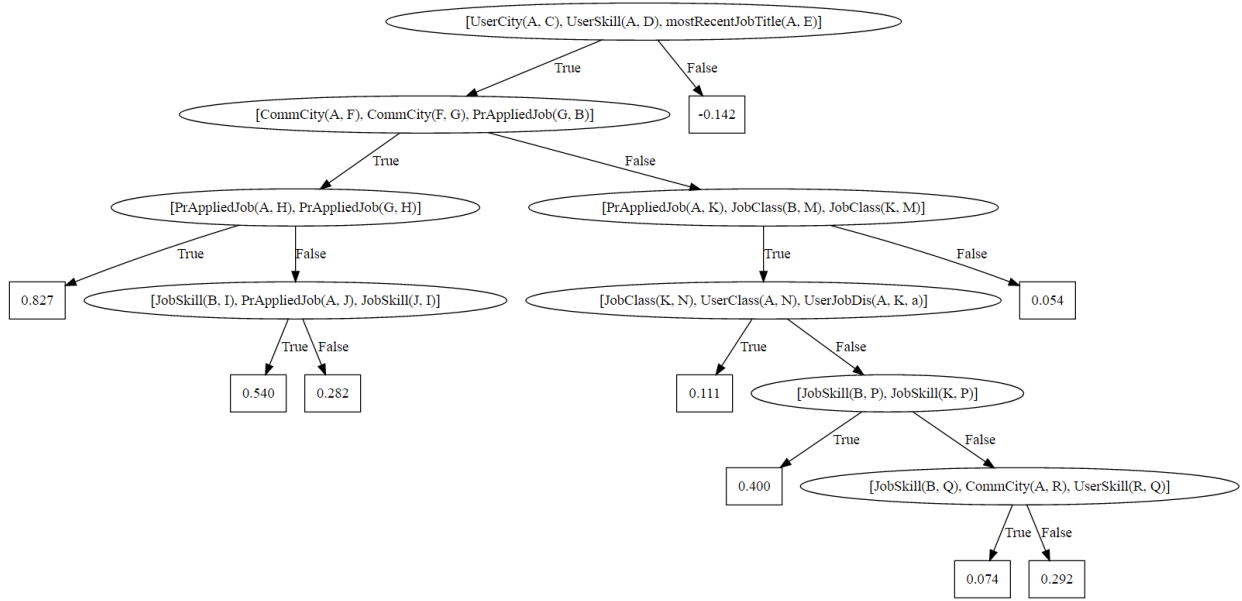}
	\end{minipage}
	\caption{ \small Sample Regression Tree}
	\label{c22tree}
\end{figure*}

Figure~\ref{c22tree} shows a sample regression tree learned from \textit{Case Manager} users. The left most path can be interpreted as: if we have the information on record about the $UserCity$, $UserSkill$ and $MostRecentJobTitle$ of the target user $A$ (in other words, the target user does not suffer from the cold start problem), and there is another user $G$ who lives in the same city and applied for the same job as user $A$ before, and if user $G$ applied for the target job $B$, then the probability for the target user $A$ to apply to the target job $B$ (in other words, the probability for the user--job pair to be matching) is $P( Match(A, B)= 1 ) = \frac{1}{1+e^{-0.827}} \approx 0.696 $. 

It is worth mentioning that we also tried to experiment with Markov Logic Networks on the same data with \textit{Alchemy 2} \cite{Alchemy09}. However, it failed after continuously running for three months due to the large scale of our data. This underscores one of the major contributions of this research in applying statistical relational learning to building a hybrid recommendation model in a real-world large-scale job recommendation domain.

\subsection{Conclusion}

We proposed to construct a hybrid job recommendation system with an efficient statistical relational learning approach which can also satisfy the unique cost requirements regarding precision and recall in this specific domain. The experiment results show the ability of our model to reduce the rate of inappropriate job recommendations. Our contribution includes: i. we are the first to apply statistical relational learning models to a real-world large-scale job recommendation system; ii. our proposed model has not only been proved to be the most efficient SRL learning approach, but also demonstrated its ability to further reduce false positive predictions; iii. the experiment results reveal a promising direction for future hybrid recommendation systems-- with proper utilization of first-order predicates, an SRL-model-based hybrid recommendation system can not only prevent the necessity for exhaustive feature engineering or pre-clustering, but can also provide a robust way to solve the cold-start problem.

\section{Application of Soft-Margin RFGB to A Rare Disease Case Study}
\label{raredis}
Besides its significant application values in recommendation systems proved in~\cite{YANG201737}, another important applicable domain for the proposed cost-sensitive RFGB (Chapter~\ref{chap2}) is the medical domain where the patients diagnosed as positive for certain diseases (positive examples) are much more than those who are not (negative examples). 
Besides, the practical requirements in medical domains are also consistent with putting less weights on the majority class which is negative in this case. For instance, when predicting heart attacks, it is imperative that any modeling approach reduces the number of false negatives (failed to predict cardiac event with adverse consequences and high cost), even if it comes at the expense of adding {\it a few more} false positives (predicted a cardiac event that did not happen, moderate cost). The low-cost examples, for this application, would be false positive examples.

We applied the Soft-Margin RFGB approach to a case study which aims to investigate the different experience patterns between people with rare diseases (minority class) and those with more common chronic illnesses (majority class). The experiment results showed superior performance of soft-margin RFGB  compared with sampling based approaches  when predicting the occurrence of rare diseases based on the self-reported behavioral data~\cite{CHASEhaley}.

\section{Application of Dynamic Bayesian Networks to A Cardiovascular Disease Study}
\label{DBNmain}

Cardiovascular disease (CVD) is the leading cause of death worldwide, accounting for more than 17.3 million deaths per year which is more than deaths from any other cause. In 2013, cardiovascular deaths represented 31 percent of all global deaths~\cite{Mozaffarian447}.
It is well known that successful and established lifestyle intervention and modification can result in prevention of the development of CVD, especially when the interventions occur in early life. 
Because of this reason, Lung and Blood Institute designed the Coronary Artery Risk Development in Young Adults (CARDIA) Study. This is a longitudinal study with 25 years of data (from 1985) collected from early adulthood of participants whose physical conditions have been monitored for the past 3 decades. 

In this project, we designed and implemented a series of experiments to model the development of Coronary Artery Calcification (CAC) level which has been proved to be the most predictive measure of subclinical Coronary Artery Disease (CAD)~\cite{detrano} by employing the CARDIA data. 
In order to discover the influence of risk factors on the development of CVD in later adult life, a temporal model that is easy to interpret is needed for training such longitudinal data. Since the revisits of the participants happened regularly and due to the highly uncertainty nature of medical data, a discrete-time temporal probabilistic model is required. Hence, we chose Dynamic Bayesian Networks (DBNs)~\cite{KollerFriedman09} and compared three standard optimization scoring metrics~\cite{KollerFriedman09} -- Bayesian Information Criterion (BIC), Bayesian Dirichlet metric (BDe) and mutual information (MI) --  for learning these temporal models. we also combined the different probabilistic networks resulting from these three metrics and evaluated their predictive ability through cross-validation.

Since the clinical measurements such as blood pressure, cholesterol level, blood glucose, etc. are usually more directly related to the CAC level, in order to better reveal the correlation between the lifestyle factors and the CVD development, only the basic socio-demographic and health behavior information were employed to prevent the more informative features interfering with the model training. In other words, only non-clinical data was given for training these DBNs to predict the incidence of CAC-level abnormality as the individual ages from early to middle adult life. 

The experiment results indicate that these behavioral features are reasonably predictive of the occurrence of increased CAC-levels. They allow the doctors to identify life-style and behavioral changes that can potentially minimize cardiovascular risks. In addition, the interpretative nature of the learned probabilistic model facilitates possible interactions between domain experts (physicians) and the model and provides them an easy way to modify/refine the learned model based on their experience/expertise.

In the following sections, a brief background on the data will be provided. Then the learning algorithm that has been adapted  will be outlined along with different scoring metrics which have been employed. Finally, the cross-validation results as well as interpretable learned models will be presented to demonstrate the effectiveness of the proposed approach.

\subsection{Proposed Approach}

One of the questions that the CARDIA study tries to address is to identify the risk factors in early life that have influence on the development of clinical CVD in later life. It is a longitudinal population study started in 1985-86 and performed in 4 study centers in the US (Birmingham, AL; Chicago, IL; Minneapolis, MN; and Oakland, CA) and includes 7 subsequent evaluations (years 2, 5, 7, 10, 15, 20, 25). It is approximately evenly distributed over race, gender, education level and beginning age. It includes various clinical and physical measurements and in-depth questionnaires about sociodemographic background, behavior, psychosocial issues, medical and family history, smoking, diet, exercise and drinking habits. 
In this project, only demographic and socio-economic features were taken into account when predicting the CAC-level as a binary prediction task. Specifically, we considered these features (with their number of categories noted in the following parentheses) -- participant's education level(9), full time(3)/part time(3) work, occupation(8), income(6), marriage status(8), number of children(3), alcohol usage(3), tobacco usage(3) and physical activities(6) during the year of the survey -- to model the development of CAC level (High($CAC>0$)/Low($CAC=0$)) in each observation of the CARDIA study. The variable names in order are \textit{03degre, 03wrkft, 03wrkpt, 03typem, 03dfpay, 03mrage, 03lvchl, 07drink, 09smknw, 18pstyr} as they are referred in the data-collection forms.\footnote{For detailed information about the features, please refer to CARDIA online resource at \url{ https://www.cardia.dopm.uab.edu/exam-materials2/data-collection-forms} } 
The chosen variables are extracted from data collected by sociodemographic questionnaires (Form 3) of the CARDIA study, as well as some features related to certain health behaviors collected by Alcohol Use Questionnaires (Form 7), Follow-up Questions for Tobacco Use Questionnaires (Form 9-TOB) and Physical Activity Questionnaires (Form 18). The target variable (cardiovascular disease) is believed to be closely related to the CAC total adjudicated agatston score which has been tested since year 15. A total number of 10 factors of all participants with 6 evaluations (at year 0, 5, 10, 15, 20, 25)  are extracted from CARDIA data as well as their CAC scores in the corresponding years. 

CARDIA data has some unique characteristics which require a preprocessing step or special modeling technique. 

The first issue with this study is that there are some missing values. Preprocessing is required for matching the different ways of recording  missing values across different study centers.
For example, for the missing values, some of them are just simply left as blank, some are denoted as ``M'' while the others as ``0'' or ``9''(as it is the option for ``No answer'' in the questionnaires). These different answers actually all present the same class in which case the participants are reluctant to answer the questions for some reason (given that they answered other questions on the questionnaire). Hence, ``M'', ``0'', `` '' and ``9'' should be clustered into one special class instead of four different classes as presented in the original data. Otherwise, the data would be unnecessarily scattered and potentially lead to inaccurate machine learning models. Hence, we treat these missing values as special cases instead of imputing them with the previous measurement or the most common measurement across the subjects. 
The other kind of missing values is due to the fact that some participants did not attend certain subsequent evaluations, i.e. the values of all the features are missing from that year.
While the participant retention rate is relatively high ($91\%$, $86\%$, $81\%$, $79\%$, $74\%$, $72\%$, and $72\%$ respectively for each evaluation), there are still at least 10 percent of the data are missing from the records. For this type of missing values, the missing values are filled in with the values from the participant's previous measurements.

Another issue is the evolution of evaluation measurements and the questionnaire designs over these 25 years of study. Some of the questions related to a certain aspect of the sociodemographic background may be divided to multiple questions or combined into one question in the follow up evaluations. For example, since year 15, the question related to the participant's marriage status  has added an option--``6. living with someone in a marriage-like relationship'' which used to be a separate question on the questionnaires of year 0 to year 10 where option `6' for the marriage status question indicated an answer different from this meaning. If the original data had been used without careful compilation, there would be considerable amount of noisy introduced into the data.

Besides, the data is composed of seven subsequent exams at different time points. Such temporal data needs dynamic models to capture the evolution of the risk factors over time.

Given these challenges, we employed a purely probabilistic formalism of dynamic Bayesian networks (DBNs) that extend Bayesian networks (BNs) to temporal setting. They employ a factorized representation that decreases the dimension from exponential in the total number of features to exponential in the sizes of parent sets. They also  handle the longitudinal data by using a BN fragment to represent the probabilistic transition between adjacent time slots, which allows both intra-time-slice and inter-time-slice arcs. Besides its ability of capturing the temporal transition probability of time-series data, DBNs allow cycles in the transition BN fragment to model feedback loops which are essential in medical domain. For instance, treatment of a disease in previous time slice can influence the disease severity in current time slice which in turn influences the treatment plan in the next time step. 

In most literature, the probabilistic influence relationships of the DBN (particularly the temporal influences) are pre-specified by domain experts and only the parameters are learned. However, since the aim here is to determine how the CAC-levels evolve as a function of socioeconomic and behavior factors, both the qualitative and quantitative aspects of the conditional dependencies need to be learned by adapting standard BN structure learning algorithms. There are two popular approaches for DBN structure learning: one is to exploit the greedy local search such as hill climbing based on certain decomposable local score functions; the other is to employ a stochastic global optimization framework. In this project, we implement the greedy search strategy and consider three different scoring functions -- the Bayesian Dirichlet (BDe) scores, Bayesian information criterion (BIC) and Mutual Information Test (MIT). 
Before discussing the scoring functions, the high-level overview of the framework is presented in Figure.\ref{flowchart} to facilitate a better understanding of the proposed approach.
\begin{figure}[htbp]
	\centering
	\includegraphics[scale=0.45]{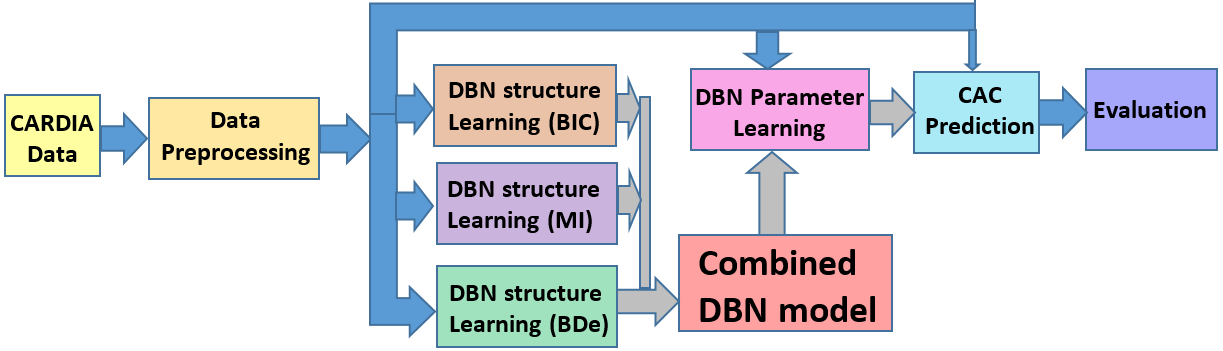}
	\caption{ Flow Chart of the Proposed Model. The blue arrows denote the flow of data and the grey arrows denote the flow of the model.}
	\label{flowchart}
\end{figure}

After preprocessing, the hill-climbing algorithm using three different scoring metrics was performed. 
First, the multi-series dynamic data was transformed into a training set by extracting every pair of sequential recorded measurements of every patient as a sample. After this step, a total of $5114$ ($|subjects|$) $*5$ ($|paired$ $time$ $slices|$) samples were obtained. Randomly drew a training set from these instances, then three models were learned by using three scoring metrics. These models were also combined in different ways by using the union of all the edges to construct a new unified model\footnote{When the combination induces intra-slice cycles, randomly remove one edge.}. The goal is to introduce more dependencies and investigate whether a more complex model is indeed more accurate. For these new unified models, the corresponding parameters were learned. The performance of all learned DBNs were then evaluated with 5-fold cross-validation.

Returning to the scoring function, the first row of Table.\ref{scoref} presents the general form of the decomposable penalized log-likelihood (DPLL)~\cite{LiuMY12} for a BN $\mathcal B$ given the data $\mathcal D$ where $\mathcal D_{il}$ is the instantiation of $X_i$ in data point $D_l$, and $PA_{il}$ is the instantiation of $X_i$'s parent nodes in $D_l$. So the general form of DPLL is the sum of individual variable scores which equals to the loglikelihood of the data given the local structure minus a penalty term for the local structure. Note that BIC and BDe mainly differ in the penalty term. 
The penalty term for BIC is presented in the second row where $q_i$ is the number of possible values of $PA_i$, $r_i$ is the number of possible values for $X_i$ and $N$ is number of samples ($5114\times 5$). Hence the BIC penalty is linear in the number of independent parameters and logarithmic in the number of instances. BDe penalty is presented in the third row where $\mathcal D_{ijk}$ is the number of times $X_i = k $ and $PA_i = j$ in $ \mathcal D$, and $\alpha_{ij}=\sum_k \alpha_{ijk}$ with $\alpha_{ijk} = \frac{\alpha}{q_ir_i}$ in order to assign equal scores to different Bayesian network structures that encode the same independence assumptions. Ignoring the details, the key is that the complexity of BIC score is independent of the data distribution, and only depend on the arity of random variables and the arcs among them while the BDe score is dependent on the data and controlled by the hyperparameters $\alpha_{ijk}$. \cite{KollerFriedman09} covers more details.
Instead of calculating the log-likelihood, MIT score (Table.\ref{scoref} last row) uses mutual information to evaluate the goodness-of-fit~\cite{VinhCCW11}. 
$I(X_i, PA_i)$ is the mutual information between $X_i$ and its parents. $\chi_{\alpha, l_{i\sigma_i(j)}}$ is chi-square distribution at significance level $1-\alpha$~\cite{VinhCCW11}. 

%
\begin{table}
	\caption{ The different scoring functions.}
	\begin{tabular}{|c|c|}
		\hline
		\multirow{3}{*}{$DPLL(\mathcal B, \mathcal D)$} & $  DPLL(\mathcal B, \mathcal D)=\sum_i^n [\sum_l^N \log P(\mathcal D_{il} |PA_{il})- Penalty(X_i, \mathcal B, \mathcal D)]$\\ \cline{2-2}
		~ & $Penalty_{BIC}(X_i, \mathcal B, \mathcal D) = \frac{q_i(r_i-1)}{2} \log N$ \\ \cline{2-2}
		~ & $Penalty_{BDe}(X_i, \mathcal B, \mathcal D) = \sum_j^{q_i} \sum_k^{r_i} \log \frac{P(\mathcal D_{ijk}| \mathcal D_{ij})}{P(\mathcal D_{ijk}| \mathcal D_{ij}, \alpha_{ij})}$ \\ \hline
		$DPMI(\mathcal B, \mathcal D)$ &  $S_{MIT}(\mathcal B, \mathcal D) = \sum_{i, PA_i\neq \emptyset} 2N*I(X_i, PA_i)-\sum_{i, PA_i\neq \emptyset}\sum_j^{q_i}\chi_{\alpha, l_{i\sigma_i(j)}}$\\ \hline
	\end{tabular}
	\label{scoref}
\end{table}


\subsection{Experimental Results of DBN on CARDIA data}

For the DBN DPLL-structure learning, we extended the BDAGL package of Murphy et al.~\cite{Eaton07_uai} to allow learning from multi-series dynamic data and to support learning with BIC score function. We also adapted DPMI-structure learning by exploiting the GlobalMIT package which was developed to model multi-series data from gene expression~\cite{VinhCCW11}. The learned DBN structure is shown in Figure~\ref{lDBN}. 
\begin{figure}[htbp]
	\centering
	\includegraphics[scale=0.6]{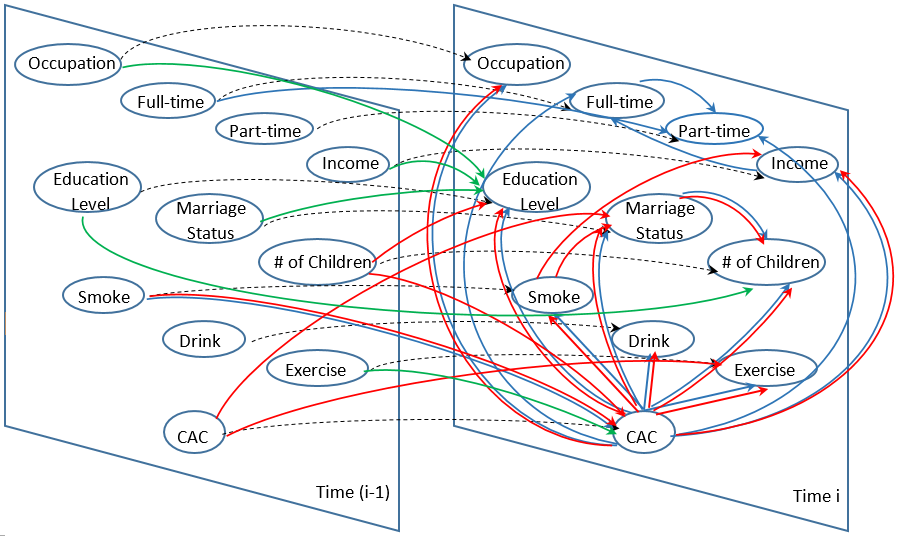}
	\caption{ Combined DBN model. The blue arcs are learned by DPLL-BDe; reds by DPLL-BIC; greens by DPMI; black dash lines are self-links which are captured by all three score metrics.}
	\label{lDBN}
\end{figure}

Note that all three score metrics learned the self-link for every variable. This shows that many socio-demographic factors are influenced by previous behavior (5 years backward). 
Observe that both BIC and BDe learned the inter-slice dependency between ``Smoke'' and ``CAC'' level. Both BIC and MI returned a temporal correlation between ``Exercise'' and ``CAC'', which indicates that previous health behaviors have strong influence on the risk of CAC in current time.

Then we combined the structure from the different learning approaches into a comprehensive model for learning parameters. We applied this DBN to test data and predict the CAC score based on the variables at the current and previous time steps. As CAC score from previous time-steps is highly predictive of future values, we hid the CAC-scores in the test set for fair evaluation.

We calculated the accuracy, AUC-ROC as well as F measure\footnote{Accuracy = (TP+TN)/(P+N); F = 2TP/(2TP+FP+FN)
} to evaluate the independent and mixed models learned by different score functions. The results are shown in Table~\ref{results_dbn}. As the table shows, the model learned by BDe score has the best performance while the MI the worst. 
We also performed t-test on the five folds results, which shows the BDe is significantly better than MI with P-values at 0.0014 (AUC-ROC), 0.0247 (Accuracy) and $5.7784\times 10^{-6}$ (F). In order to rule out the possibility that the better performance of BDe is resulted from the non-temporal information which MI does not have, we also experimented on BNs with intra-slice arcs only as well as DBNs with inter-slice arcs only. And the results showed that the difference between them is not statistically significant at 5\% significance level. In other words, the temporal information is as important as the non-temporal information in predicting the CAC level. 
Compared to BDe alone, the combined model has deteriorated performance. This is probably because the combined model has more arcs which exponentially increases the parameter space and the limited amount of training data cannot guarantee the accurate training for such high dimension model (possibly overfitting). 

\begin{table}
	\centering
	\caption{ Model Evaluation Results.}
	\begin{tabular}{|c|c|c|c|}
		\hline
		~  & Accuracy &  AUC-ROC & F measure	\\ \hline
		 MI & 0.5774 &	0.4870 & 0  \\ \hline
		 BIC & 0.6558  & 0.6979 & 0.5417  \\ \hline
		 BDe & \textbf{0.6805}  & \textbf{0.7139} & \textbf{0.6144}  \\ \hline
		 MI+BIC & 0.6482 & 0.6473 & 0.4640 \\ \hline
		 MI+BDe & 0.6715 & 0.6809 & 0.5632 \\ \hline
		 BIC+BDe & 0.6701 &	0.7092 & 0.5880 \\ \hline
		 MI+BIC+BDe & 0.6600 & 0.6581 & 0.5244	 \\ \hline
		 Inter-s Arcs & 0.5774 & 0.6013 & 0.0005 \\ \hline
		 Intra-s Arcs & 0.5853 & 0.6489 & 0.0501 \\ \hline
	\end{tabular}
	\label{results_dbn}
\end{table}

While our work generatively models all the variables across different years, extending this work to predict CAC-levels discriminatively remains an interesting direction. 
Considering more risk factors and more sophisticated algorithms which allow learning temporal influence jumping through multiple time slices are other directions. The end goal is the development of interventions for young adults that can reduce the risk of CVDs at their later ages.  

\section{Application of Exponential Family Models to Predicting Cardiovascular Events}
In this section, we propose to model a specific medical condition Coronary Artery Disease (CAD) which is the leading cause of death for people in the United States and killing more than 370,000 people annually~\cite{death}. A serious condition a lot of CAD patients suffer from is the narrowed or blocked artery caused by plaque. One of the common procedures to restore blood flow through the artery is {\em Angioplasty}. A considerable amount of patients treated with angioplasty often need repeat revascularization procedures due to the restenosis~\cite{Alfonso14}. 
Given its medical significance, there are some research investigating the predictive factors of restenosis~\cite{Gaudry16,Kastrati97,KUNTZ199315,Iida2014792}.
Kuntz et al.~\citeyear{KUNTZ199315} researched on the association between demographic or angiographic variables and continuous measures of restenosis after stenting or angioplasty with a continuous regression model. 
Iida et al.~\citeyear{Iida2014792} applied the Cox proportional hazards model to investigate the factors associated with restenosis after endovascular therapy, which includes demographic, medication and other disease characteristics of the patients. 

However, we do not only aim to investigate the association between the restenosis and predictive factors, but also to predict the count of restenosis. This allows us to better estimate a new patient's condition by analyzing his/her historical health record data. 
Since the number of angioplasties a patient needs is usually associated with the severity and complexity of his/her CAD condition, modeling and predicting the count of angioplasties can provide better understanding on the future CAD condition of patients and hence facilitate better treatment plans in advance. 
Another contribution of our work lies in the fact that there is no other significant research which focuses on the factors over a substantial period of time to discover their effect on the number of angioplasties.

Faithfully modeling such count data could potentially provide a better way to exploit the large scale of health records data in order to reveal the development patterns of certain disease or facilitate the discoveries of the significant influential factors on a specific medical condition. 
Most of the previous work that employed machine learning algorithms modeled such count variables as multinomial distributed~\cite{KollerFriedman09} or assumed the Gaussian assumption over these count variables. But they all have obvious limitations as introduced in Chapter~\ref{chap4}. Hence, Poisson models have been increasingly employed in the context of the count data~\cite{Ma2008, Bolker2009a, Allen13}.
However, little research has been performed on comparing these two different probability distribution assumptions in real-world medical domains. 

In this project, we target the number of {\em Angioplasty} as it is a significant indicator for the severity of patients' coronary disease conditions and aim to examine various probabilistic graphical models on predicting the count of {\em Angioplasty} by learning from the historical medical conditions of the CAD patients.

To model such count variables, we employed probabilistic graphical models with two different assumptions on the probability distributions of the target variable.
One is multinomial distribution assumption which assumes that each sample is independently extracted from an identical categorical distribution, and the numbers of samples falling into each category follow a {\em multinomial distribution}. 
Another assumption is the {\em Poisson distribution} which states that each sample is a instance of a count variable following Poisson distribution.

It must be mentioned that such assumptions on the data probability distributions cannot be examined by simply employing certain statistical analysis methods, such as data transformations, on the interested variable alone. This is because most of the time we also aim to know how other variables in the domain influence the target variable, and which assumption more faithfully fits the conditional probability distributions of the target variable given the variables on which it depends. 

Consequently, we propose a set of experiments to model the dependencies among variables and at the same time examine these probability distribution assumptions by employing probabilistic graphical models with two different probability distributions from the exponential family. 
Specifically, we consider the problem of predicting the number of Angioplasty procedures from temporal EHR data and consider different learning methods for these two different distributions. 

We make the following key contributions: first, we consider multiple base learners in the context of gradient boosted multinomial and Poisson models for the task of predicting the number of Angioplasty procedures; second, we consider different types of gradient updates for learning Poisson models; third, we consider different types of modeling for the predictive features when learning these models. Finally, we perform comprehensive analyses on the real EHR data and demonstrate the usefulness of exponential family distributions for this interesting task.

\subsection{Modeling EHRs}
\label{hrdnehr}

We use the electronic health record extracted from {\em Regenstrief} Institute. The data was extracted from  $5991$ patients for a total of $4928799$  health records spanning from $1973$ to $2015$.  The target variable $Angioplasty$ was assigned the value as the number of Angioplasty through the trajectory of the patient's medical record. Predictive factors of the target variable include medical measurements,  procedures, medications as well as certain health behaviors before the first Angioplasty. The value space of the variables is mix boolean-numeric domain. Boolean variables include Procedures for EKG, Beta-Blocker Medications, Lipid Lowering Medication, Tricyclic Anti-Depressant, Calcium Channel Blocker Medication, DM Medication, HTN Medication, Post acute coronary syndrome, Presence of atrial fib, History of alcoholism and Smoke status while numeric variables include diastolic blood pressure, systolic blood pressure, High-density lipoproteins (HDL), Low-density lipoprotein (LDL), A1C test, Triglyceride and body mess index. 

Figure~\ref{trj} shows a sample EHR trajectory of one patient with all records that are related to him/her from the EHR. Each node along the time-line represents a medical event happened at that time point, e.g. having a \textit{Procedures for EKG} at 3/16/1995, taking HTN Medication at 6/19/2001 or being tested for A1C at 10/31/2009. The red nodes indicate the occurrences of $Angioplasty$.
\begin{figure*}[htbp!]
	\begin{minipage}[b]{1.0\textwidth}
		\centering
		\includegraphics[scale=0.35]{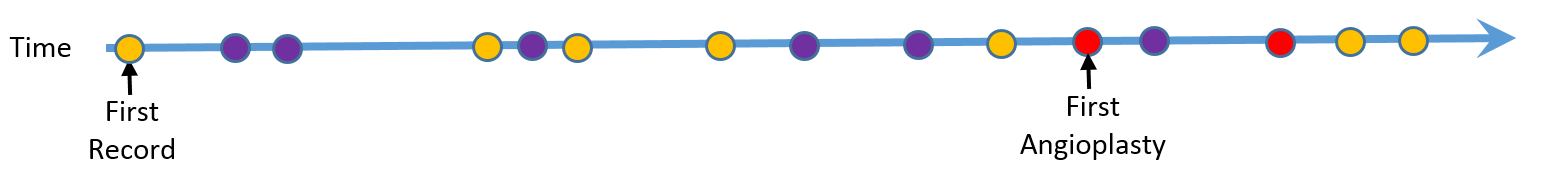}
		\caption{Sample Trajectory of One EHR. \\The red circles indicate the events of $Angioplasty$; purple circles indicate the events of medical procedures with binary values; orange circles indicate the events of medical measurements with continuous values. Events before the first $Angioplasty$ would be aggregated with certain aggregators according to their value types.  }
		\label{trj}
	\end{minipage}
\end{figure*}

The time point when the first $Angioplasty$ procedure was performed is used to determine the end time point of the segment over which the other sequential events are extracted and aggregated. For the boolean valued variables (usually the usage of certain medications or procedures), we employed two different aggregation functions: i). \textit{Indicator} which is $1$ if a certain event happens at least once within the segment (i.e. before the first $Angioplasty$) and $0$, otherwise; ii). \textit{Count} which is the number of occurrences for that event along the chosen segment of the patient's EHR trajectory. 
For the case of numeric variables, we considered three different ways for aggregating the multiple observations before the first Angioplasty - i). \textit{Min} which is the min value of all the recorded values through the segment for a certain medical measurement; ii). \textit{Max} which is the max value of all the measured values within the segment; iii). \textit{Mean} which is the mean value of those values. We also considered the final measurement before the first Angioplasty as the representative measure of the observations till that time point.
If a patient never had $Angioplasty$, the target variable has the value of $0$ and the values of other variables are calculated by aggregating their corresponding values over the entire trajectory instead of a segment.  

We aim to answer the following questions explicitly:
\begin{itemize}
	\item{Q1:} Which of the two exponential family distributions faithfully models the number of Angioplasty procedures?
	\item{Q2:} Does the choice of base learner or the update type impact the modeling ability of the learned distributions? 
	\item{Q3:} Does the choice of aggregation function matter when predicting the number of procedures?
\end{itemize}

\subsection{Preliminary Results}

As mentioned earlier, we considered both Poisson and Multinomial distributions with gradient boosting as the learning algorithm. Inside the gradient boosting, we considered two types of base learners for Multinomials - Trees and Neural networks~\cite{weka} denoted as  REPTree-B and NN-B respectively in our results. For the Poisson assumption,  we employ discriminative Poisson model whose structure is similar as a reversed Naive Bayes but the conditional probability distribution of the target variable follows Poisson distribution and two different update manners are considered - multiplicative boosting~\cite{pdn15} (denoted as DPM-MGB) and additive boosting(denoted as DPM-AGB). 

To evaluate these questions, we calculated the Mean Square Error (MSE) 
\begin{equation*}
MSE = \frac{\sum_{i=1}^N [1- p(y_i = \hat{y}_i; X_i)]^2}{N},
\end{equation*}
where $\hat{y}_i$ is the true value of the target variable of the $i^{th}$ patient and $N$ is the number of patients.
We also calculated the mean Log-likelihood (LL) for DPMs, since the other classifiers often predict the probability for $Angioplasty$ being high values as 0, hence the LL score cannot be reported for them. 
We performed 5-fold cross-validation using the data from $5991$ patients and aggregated their results.

\begin{table*}[htbp!]
	\centering
	\caption{MSE of Predictions from Boosting Different Classifiers}
	\begin{minipage}{\textwidth}
		\centering
		\small
		\begin{tabular}{|c|c|c|c|c|c|}
			\hline
			Numeric & Boolean & REPTree-B & NN-B  & DPM-AGB & DPM-MGB \\ \hline
			\multirow{2}{*}{Min} & Count &  0.071 & 0.705 & 0.072 & \textbf{0.066} \\ \cline{2-6}
			~ & Indicator &  0.064 & 0.819 & 0.067 &  \textbf{ 0.058 } \\ \hline
			\multirow{2}{*}{Max} & Count & 0.066 & 0.614 & 0.069 & \textbf{0.063}    \\ \cline{2-6}
			~ & Indicator &  0.063 & 0.994 & 0.066 &  \textbf{0.059} \\ \hline
			\multirow{2}{*}{Mean} & Count & 0.064 & 0.255 &  0.068 & \textbf{0.061} \\ \cline{2-6}
			~ & Indicator & 0.063  & 0.513 & 0.068 &  \textbf{0.061} \\ \hline
			Latest & Latest &  0.065 & 0.438 & 0.068 & \textbf{0.061}\\  \hline			
		\end{tabular}		
		\label{mse}
	\end{minipage}
\end{table*}

The first observation on Table~\ref{mse} is that the Poisson model with multiplicative updates tends to have lower MSE than their multinomial counterparts. Its superior performance is statistically significant  at the $3\%$ significance level except for the case of Indicator Mean when the \textit{p}-value equals 0.1. DPM-AGB has slightly higher MSE than multinomial boosting with regression trees (RepTree-B). However, the differences between their performance are NOT statistically significant when applying the aggregators: Count Min ($\textit{p}-value = 0.588$), Indicator Min ($\textit{p}-value = 0.088$) and Latest ($\textit{p}-value = 0.077$).  We hypothesize that this is due to the sensitiveness of the step-size parameter of the gradient steps in additive Poisson models. The multiplicative models appear more robust to the step size (an observation made by Hadiji et al.~\cite{pdn15}). So we answer Q1 cautiously in that Poisson models appear more suited for this task if they are robust.

With respect to the base learners employed, it appears that the neural networks do not perform as well as trees when boosted - possibly due to overfitting. The trees serve as better weak learners for the boosted model. For the gradient updates, the multiplicative models appear to exhibit better performance than additive models which are sensitive to their step-sizes. Thus Q2 can be answered in favor of trees and multiplicative updates.

Finally, for the comparison of different aggregators, we considered the model with the best performance (i.e. DPM-MGB). The differences between the performances of Indicator and Count are NOT statistically significant at $3\%$ significance level, which indicates that the Indicator function is as useful as Count for boolean variables (in that knowing whether an event happened appears to be as informative as the number of times the event happened). 
Same results are observed with numeric variables as well. Thus answering Q3, there does not seem to be too much difference which is statistically significant in the choice of aggregation functions .

It is intriguing to check if the aggregators can result in vastly different models. To this effect, we analyze the first tree learned by each Poisson model. The first tree is a good indicator of the most important set of features that are employed in the prediction task. These trees are presented in Figure~\ref{prt}. As can be seen clearly, all the settings use the same feature in the root, which is to check if the patient had post acute coronary syndrome. Alcohol consumption appears on the left subtree in all the aggregators and the Triglyceride appears on the right. It is also instructive to note that the regression values at the leaves are similar as well. This demonstrates that the aggregators do not necessarily influence the final learned model.

In summary, our results are encouraging in that extensive feature engineering is not necessary when employing a reasonable probability distribution for modeling Angioplasties. Our results show that modeling the count distributions using Poisson regression model allows us to faithfully capture the number of procedures (as the high log-likelihood numbers show in Table~\ref{ll}).
\begin{table}[htbp!]
	\caption{Log-Likelihood of Results from DPM-MGB  }
	\centering
	\begin{tabular}{|c|c|c|c|c|}
		\hline
		LL & Min & Max & Mean & Latest \\ \hline
		Count & -0.486 & -0.37 & -0.384 & \backslashbox{}{} \\ \hline
		Indicator & -0.412 & -0.399 & -0.414 & \backslashbox{}{}\\ \hline
		Latest & \backslashbox{}{} & \backslashbox{}{} & \backslashbox{}{} & -0.409\\ \hline
		
	\end{tabular}		
	\label{ll}
\end{table}

\begin{figure}[htbp!]
	\begin{minipage}[b]{0.5\textwidth}
		\centering
		\includegraphics[scale=0.19]{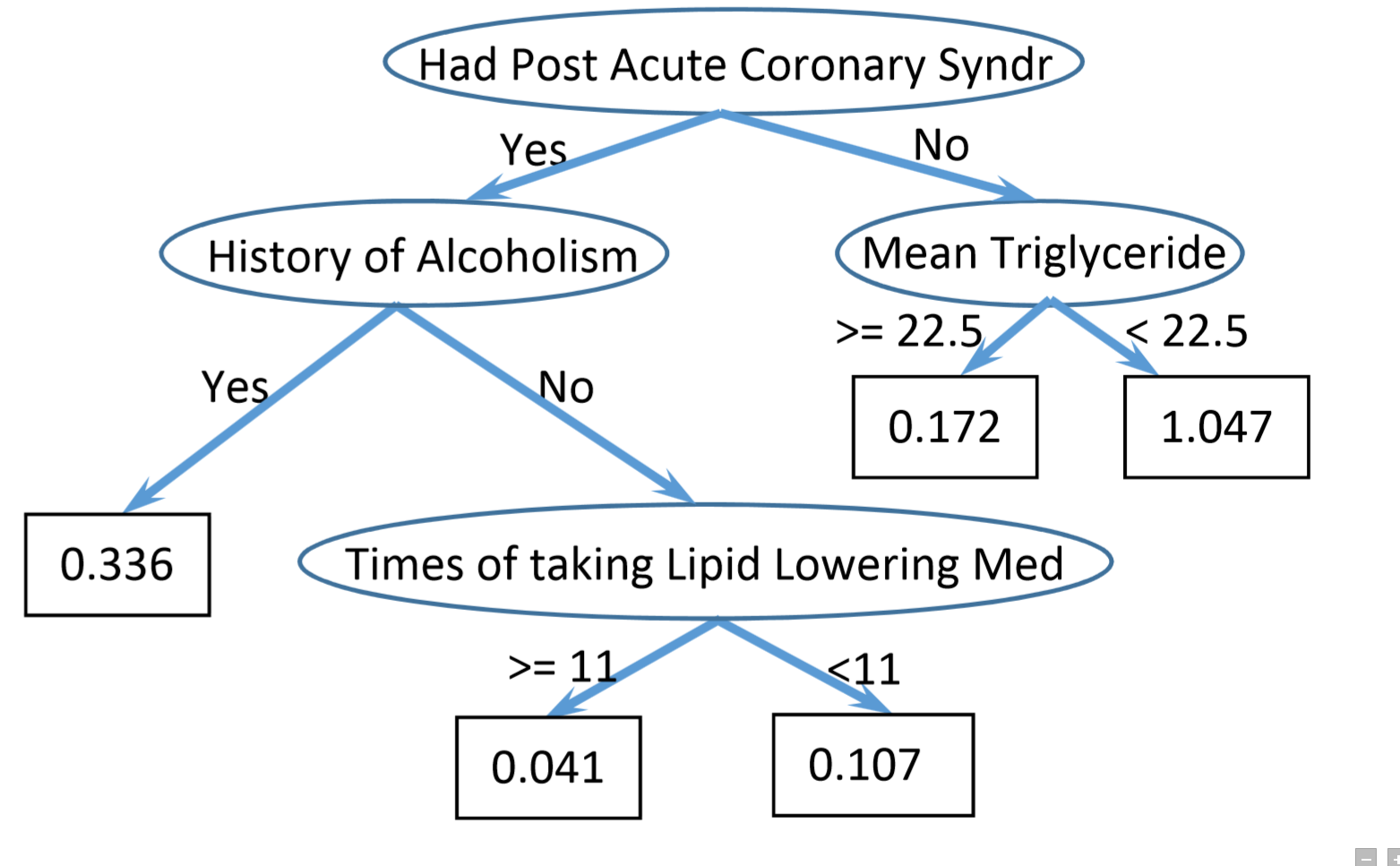}
		\centering (a) 
	\end{minipage}	
	\begin{minipage}[b]{0.5\textwidth}
		\centering
		\includegraphics[scale=0.19]{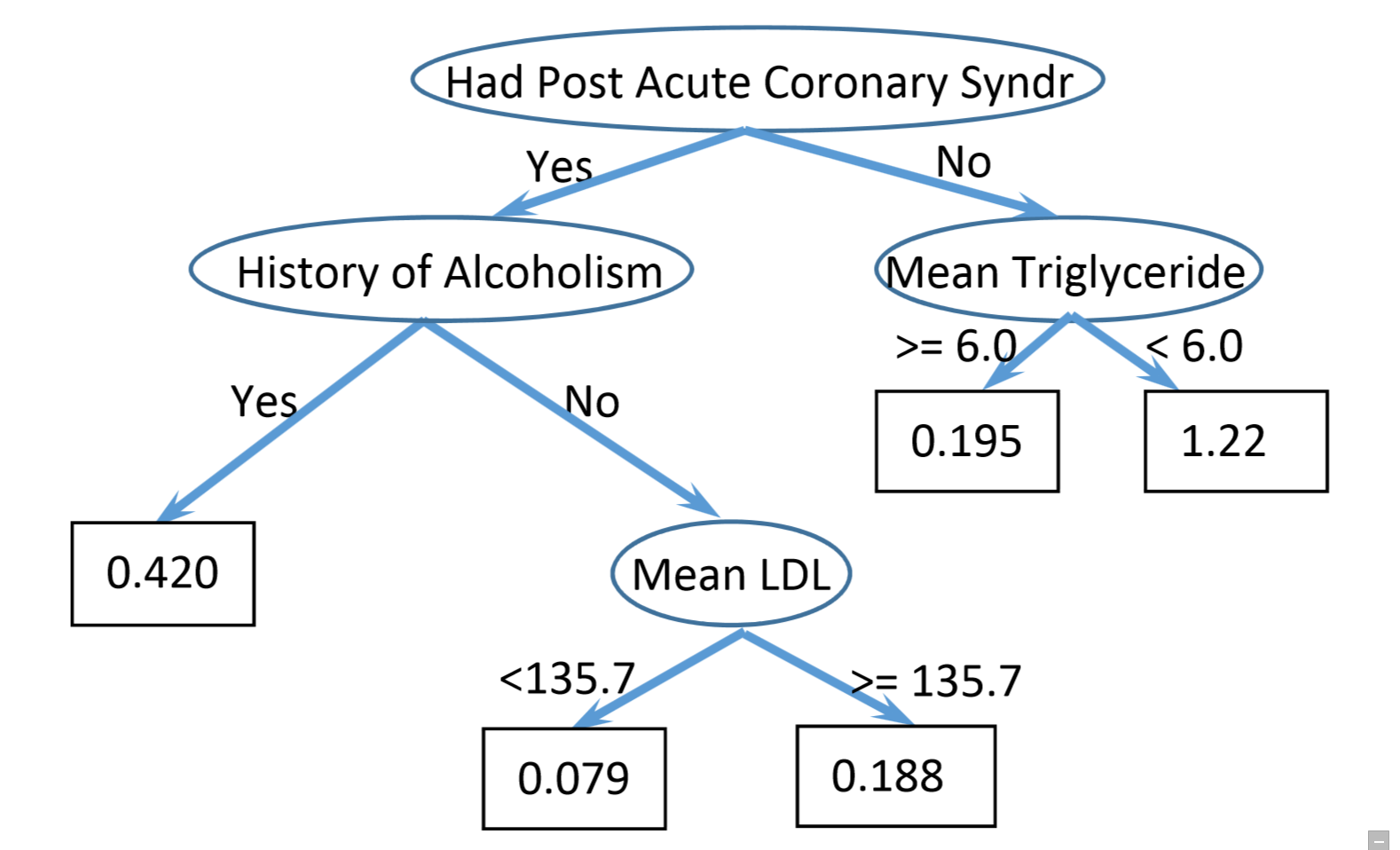}
		\centering (b) 
	\end{minipage}
	\begin{minipage}[b]{0.5\textwidth}
		\centering
		\includegraphics[scale=0.2]{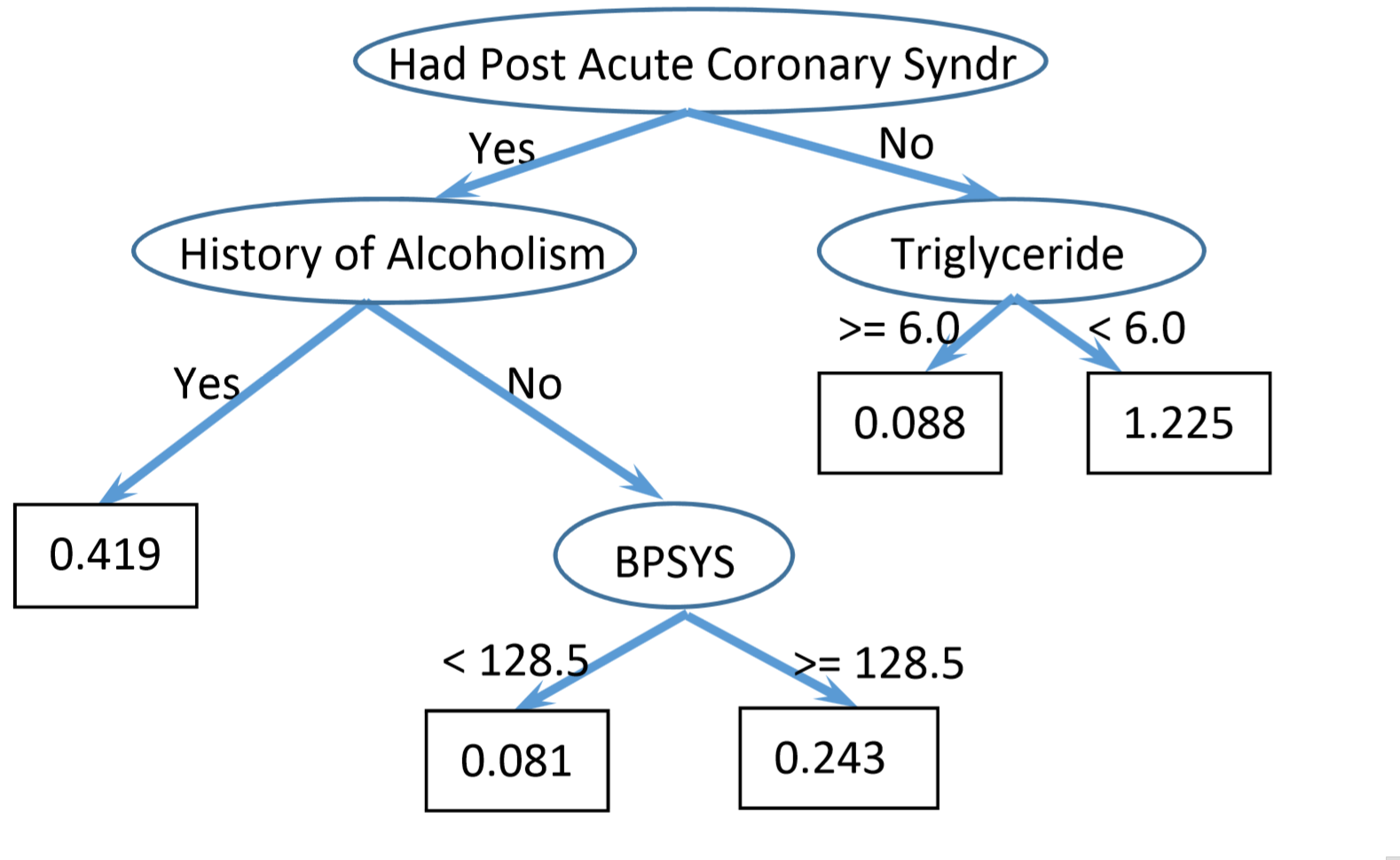}
		\centering (c) 
	\end{minipage}			
	\caption{ Sample Learned Poisson Regression Tree for each of the aggregator. (a) is the count mean, (b) is the indicator mean and the (c) is the latest.  It is worth noting that in all the three cases, the top levels appear to be very similar demonstrating that the learned mode is indeed robust for all settings.å}
	\label{prt}
\end{figure}

\chapter{Future Work}
\label{chap6}

\section{Predicting Cardiovascular Events with Hybrid Relational Dependency Networks}

The previous work~\cite{YangBIBM17} has shown the significance of probabilistic distribution assumptions being consistent with true distributions of the target variable. Actually, in order to prevent potential information loss, it is essential to keep the original value forms of all variables in the domain, i.e. target variables and predictor variables as well. The great potential of hybrid relational dependency networks proposed in Chapter~\ref{chap4} is ready to be explored by applying it to the real medical study data as well as electronic health record data.  

\subsection{Modeling CARDIA Data}

As introduced in Chapter~\ref{DBNmain}, CARDIA Study was designed to investigate how the lifestyle factors and physical conditions in early life influence the development of clinical CVD in later life. The previous work~\cite{YangKTCN15} employed Dynamic Bayesian networks to model the conditional dependences of the CAC level given the socioeconomic factors and health behavior within 5 years. In that work, the variables are either originally or discretized into discrete values and modeled as multinomial distributed. However, there are several open questions unanswered which could potentially be solved with the proposed hybrid statistical relational learning approach: 1). the correlations between the CVD development and risk factors in early life (say 20 years earlier) has not yet been investigated; 2). will the prediction performance be improved if the underlying assumptions on the variables' probabilistic distributions are closer to the original value types? 3). As mentioned in Chapter~\ref{DBNmain}, CARDIA is a longitudinal population study performed in 4 study centers located in Birmingham, AL; Chicago, IL; Minneapolis, MN; and Oakland, CA. If there are certain hidden factors which make the i.i.d.assumption invalid for individuals from the same location, it probably would provide a better way to inference the missing values for some participants. In other words, for predicting the CVD risk of participants with missing sub-tests, will the statistical relational learning approach which can make use of the attribute values of related objects (in this case, the participants from the same location) beat the standard propositional learning approach?     

Motivated by the questions presented above, the possible directions for future work include: 1). employ the hybrid statistical relational learning approach proposed in Chapter~\ref{hplm} and model the dependency of CAC level at year 25 on the risk factors of year 0, 2, 5, 7, 10, 15 and 20 respectively. By comparing the prediction performance of these 7 hybrid relational models, the significance of the risk factors at early age would be revealed; 2). On top of that, define predicate $SameCenter(X, Y)$, which can introduce related objects into the searching space. By comparing its performance with standard propositional models, we would be able to investigate the potential to inference the condition of certain diseases with information on locationally related individuals.  

\subsection{Modeling EHR Data}

The current work could be extended in two directions: 1). employ the proposed hybrid statistical relational model in Chapter~\ref{hplm} which allows continuous valued variables, so the medical measurements, such as blood pressure, cholesterol level can maintain the original form and information instead of being discretized into categorical values; 2). include the related objects into the training set, in this case, they are the patients in the control group of CVD patients.

The goal of the future work is to investigate whether learning with continuous values performs better than discretizing them into categorical values and if hybrid statistical relational learning with information on the related objects can facilitate a more accurate prediction model for this medical task. Besides, investigating the previous questions Q1 to Q3 in the context of hybrid statistical relational learning could also potentially reveal some interesting patterns for medical research on cardiovascular disease.  

\section{Predicting Cardiovascular Events with Relational Continuous-Time Bayesian Networks}

Faithfully modeling longitudinal medical data could potentially provide a better way to exploit the large scale health records data in order to reveal the development patterns of certain disease or facilitate the discoveries of the significant influential factors on a specific medical condition.
The prior work on applying dynamic Bayesian networks to CARDIA data has shown promising results of this research direction~\cite{YangKTCN15}. Given the effectiveness of proposed relational continuous-time Bayesian networks~\cite{YangKKN16} on employing the information from related objects, it would be interesting to see if such relational temporal modeling approach would result in better prediction  when applied to the real-world data.

\subsection{Modeling CARDIA Data}

As mentioned in the previous section~\ref{DBNmain}, CARDIA study is a longitudinal population study including a sequence of 8 evaluations at year 1985, 1987, 1990, 1992, 1995, 2000, 2005, 2010. Since dynamic Bayesian networks require a constant time span between adjacent time slices, the previous experiments omitted the records from year 1987 and year 1992 to obtain a constant sampling rate which is once in every five years. The relational continuous-time Bayesian networks can be applied to CARDIA data by: 1). converting the evaluation records from crosswise at each observation point to trajectories of variables where only transitions of the variables' states are represented as events along the trajectories; 2). considering the information from locationally related participants when training the temporal models. The goal of this research would be to investigate the potential of relational time-line analysis in predicting cardiovascular disease based on the historical physical conditions and health behavior patterns. 

\subsection{Modeling EHR Data}

Another potential application topic would be to predict the CVD event {\em Angioplasty} by training a relational continuous time model from the historical medical informations in EHRs. This proposed work differs from the previous work~\cite{NandiniBIBM17, YangBIBM17} in that: 1). the value of the target variable is not aggregated over the trajectories of patients' EHRs but a boolean value with a time stamp attached to it indicating when the event happened; 2). the values of the predictive variables are not aggregated to a single value but preserved their original value and time information. The goal of this proposed direction would be to investigate the possibility of predicting the CVD event by employing the informations on the historical health conditions of the target patient as well as his/her counterparts in the control group.  

\def\baselinestretch{1}
\chapter{Conclusions}
\ifpdf
    \graphicspath{{Conclusions/ConclusionsFigs/PNG/}{Conclusions/ConclusionsFigs/PDF/}{Conclusions/ConclusionsFigs/}}
\else
    \graphicspath{{Conclusions/ConclusionsFigs/EPS/}{Conclusions/ConclusionsFigs/}}
\fi

\def\baselinestretch{1.66}

Statistical relational learning combines first-order logic which can represent the rich relations existing among the objects in EHRs with probability theory which can capture the uncertainty in medical domains. Temporal modeling provides an effective way to represent and learn from sequence data which is a common format in health related research. These two techniques have been used to address different challenging tasks in medical domains and proved their promising effectiveness respectively~\cite{WeissNP12,Natarajan2013,WeissNPMP12,YangKTCN15}. However, there are still numerous challenging problems demanding advanced statistical relational learning approaches, such as learning in class-imbalanced relational domains, learning from  continuous-time structured data, learning in hybrid domains, etc. In this thesis, the advanced probabilistic logic models for time-line analysis, for cost-sensitive learning and for hybrid structured data mining are proposed. Their theoretical foundations as well as convergence properties have been presented and proved. The remarkable prediction improvement of the cost-sensitive RFGB and RCTBN are shown in various standard relational data sets. 

To evaluate the practical significance of the proposed approaches, two possible directions for the future work are proposed. A series of experiments are designed to employ real-world data from a cardiovascular study as well as EHR data to evaluate the performances of proposed models, i.e. relational continuous-time Bayesian networks and hybrid relational functional gradient boosting.  

This thesis has shown the great potential of advanced statistical relational learning approaches in resolving the challenges faced by exploiting the longitudinal clinical data which is usually high dimensional, rich in relations, noisy, class-imbalanced, irregular sampled and heterogeneous in values and formats.  




\backmatter 
\appendix
\chapter{Appendix A}
\section*{Proof of the Modification on Weigting AUC-ROC}
\label{appendix:a}
\newtheorem*{mydef}{Theorem}
\begin{mydef}
	Equally dividing the area in ROC space and assigning every region the weights from Equation~\ref{weight}, the weighted AUC-ROC would have the following characteristics:\\
	1. The range of the weighted AUC-ROC is [0, 1];\\
	2. When $\gamma$ is 0, the resulting weighted-AUC is equal to the conventional AUC;\\
	3. when $\gamma$ is 1, only the area at the top is considered.
\end{mydef}
\begin{proof}
	Expanding on the weight formula in eq. ~\ref{weight},\\
	$w_0=1-\gamma$, \\
	$w_1=(1-\gamma)\times \gamma +(1-\gamma)$, \\
	... \\
	$w_{n-1}=(1-\gamma)\times \gamma^{n-1} +(1-\gamma)\times \gamma^{n-2}+...+(1-\gamma)$\\
	As can be seen, $w_{n-1}$ is a geometric progression with first term $a_1=1-\gamma$ and common ratio $r=\gamma$, so 
	$w_{n-1}=\frac{(1-\gamma)\cdot (1-\gamma^n)}{1-\gamma}=1-\gamma^n$.
	Assume, there are $N+1$ partitions in the ROC space, since the area is divided equally, so the maximum area of each part is $1/(N+1)$, so 
	\begin{align*}
	& Weighted \ \ AUC-ROC=\sum_{i=0}^N area(i)\times W(i)  \nonumber \\
	& \leqslant \sum_{i=0}^N \frac{W(i)}{N+1}= \frac{\sum_{i=0}^N W(i)}{N+1} \nonumber \\ 
	& = \frac{(1-\gamma)+(1-\gamma^2)+...+(1-\gamma^n)+w_n}{N+1} \nonumber \\
	& = \frac{N-\frac{\gamma(1-\gamma^n)}{1-\gamma}+w_n}{N+1} \nonumber \\
	& = \frac{N-\frac{\gamma(1-\gamma^n)}{1-\gamma}+\frac{(1-\gamma^n)\gamma+(1-\gamma)}{1-\gamma}}{N+1} \\
	& = 1
	\end{align*}
	So, it is proved that the maximum value of the weighted AUC-ROC is 1. 
	
	It is easy to see that by substitute $\gamma=0$ in Equation~\ref{weight}, $w_0=w_1=...=w_n=1$, every partition has the weight 1 as it is in the conventional AUC; when $\gamma=1$, $w_i=1-\gamma^i=0 / / (for i=0,...,n-1)$ and $w_n=1$, only the top area has non-zero weight.\end{proof}



\bibliographystyle{plainnat}
\renewcommand{\bibname}{References} 
\bibliography{References/references} 
\end{document}